\documentclass[nohyperref]{article}

\usepackage{amsmath,amsfonts,bm}

\def\eqref#1{equation~\ref{#1}}
\def\1{\bm{1}}

\def\vzero{{\bm{0}}}
\def\vone{{\bm{1}}}
\def\vmu{{\bm{\mu}}}
\def\vnu{{\bm{\nu}}}
\def\vtheta{{\bm{\theta}}}
\def\vphi{{\bm{\phi}}}
\def\veta{{\bm{\eta}}}
\def\vepsilon{{\bm{\epsilon}}}

\def\vsigma{{\bm{\sigma}}}

\def\vg{{\bm{g}}}
\def\vh{{\bm{h}}}

\def\vw{{\bm{w}}}
\def\vx{{\bm{x}}}
\def\vy{{\bm{y}}}

\def\mI{{\bm{I}}}

\def\mM{{\bm{M}}}

\def\mSigma{{\bm{\Sigma}}}

\DeclareMathAlphabet{\mathsfit}{\encodingdefault}{\sfdefault}{m}{sl}
\SetMathAlphabet{\mathsfit}{bold}{\encodingdefault}{\sfdefault}{bx}{n}

\def\gN{{\mathcal{N}}}
\def\gO{{\mathcal{O}}}

\def\sR{{\mathbb{R}}}

\newcommand{\E}{\mathbb{E}}

\newcommand{\R}{\mathbb{R}}

\newcommand{\softmax}{\mathrm{softmax}}

\newcommand{\KL}{D_{\mathrm{KL}}}

\newcommand{\Cov}{\mathrm{Cov}}

\newcommand{\lelbo}{L_{\mathrm{elbo}}}

\newcommand{\diag}{\mathrm{diag}}

\DeclareMathOperator*{\argmax}{arg\,max}

\DeclareMathOperator{\tr}{tr}

\usepackage{microtype}
\usepackage{graphicx}
\usepackage{subfig}
\usepackage{booktabs} % for professional tables
\usepackage{multirow}
\usepackage{colortbl}
\usepackage{float}
\usepackage{makecell}

\usepackage{hyperref}

\usepackage[accepted]{icml2022}

\usepackage{amsmath}
\usepackage{amssymb}
\usepackage{mathtools}
\usepackage{amsthm}
\usepackage{thmtools, thm-restate}

\usepackage[capitalize,noabbrev]{cleveref}

\theoremstyle{plain}
\newtheorem{theorem}{Theorem}[section]
\newtheorem{proposition}[theorem]{Proposition}
\newtheorem{lemma}[theorem]{Lemma}

\theoremstyle{definition}

\theoremstyle{remark}

\usepackage[textsize=tiny]{todonotes}

\def\nprddpm{NPR-DDPM}
\def\nprddim{NPR-DDIM}
\def\nprdpm{NPR-DPM}
\def\snddpm{SN-DDPM}
\def\snddim{SN-DDIM}
\def\sndpm{SN-DPM}

\icmltitlerunning{Submission and Formatting Instructions for ICML 2022}

\begin{document}

\twocolumn[
\icmltitle{Estimating the Optimal Covariance with Imperfect Mean \\ in Diffusion Probabilistic Models}



\icmlsetsymbol{equal}{*}

\begin{icmlauthorlist}
\icmlauthor{Fan Bao}{thu}
\icmlauthor{Chongxuan Li}{renmin}
\icmlauthor{Jiacheng Sun}{huawei}
\icmlauthor{Jun Zhu}{thu}
\icmlauthor{Bo Zhang}{thu}
\end{icmlauthorlist}

\icmlaffiliation{thu}{Dept. of Comp. Sci. \& Tech., Institute for AI, Tsinghua-Huawei Joint Center for AI, BNRist Center, State Key Lab for Intell. Tech. \& Sys., Tsinghua University}
\icmlaffiliation{renmin}{Gaoling School of AI, Renmin University of China; Beijing Key Lab of Big Data Management and Analysis Methods, Beijing, China}
\icmlaffiliation{huawei}{Huawei Noah's Ark Lab}

\icmlcorrespondingauthor{Chongxan Li}{chongxuanli@ruc.edu.cn}
\icmlcorrespondingauthor{Jun Zhu}{dcszj@tsinghua.edu.cn}

\icmlkeywords{Machine Learning, ICML}

\vskip 0.3in
]



\printAffiliationsAndNotice{}  % leave blank if no need to mention equal contribution

\begin{abstract}
Diffusion probabilistic models (DPMs) are a class of powerful deep generative models (DGMs). Despite their success, the iterative generation process over the full timesteps is much less efficient than other DGMs such as GANs. Thus, the generation performance on a subset of timesteps is crucial, which is greatly influenced by the covariance design in DPMs. In this work, we consider diagonal and full covariances to improve the expressive power of DPMs. We derive the optimal result for such covariances, and then correct it when the mean of DPMs is imperfect. Both the optimal and the corrected ones can be decomposed into terms of conditional expectations over functions of noise. Building upon it, we propose to estimate the optimal covariance and its correction given imperfect mean by learning these conditional expectations. Our method can be applied to DPMs with both discrete and continuous timesteps. We consider the diagonal covariance in our implementation for computational efficiency. For an efficient practical implementation, we adopt a parameter sharing scheme and a two-stage training process. Empirically, our method outperforms a wide variety of covariance design on likelihood results, and improves the sample quality especially on a small number of timesteps.
\end{abstract}

\section{Introduction}
\label{sec:intro}

Recently, diffusion probabilistic models (DPMs)~\cite{sohl2015deep,ho2020denoising,song2020score} have shown great promise for generative modeling. Such models smoothly inject noise to the data distribution, which forms a diffusion process. By learning to reverse the process using a Markov model, DPMs are able to generate high quality images~\cite{ho2020denoising,song2020score,dhariwal2021diffusion} and audios~\cite{chen2020wavegrad,kong2020diffwave}, which are comparable or even superior than the current state-of-the-art generative models~\cite{brock2018large,wu2019logan,karras2020analyzing,binkowski2019high,kalchbrenner2018efficient}.

However, the iterative generation over the full timesteps of DPMs makes them much less efficient than generative models such as GANs~\cite{goodfellow2014generative}. Thus, the generation performance on a subset of timesteps is crucial. In this case, the transition of the reversed diffusion process becomes more complex~\cite{xiao2021tackling}, and the covariance design in DPMs matters~\cite{nichol2021improved,bao2022analytic}. Most prior works~\cite{ho2020denoising,song2020denoising,bao2022analytic} use an isotropic covariance that only depends on the timestep without considering the state. 
A notable recent progress is Analytic-DPM~\cite{bao2022analytic}, which estimates the optimal isotropic covariance (in the sense of maximum likelihood) instead of using handcrafted values~\cite{ho2020denoising,song2020denoising} and shows significant improvement on likelihood estimation and sampling efficiency. However, the isotropic covariance sacrifices the expressive power of DPMs for simplicity. Moreover, the optimality of the covariance estimate in~\citet{bao2022analytic} holds by assuming that the optimal mean is known, which is not the case in practice.

To overcome the aforementioned limitations and further improve DPMs for likelihood estimation and sampling efficiency, we consider diagonal and full covariances to improve the expressive power of DPMs.
We derive the optimal mean and covariance from the perspective of the maximum likelihood. We also correct the optimal covariance given an imperfect mean (i.e., considering the approximation and optimization errors) in terms of maximum likelihood.
Both the optimal and the corrected ones can be decomposed into terms of conditional expectations over functions of noise, which can be estimated by minimizing mean squared error (MSE) losses. 
Although our theory applies to the full covariance case, in our implementation, we consider the diagonal covariance for computational efficiency. Besides, we adopt a parameter sharing scheme for inference efficiency and a two-stage training process motivated by our theory.

Our method is applicable to DPMs with both discrete~\cite{ho2020denoising} and continuous~\cite{song2020score} timesteps. In our experiments, we directly compare our method with a variety of baselines~\cite{ho2020denoising,song2020denoising,bao2022analytic,song2020score} in terms of sample quality and likelihood estimation in DPMs with both discrete and continuous timesteps. At nearly the same computation cost, our method consistently outperforms these baselines on likelihood estimation. Besides, our method also outperforms these baselines on the sample quality in most cases, especially when the number of timesteps is small.

\begin{table*}[t]
    \centering
    \begin{tabular}{ccc}
    \toprule
       \multirow{7}{*}{\makecell{Forward \\ process}} & $\overline{\alpha}_n$, $\overline{\beta}_n$ & \makecell{represent the cumulative amount of Gaussian noise at timestep $n$, \\ i.e., $q(\vx_n|\vx_0) = \gN(\vx_n|\sqrt{\overline{\alpha}_n}\vx_0, \overline{\beta}_n \mI)$.} \\
        \cmidrule{2-3} 
        & $\lambda_n^2$ & represents the variance of the transition $q(\vx_{n-1}|\vx_n, \vx_0)$. \\
        \cmidrule{2-3} 
        & $\tilde{\vmu}_n(\vx_n, \vx_0)$ & represents the mean of the transition $q(\vx_{n-1}|\vx_n, \vx_0)$. \\
        \cmidrule{2-3} 
        & $\gamma_n$ & represents the coefficient of $\vx_0$ in $\tilde{\vmu}_n(\vx_n, \vx_0)$. \\
        \cmidrule{2-3} 
        & $\alpha_n$, $\beta_n$ & \makecell{represent the amount of added Gaussian noise at a single timestep $n$ in \\ the DDPM forward process,  i.e., $q(\vx_n|\vx_{n-1}) = \gN(\vx_n|\sqrt{\alpha_n}\vx_{n-1}, \beta_n \mI)$.}\\
        \midrule
        \multirow{2}{*}{\makecell{Reverse \\ mean}} & $\vmu_n(\vx_n)$, $\vmu_n^*(\vx_n)$ & represent the mean of $p(\vx_{n-1}|\vx_n)$ and the optimal one. \\
        \cmidrule{2-3} 
        &$\hat{\vepsilon}_n(\vx_n)$, $\hat{\vmu}_n(\vx_n)$ & represent the noise prediction network and the estimate of $\vmu_n^*(\vx_n)$. \\
        \midrule
        \multirow{4}{*}{\makecell{Reverse \\ covariance}} & $\mSigma_n(\vx_n)$, $\vsigma_n^*(\vx_n)^2$ & represent the covariance of $p(\vx_{n-1}|\vx_n)$ and the optimal diagonal one. \\
        \cmidrule{2-3} 
        & $\vh_n(\vx_n)$, $\hat{\vsigma}_n(\vx_n)^2$ & represent the SN prediction network and the estimate of $\vsigma_n^*(\vx_n)^2$.\\
        \cmidrule{2-3} 
        & $\tilde{\vsigma}_n^*(\vx_n)^2$ & represents the corrected optimal diagonal covariance with an imperfect mean.\\
        \cmidrule{2-3} 
        & $\vg_n(\vx_n)$, $\hat{\tilde{\vsigma}}_n^*(\vx_n)^2$ & represent the NPR prediction network and the estimate of $\tilde{\vsigma}_n^*(\vx_n)^2$. \\
    \bottomrule
    \end{tabular}
    \caption{Notations used in this paper.}
    \label{tab:notation}
\end{table*}

\section{Background}
\label{sec:bg}

Diffusion probabilistic models (DPMs) are special Markov models with Gaussian transitions:
\begin{align}
    & p(\vx_{0:N}) = p(\vx_N) \prod\limits_{n=1}^N p(\vx_{n-1}|\vx_n), \label{eq:markov} \\
    & p(\vx_{n-1}|\vx_n) = \gN(\vx_{n-1}|\vmu_n(\vx_n), \mSigma_n(\vx_n)), \nonumber
\end{align}
which aim to reverse a forward process $q(\vx_{1:N}|\vx_0)$ that gradually injects noise to a data distribution $q(\vx_0)$. \citet{song2020denoising} consider a family of forward processes indexed by a non-negative vector $\lambda=(\lambda_1, \cdots, \lambda_N) \in \R_{\geq 0}^N$:
\begin{align}
\label{eq:ddim}
    & q(\vx_{1:N}|\vx_0) = q(\vx_N|\vx_0) \prod\limits_{n=2}^N q(\vx_{n-1}|\vx_n, \vx_0), \\
    & q(\vx_N|\vx_0) = \gN(\vx_N|\sqrt{\overline{\alpha}_N}\vx_0, \overline{\beta}_N \mI), \nonumber \\
    & q(\vx_{n-1}|\vx_n, \vx_0) = \gN(\vx_{n-1}|\tilde{\vmu}_n(\vx_n, \vx_0), \lambda_n^2 \mI), \nonumber \\
    & \tilde{\vmu}_n(\vx_n, \vx_0) \!=\! \sqrt{\overline{\alpha}_{n-1}} \vx_0 \!+\! \sqrt{\overline{\beta}_{n-1} - \lambda_n^2} \cdot \frac{\vx_n \!-\! \sqrt{\overline{\alpha}_n}\vx_0}{\sqrt{\overline{\beta}_n}}, \nonumber
\end{align}
where $\overline{\alpha}_1, \overline{\alpha}_2, \cdots, \overline{\alpha}_N \in (0, 1)$ is a strictly decreasing sequence, $\overline{\beta}_n \coloneqq 1 - \overline{\alpha}_n$ and $\mI$ is the identity matrix. Observing that $q(\vx_n|\vx_0) = \gN(\vx_n|\sqrt{\overline{\alpha}_n}\vx_0, \overline{\beta}_n \mI)$, we can quickly sample $\vx_n$ given $\vx_0$ by
\begin{align*}
    \vx_n = \sqrt{\overline{\alpha}_n} \vx_0 + \sqrt{\overline{\beta}_n} \vepsilon_n, \quad \vepsilon_n \sim \gN(\vzero, \mI).
\end{align*}

Let $\alpha_n \coloneqq \overline{\alpha}_n / \overline{\alpha}_{n-1}$, $\beta_n \coloneqq 1 - \alpha_n$ and $\tilde{\beta}_n \coloneqq \frac{\overline{\beta}_{n-1}}{\overline{\beta}_n}\beta_n$.
Two commonly used forward processes are the denoising diffusion probabilistic model (DDPM) forward process (corresponding to $\lambda_n^2 = \tilde{\beta}_n$) and the denoising diffusion implicit model (DDIM) forward process (corresponding to $\lambda_n^2 = 0$). In particular, the DDPM forward process is Markovian with linear Gaussian transition:
\begin{align*}
    q(\vx_n|\vx_{n-1}) = \gN(\vx_n|\sqrt{\alpha_n}\vx_{n-1}, \beta_n \mI).
\end{align*}

The reverse process is learned by maximizing the evidence lower bound (ELBO) $\lelbo=\E_q \log \frac{p(\vx_{0:N})}{q(\vx_{1:N}|\vx_0)}$ on log-likelihood, or equivalently, by minimizing the KL divergence between the forward and the reverse process
\begin{align}
\label{eq:elbo}
\max\limits_{\{\vmu_n, \mSigma_n\}_{n=1}^N} \lelbo \Leftrightarrow \min\limits_{\{\vmu_n, \mSigma_n\}_{n=1}^N} \KL(q(\vx_{0:N}) \| p(\vx_{0:N})).
\end{align}
Note that problem~{(\ref{eq:elbo})} optimizes the mean and the covariance jointly. Thus, we term it \textit{joint optimization}.

For simplicity and at the cost of the flexibility, prior works~\cite{ho2020denoising,song2020denoising,bao2022analytic} set $\mSigma_n(\vx_n) = \sigma_n^2 \mI$ to an isotropic covariance that only depends
on the timestep $n$, where $\sigma_n^2$ is the variance for each component. Under this restriction, Analytic-DPM~\cite{bao2022analytic} shows that both the optimal mean $\vmu_n^*(\vx_n)$ and variance $\sigma_n^{*2}$ have analytic forms w.r.t. the conditional expectation of the noise\footnote{The original work~\cite{bao2022analytic} uses the score function $\nabla_{\vx_n} \log q(\vx_n)$ to represent $\vmu_n^*(\vx_n)$ and $\sigma_n^{*2}$, which is equivalent to the conditional expectation of the noise up to multiplying a constant: $\nabla_{\vx_n} \log q(\vx_n) = - \frac{1}{\sqrt{\overline{\beta}_n}} \E_{q(\vx_0|\vx_n)} [\vepsilon_n]$.} $\E_{q(\vx_0|\vx_n)} [\vepsilon_n]$:
\begin{align}
& \vmu_n^*(\vx_n) \!=\! \tilde{\vmu}_n \!\left(\!\vx_n, \frac{1}{\sqrt{\overline{\alpha}_n}} ( \vx_n \!-\! \sqrt{\overline{\beta}_n} \E_{q(\vx_0|\vx_n)} [\vepsilon_n]) \!\right),\! \label{eq:opt_mean} \\
& \sigma_n^{*2} = \lambda_n^2 + \gamma_n^2 \frac{\overline{\beta}_n}{\overline{\alpha}_n} \left(1 \!-\!  \E_{q(\vx_n)} \frac{\|\E_{q(\vx_0|\vx_n)} [\vepsilon_n]\|_2^2}{d} \right), \label{eq:opt_var}
\end{align}
where $\vepsilon_n = \frac{\vx_n - \sqrt{\overline{\alpha}_n} \vx_0}{\sqrt{\overline{\beta}_n}}$ is the noise used to generate $\vx_n$ from $\vx_0$, $d$ is the dimension of the data $\vx_0$ and $\gamma_n = \sqrt{\overline{\alpha}_{n-1}} - \sqrt{\overline{\beta}_{n-1} - \lambda_n^2} \sqrt{\frac{\overline{\alpha}_n}{\overline{\beta}_n}}$. \citet{ho2020denoising} estimate $\vmu_n^*(\vx_n)$ by a noise prediction network $\hat{\vepsilon}_n(\vx_n)$
\begin{align}
\label{eq:mu_param}
    \hat{\vmu}_n(\vx_n) = \tilde{\vmu}_n \left(\vx_n, \frac{1}{\sqrt{\overline{\alpha}_n}} ( \vx_n - \sqrt{\overline{\beta}_n} \hat{\vepsilon}_n(\vx_n)) \right).
\end{align}
Here $\hat{\vepsilon}_n(\vx_n)$ aims to learn $\E_{q(\vx_0|\vx_n)}[\vepsilon_n]$ by minimizing the following MSE loss:
\begin{align}
\label{eq:mse_eps}
\min\limits_{\{\hat{\vepsilon}_n\}_{n=1}^N} \E_n \E_{q(\vx_0, \vx_n)} \|\vepsilon_n - \hat{\vepsilon}_n(\vx_n)\|_2^2,
\end{align}
where $n$ is uniformly sampled from $\{1,2,\cdots,N\}$. Eq.~{(\ref{eq:mse_eps})} admits an optimal solution when $\hat{\vepsilon}_n(\vx_n) = \E_{q(\vx_0|\vx_n)}[\vepsilon_n]$ for all $n \in \{1,2,\cdots,N\}$. 

The optimal variance in Eq.~{(\ref{eq:opt_var})} can also be estimated using $\hat{\vepsilon}_n(\vx_n)$ as following~\cite{bao2022analytic}
\begin{align}
\label{eq:adpm}
    \hat{\sigma}_n^2 = \lambda_n^2 + \gamma_n^2 \frac{\overline{\beta}_n}{\overline{\alpha}_n} \left(1 -  \E_{q(\vx_n)} \frac{\|\hat{\vepsilon}_n(\vx_n)\|_2^2}{d} \right).
\end{align}

\section{Method}
\label{sec:method}

We improve the covariance estimate of DPMs building upon the recent progress~\cite{bao2022analytic}. 
First, we consider diagonal and full covariances instead of using an isotropic one to improve the expressive power of DPMs and obtain the expression of the optimal solution, as detailed in Section~\ref{sec:flex_cov}. Second, we correct the optimal covariance given an imperfect mean (i.e., with approximation and optimization errors), as presented in Section~\ref{sec:opt_cov}.  For clarity, we refer the readers to Appendix~\ref{sec:proof} for all proofs.

Although our theoretical results are applicable to the full covariance case (see Appendix~\ref{sec:ext_cov_g}), it is often time consuming to obtain samples from Gaussian transitions with full covariances (e.g., via Cholesky decomposition).
To balance the flexibility and the time cost, we focus on diagonal covariances throughout the main text including the experiments.

\subsection{The Optimal Solution for Covariances beyond Isotropic Ones}
\label{sec:flex_cov}

Instead of using
an isotropic covariance that only depends on
the timestep without considering the state, 
we consider a diagonal covariance to improve the expressive power of DPMs. To be clear and rigorous, we first construct an example where the data distribution is a mixture of Gaussian. Even in such a simple case, we can prove that the optimal ELBO (in Eq.~{(\ref{eq:elbo})}) with a diagonal covariance is \emph{strictly greater} than that with an isotropic covariance. See a formal description and a detailed proof in Proposition~\ref{thm:express_example}.

Formally, we consider the covariance in the form of $\mSigma_n(\vx_n) = \diag(\vsigma_n(\vx_n)^2)$, where $\vsigma_n(\vx_n)^2: \R^d \rightarrow \R^d$ is the diagonal of the covariance matrix and $\diag(\cdot)$ denotes the vector-to-matrix diag operator.
Naturally, we optimize problem~{(\ref{eq:elbo})} w.r.t. the mean and covariance jointly and derive the the optimal solution, as summarized in Theorem~\ref{thm:opt_diag_no_err}. See the result for full covariances in Appendix~\ref{sec:ext_cov_g}.

\begin{restatable}{theorem}{optdiagnoerr}
\label{thm:opt_diag_no_err}
(Optimal solution to joint optimization) Suppose $\mSigma_n(\vx_n) = \diag(\vsigma_n(\vx_n)^2)$. Then the optimal mean to problem~{(\ref{eq:elbo})} is $\vmu_n^*(\vx_n)$ as in Eq.~{(\ref{eq:opt_mean})}, and the optimal covariance to   problem~{(\ref{eq:elbo})} is 
\begin{align*}
\vsigma_n^*(\vx_n)^2 \!=\! \lambda_n^2 \vone \!+\! \gamma_n^2 \frac{\overline{\beta}_n}{\overline{\alpha}_n} \!\left(\E_{q(\vx_0|\vx_n)}[\vepsilon_n ^2] \!-\! \E_{q(\vx_0|\vx_n)}[\vepsilon_n]^2\right),
\end{align*}
where $\vepsilon_n = \frac{\vx_n - \sqrt{\overline{\alpha}_n} \vx_0}{\sqrt{\overline{\beta}_n}}$ is the noise used to generate $\vx_n$ from $\vx_0$, $(\cdot)^2$ is the element-wise square, $\vone$ is the vector of ones and $\gamma_n = \sqrt{\overline{\alpha}_{n-1}} - \sqrt{\overline{\beta}_{n-1} - \lambda_n^2} \sqrt{\frac{\overline{\alpha}_n}{\overline{\beta}_n}}$.
\end{restatable}
The proof idea is similar to Theorem 1 in~\citet{bao2022analytic}, which only considers the isotropic covariance case.

To estimate the optimal covariance $\vsigma_n^*(\vx_n)^2$ in Theorem~\ref{thm:opt_diag_no_err}, we need to estimate both $\E_{q(\vx_0|\vx_n)}[\vepsilon_n]^2$ and $\E_{q(\vx_0|\vx_n)}[\vepsilon_n ^2]$.

As mentioned in Section~\ref{sec:bg}, the noise prediction network $\hat{\vepsilon}_n(\vx_n)$ aims to estimate $\E_{q(\vx_0|\vx_n)}[\vepsilon_n]$ by minimizing the MSE loss in Eq.~{(\ref{eq:mse_eps})}. Thus, we use $\hat{\vepsilon}_n(\vx_n)^2$ to estimate $\E_{q(\vx_0|\vx_n)}[\vepsilon_n]^2$. Note that the error of $\hat{\vepsilon}_n(\vx_n)^2$ is likely to be amplified compared to that of $\hat{\vepsilon}_n(\vx_n)$ (see Appendix~\ref{sec:err_amp} for details). However, we still find it performs well on sample quality (see Section~\ref{sec:exp_sample}).
% since $\|\hat{\vepsilon}_n(\vx_n)^2 - \E_{q(\vx_0|\vx_n)}[\vepsilon_n]^2\|_2 \leq A \|\hat{\vepsilon}_n(\vx_n) - \E_{q(\vx_0|\vx_n)}[\vepsilon_n]\|_2$, where $A = \|\hat{\vepsilon}_n(\vx_n)\|_\infty+\|\E_{q(\vx_0|\vx_n)}[\vepsilon_n]\|_\infty$ 

As for $\E_{q(\vx_0|\vx_n)}[\vepsilon_n ^2]$, we note that it is the expectation of the \textit{squared noise (SN)} $\vepsilon_n^2$ conditioned on $\vx_n$. Such a conditional expectation can be learned using a neural network $\vh_n(\vx_n) \in \sR^d$ trained on the following MSE loss
\begin{align}
\label{eq:mse_eps2}
    \min\limits_{\{\vh_n\}_{n=1}^N} \E_n \E_{q(\vx_0, \vx_n)} \|\vepsilon_n^2 - \vh_n(\vx_n) \|_2^2,
\end{align}
where $n$ is uniformly sampled from $\{1,2,\cdots,N\}$ and $\vh_n$ attempts to predict $\vepsilon_n^2$ given $\vx_n$. Here we call $\vh_n(\vx_n)$ the \textit{SN prediction network}. Eq.~{(\ref{eq:mse_eps2})} admits an optimal solution when $\vh_n(\vx_n) = \E_{q(\vx_0|\vx_n)}[\vepsilon_n^2]$ for all $n \in \{1,2,\cdots,N\}$.

By estimating $\E_{q(\vx_0|\vx_n)}[\vepsilon_n]^2$ with $\hat{\vepsilon}_n(\vx_n)^2$ and estimating $\E_{q(\vx_0|\vx_n)}[\vepsilon_n^2]$ with $\vh_n(\vx_n)$, we can estimate $\vsigma_n^*(\vx_n)^2$ by
\begin{align}
\label{eq:opt_est}
    \hat{\vsigma}_n(\vx_n)^2 = \lambda_n^2 \vone + \gamma_n^2 \frac{\overline{\beta}_n}{\overline{\alpha}_n} (\vh_n(\vx_n) - \hat{\vepsilon}_n(\vx_n)^2).
\end{align}
In the paper, we refer to the above estimate as \textit{\sndpm}.

\subsection{The Optimal Covariance with an Imperfect Mean}
\label{sec:opt_cov}

It is plausible to obtain the estimates $\hat{\vmu}_n(\vx_n)$ and $\hat{\vsigma}_n(\vx_n)^2$, 
according to Theorem~\ref{thm:opt_diag_no_err}.
However, we argue that the exact optimal solution $\vmu^*_n(\vx_n), \vsigma^*_n(\vx_n)^2$ in Theorem~\ref{thm:opt_diag_no_err} cannot be achieved due to the approximation and optimization errors, making it possible to further improve the estimates considering the errors. 
In fact, below we present how to correct the optimal covariance given a potentially imperfect mean $\hat{\vmu}_n$. 

We note that, in Theorem~\ref{thm:opt_diag_no_err}, the optimal mean is irrelevant to the covariance, while the optimal covariance is expressed by the optimal mean. Naturally, we can first optimize the mean solely to obtain the mean estimate $\hat{\vmu}_n$ and then optimize the covariance solely given $\hat{\vmu}_n$. 
Such a two-stage approach is at least not worse than the estimate in Section~\ref{sec:flex_cov} under the maximum likelihood, because we have
\begin{align}
    \lelbo(\hat{\vmu}_n, \hat{\vsigma}_n(\cdot)^2) \leq \max\limits_{\vsigma_n(\cdot)^2}
    \lelbo(\hat{\vmu}_n, \vsigma_n(\cdot)^2),
\end{align}
where the equality does not hold in general.

To formalize the idea, we obtain the optimal solutions to the following two problems
\begin{align}
\mbox{Given arbitrary covariance }\mSigma_n, &\max\limits_{\{\vmu_n\}_{n=1}^N} \lelbo, \label{eq:elbo_mean} \\
\mbox{Given arbitrary mean }\vmu_n, &\max\limits_{\{\mSigma_n\}_{n=1}^N} \lelbo \label{eq:elbo_cov},
\end{align}
as summarized in Theorem~\ref{thm:opt_mean} and Theorem~\ref{thm:opt_diag}.

% \begin{align}
%     \lelbo(\vmu^*_n, \mSigma^*_n)
% \end{align}

% Note that Eq.~{(\ref{eq:elbo})} is a joint optimization problem of $\lelbo$ w.r.t. both the mean and the covariance, and the solution in Theorem~\ref{thm:opt_diag_no_err} is only optimal for this joint optimization. In this part, we consider the following two optimization problems w.r.t. the mean or the covariance solely:

% Compared to the joint optimization problem~{(\ref{eq:elbo})}, problems~{(\ref{eq:elbo_mean})} and~{(\ref{eq:elbo_cov})} can characterize more practical settings\junz{more practical in what sense? intuitively, if you can have joint optimum, why care about the subproblems? need some motivation...}, e.g., what the optimal covariance is when the mean is imperfect.

\begin{restatable}{theorem}{optmean}
\label{thm:opt_mean}
(Optimal solution to optimization solely w.r.t. mean)
For any covariance $\mSigma_n$, the optimal mean to problem~{(\ref{eq:elbo_mean})} is always Eq.~{(\ref{eq:opt_mean})}, i.e.,
\begin{align*}
\vmu_n^*(\vx_n) \!=\! \tilde{\vmu}_n \! \left(\vx_n, \frac{1}{\sqrt{\overline{\alpha}_n}} ( \! \vx_n \!-\! \sqrt{\overline{\beta}_n} \E_{q(\vx_0|\vx_n)} [\vepsilon_n]) \right),
\end{align*}
which is irrelevant to $\mSigma_n$.
\end{restatable}
Theorem~\ref{thm:opt_mean} reveals that the optimal mean doesn't depend on the covariance and therefore we can learn the mean on Eq.~{(\ref{eq:mse_eps})} without knowing information about the covariance.

In contrast to the mean, the optimal covariance to problem~{(\ref{eq:elbo_cov})} depends on the mean, as shown in Theorem~\ref{thm:opt_diag}. See the result for full covariances in Appendix~\ref{sec:ext_cov_g}.
\begin{restatable}{theorem}{optdiag}
\label{thm:opt_diag}
(Optimal solution to optimization w.r.t. covariance solely) Suppose $\mSigma_n(\vx_n) = \diag(\vsigma_n(\vx_n)^2)$. For any mean $\vmu_n(\vx_n)$ that is parameterized by a noise prediction network $\hat{\vepsilon}_n(\vx_n)$ as in Eq.~{(\ref{eq:mu_param})}, the optimal covariance $\tilde{\vsigma}_n^*(\vx_n)^2$ to problem~{(\ref{eq:elbo_cov})} is
\begin{align}
\tilde{\vsigma}_n^*(\vx_n)^2 =
& \vsigma_n^*(\vx_n)^2 + \gamma_n^2 \frac{\overline{\beta}_n}{\overline{\alpha}_n} \underbrace{(\hat{\vepsilon}_n(\vx_n) - \E_{q(\vx_0|\vx_n)}[\vepsilon_n] )^2}_{\text{\normalsize{error}}} \nonumber \\
= & \lambda_n^2 \vone \!+\! \gamma_n^2 \frac{\overline{\beta}_n}{\overline{\alpha}_n} \E_{q(\vx_0|\vx_n)}[(\vepsilon_n \!-\! \hat{\vepsilon}_n(\vx_n))^2], \label{eq:opt_cov}
\end{align}
where $\vsigma_n^*(\vx_n)^2$ is the optimal covariance to the joint optimization problem in Theorem~\ref{thm:opt_diag_no_err}, $\vepsilon_n = \frac{\vx_n - \sqrt{\overline{\alpha}_n} \vx_0}{\sqrt{\overline{\beta}_n}}$ is the noise used to generate $\vx_n$ from $\vx_0$, $(\cdot)^2$ is the element-wise square, $\vone$ is the vector of ones and $\gamma_n = \sqrt{\overline{\alpha}_{n-1}} - \sqrt{\overline{\beta}_{n-1} - \lambda_n^2} \sqrt{\frac{\overline{\alpha}_n}{\overline{\beta}_n}}$.
\end{restatable}

The proof of Theorem~\ref{thm:opt_diag} mainly builds on this fact: minimizing the KL divergence between a target density and a Gaussian density conditioned on a fixed mean is equivalent to firstly obtaining the second moment of the target density, and then correcting it by the difference between the first moments of the two densities. We term it \textit{conditioned moment matching} (see Lemma~\ref{thm:opt_sigma_general}).

Theorem~\ref{thm:opt_diag} shows that the optimal covariance to the joint optimization problem in Theorem~\ref{thm:opt_diag_no_err} should be corrected by the error of $\hat{\vepsilon}_n(\vx_n)$ when $\hat{\vepsilon}_n(\vx_n) \neq \E_{q(\vx_0|\vx_n)}[\vepsilon]$, which is inevitable in practice due to a non-realizable model family and the optimization error of gradient-based methods.

According to Eq.~{(\ref{eq:opt_cov})}, the corrected covariance $\tilde{\vsigma}_n^*(\vx_n)^2$ is determined by $\E_{q(\vx_0|\vx_n)}[(\vepsilon_n - \hat{\vepsilon}_n(\vx_n))^2]$. We learn it by training a neural network $\vg_n(\vx_n) \in \sR^d$ that predicts $(\vepsilon_n - \hat{\vepsilon}_n(\vx_n))^2$ given $\vx_n$ on the following MSE loss
\begin{align}
\label{eq:mse_npr}
    \min\limits_{\{\vg_n\}_{n=1}^N} \E_n \E_{q(\vx_0, \vx_n)} \|(\vepsilon_n \!-\! \hat{\vepsilon}_n(\vx_n))^2 - \vg_n(\vx_n)\|_2^2,
\end{align}
where $n$ is uniformly sampled from $\{1,2,\cdots,N\}$.
Here we term $(\vepsilon_n - \hat{\vepsilon}_n(\vx_n))^2$ the \textit{noise prediction residual (NPR)}, which is the residual between the ground truth noise and the noise predicted by the noise prediction network, and we call $\vg_n(\vx_n)$ the \textit{NPR prediction network}.
Eq.~{(\ref{eq:mse_npr})} admits an optimal solution when $\vg_n(\vx_n) = \E_{q(\vx_0|\vx_n)}[(\vepsilon_n - \hat{\vepsilon}_n(\vx_n))^2]$ for all $n \in \{1,2,\cdots,N\}$. Note that in contrast to $\hat{\vepsilon}_n(\vx_n)^2$ used in SN-DPM, $\vg_n(\vx_n)$ doesn't have the error amplification problem (see Appendix~\ref{sec:err_amp} for details). With a NPR prediction network $\vg_n(\vx_n)$ trained on Eq.~{(\ref{eq:mse_npr})}, we estimate $\tilde{\vsigma}_n^*(\vx_n)^2$ by
\begin{align}
\label{eq:err_est}
    \hat{\tilde{\vsigma}}_n(\vx_n)^2 = \lambda_n^2 \vone + \gamma_n^2 \frac{\overline{\beta}_n}{\overline{\alpha}_n} \vg_n(\vx_n).
\end{align}
In the paper, we refer to the above estimate as \textit{\nprdpm}.

\section{Implementation}
\label{sec:imp}

We provide more details of implementing \sndpm{}  and \nprdpm{}.
Specifically, we design a parameter sharing scheme for inference efficiency in Section~\ref{sec:ps}. Motivated by Theorem~{\ref{thm:opt_mean} \& \ref{thm:opt_diag}}, we employ a two-stage training process for both \sndpm{} and \nprdpm{} in Section~\ref{sec:pretrain}.

\subsection{Parameter Sharing and Inference Efficiency}
\label{sec:ps}

During inference, both \sndpm{} and \nprdpm{} need to evaluate two networks.
To reduce the computation cost, we let the two neural networks share parameters as follows
\begin{align}
    \hat{\vepsilon}_n(\vx_n) & = \mathrm{NN}_1(\mathrm{UNet}(\vx_n, n; \vtheta); \vphi_1), \nonumber \\
    \vh_n(\vx_n) \mbox{ or } \vg_n(\vx_n) & = \mathrm{NN}_2(\mathrm{UNet}(\vx_n, n; \vtheta); \vphi_2), \label{eq:nn2}
\end{align}
where UNet is the most commonly used architecture in DPMs~\cite{ho2020denoising,song2020score} and is parameterized by the shared parameter $\vtheta$, $\mathrm{NN}_1$ and $\mathrm{NN}_2$ are two small neural networks such as convolution layers and are parameterized by $\vphi_1$ and $\vphi_2$ respectively.

In comparison to original DPMs with the same UNet, both \sndpm{} and \nprdpm{} require negligible additional memory cost and at most $10\%$ more computation time (see details in Appendix~\ref{sec:mem_time}). 

\subsection{Two-Stage Learning and  Pretraining}
\label{sec:pretrain}
\vspace{-.1cm}

Recall that the optimal mean to problem~{(\ref{eq:elbo_mean})} is irrelevant to the covariance. Thus, the learning of the mean and the covariance can naturally be divided into two stages. In the first stage, we learn the mean by training a noise prediction network on Eq.~{(\ref{eq:mse_eps})}, or just use a pretrained one from prior works. In the second stage, we fix the parameter of the noise prediction network, take the UNet as a pretrained model, and only learn the covariance by tuning the parameter $\vphi_2$ (in Eq.~{(\ref{eq:nn2})}) of the NPR (or SN) prediction network using the gradient of Eq.~{(\ref{eq:mse_npr})} (or Eq.~{(\ref{eq:mse_eps2})}):
\begin{align*}
& \nabla_{\vphi_2} \E_n \E_{q(\vx_0, \vx_n)} \|(\vepsilon_n - \hat{\vepsilon}_n(\vx_n))^2 - \vg_n(\vx_n)\|_2^2, \mbox{ or } \\
& \nabla_{\vphi_2} \E_n \E_{q(\vx_0, \vx_n)} \|\vepsilon_n^2 - \vh_n(\vx_n) \|_2^2.
\end{align*}
Note that this pretraining paradigm for the second stage only needs to tune the small neural network $\mathrm{NN}_2$, which does not need to keep the computation graph of the UNet. Algorithm~{\ref{algo:npr}\&\ref{algo:sn}} present the second stage learning procedure.

\begin{figure*}
\begin{minipage}{0.48\textwidth}
\begin{algorithm}[H]
  \begin{algorithmic}[1]
  \caption{Learning of the NPR prediction network}
  \label{algo:npr}
    \STATE {\bfseries Input:} The NPR prediction network $\vg_n(\vx_n)$ and its tunable parameter $\vphi_2$; the pretrained noise prediction network $\hat{\vepsilon}_n(\vx_n)$
    \REPEAT 
    \STATE $\vx_0 \sim q(\vx_0)$
    \STATE $n \sim \mathrm{Uniform}(\{1,\cdots,N\})$
    \STATE $\vepsilon_n \sim \gN(\vzero, \mI)$
    \STATE $\vx_n = \sqrt{\overline{\alpha}_n} \vx_0 + \sqrt{\overline{\beta}_n} \vepsilon_n$
    \STATE Take gradient descent step on \\ \qquad $\nabla_{\vphi_2} \|(\vepsilon_n - \hat{\vepsilon}_n(\vx_n))^2 - \vg_n(\vx_n)\|_2^2$
    \UNTIL{converged}
  \end{algorithmic}
\end{algorithm}
\end{minipage}
\hspace{.1cm}
\begin{minipage}{0.48\textwidth}
\begin{algorithm}[H]
  \begin{algorithmic}[1]
  \caption{Learning of the SN prediction network}
  \label{algo:sn}
    \STATE {\bfseries Input:} The SN prediction network $\vh_n(\vx_n)$ and its tunable parameter $\vphi_2$
    \REPEAT 
    \STATE $\vx_0 \sim q(\vx_0)$
    \STATE $n \sim \mathrm{Uniform}(\{1,\cdots,N\})$
    \STATE $\vepsilon_n \sim \gN(\vzero, \mI)$
    \STATE $\vx_n = \sqrt{\overline{\alpha}_n} \vx_0 + \sqrt{\overline{\beta}_n} \vepsilon_n$
    \STATE Take gradient descent step on $\nabla_{\vphi_2} \|\vepsilon_n^2 - \vh_n(\vx_n) \|_2^2$
    \UNTIL{converged}
  \end{algorithmic}
\end{algorithm}
\end{minipage}
\end{figure*}

\section{Extension to DPMs with Continuous Timesteps}
\label{sec:ext_cdpm}

\citet{song2020score} generalize DPMs to continuous timesteps by introducing a SDE $\mathrm{d} \vx = f(t) \vx \mathrm{d} t + g(t) \mathrm{d} \vw$, where $f(t)$ and $g(t)$ are two pre-specified scalar functions and $\vw$ is the standard Wiener process. The SDE can be viewed as a continuous version of the forward process in Eq.~{(\ref{eq:ddim})}. It constructs a diffusion process $\{\vx_t\}_{t=0}^T$ indexed by a continuous timestep $t \in [0, T]$, where $\vx_0$ obeys the data distribution $q(\vx_0)$. For $0 \leq s < t \leq T$, the conditional distribution of $\vx_t$ given $\vx_s$ is $q(\vx_t|\vx_s) = \gN(\vx_t|\sqrt{\alpha_{t|s}}\vx_t, \beta_{t|s} \mI)$, where $\alpha_{t|s} = e^{2 \int_s^t f(\tau) \mathrm{d} \tau}$ and $\beta_{t|s} = \int_s^t g(\tau)^2 \alpha_{t|\tau} \mathrm{d} \tau$. Following \citet{bao2022analytic,kingma2021variational}, we reverse the diffusion process from timesteps $t$ to $s$ by $p(\vx_s|\vx_t) = \gN(\vx_s|\vmu_{s|t}(\vx_t), \mSigma_{s|t}(\vx_t))$ ($0 \leq s < t \leq T$), which is learned by
\begin{align}
\label{eq:ckl}
    \min\limits_{\vmu_{s|t}, \mSigma_{s|t}} \ell_{s,t} \coloneqq \E_{q(\vx_t)} \KL(q(\vx_s|\vx_t) \| p(\vx_s|\vx_t)),
\end{align}
where $q(\vx_t)$ is the marginal distribution of $\vx_t$. Similarly to the optimal solution in Eq.~{(\ref{eq:opt_mean}) \& (\ref{eq:opt_var})},  \citet{bao2022analytic} derive the optimal mean
\begin{align}
\label{eq:opt_mean_c}
    \vmu_{s|t}^*(\vx_t) = \frac{1}{\sqrt{\alpha_{t|s}}}(\vx_t - \frac{\beta_{t|s}}{\sqrt{\beta_{t|0}}} \E_{q(\vx_0|\vx_t)}[\vepsilon_t]),
\end{align}
where $\vepsilon_t = \frac{\vx_t - \sqrt{\alpha_{t|0}}\vx_0}{\sqrt{\beta_{t|0}}}$, as well as the optimal isotropic covariance $\sigma_{s|t}^{*2} \mI$. Similarly to Eq.~{(\ref{eq:mse_eps})}, a noise prediction network $\hat{\vepsilon}_t(\vx_t)$ is used to learn $\E_{q(\vx_0|\vx_t)}[\vepsilon_t]$, inducing an estimate of the optimal mean
\begin{align}
\label{eq:mu_param_c}
    \hat{\vmu}_{s|t}(\vx_t) = \frac{1}{\sqrt{\alpha_{t|s}}}(\vx_t - \frac{\beta_{t|s}}{\sqrt{\beta_{t|0}}} \hat{\vepsilon}_t(\vx_t)).
\end{align}
Firstly, we extend Theorem~\ref{thm:opt_diag_no_err} to DPMs with continuous timesteps. We derive the optimal solution to problem~{(\ref{eq:ckl})} when $\mSigma_{s|t}(\vx_t) = \diag(\vsigma_{s|t}(\vx_t)^2)$, as shown in Proposition~\ref{thm:opt_cov_c_no_err}.

\begin{restatable}{proposition}{optcovcnoerr}
\label{thm:opt_cov_c_no_err}
Suppose $\mSigma_{s|t}(\vx_t) = \diag(\vsigma_{s|t}(\vx_t)^2)$. Then the optimal mean to problem~{(\ref{eq:ckl})} is $\vmu_{s|t}^*(\vx_t)$ as in Eq.~{(\ref{eq:opt_mean_c})}, and the optimal covariance to problem~{(\ref{eq:ckl})} is
\begin{align*}
\vsigma_{s|t}^*(\vx_t)^2 \!=\! \tilde{\beta}_{s|t}\vone \!+\! \frac{\beta_{t|s}^2}{\beta_{t|0}\alpha_{t|s}}\! \left( \! \E_{q(\vx_0|\vx_t)}\![\vepsilon_t^2 \!] \!-\! \E_{q(\vx_0|\vx_t)}\![\vepsilon_t \!]^2 \! \right),
\end{align*}
where $\vepsilon_t = \frac{\vx_t - \sqrt{\alpha_{t|0}} \vx_0}{\sqrt{\beta_{t|0}}}$ is the noise used to generate $\vx_t$ from $\vx_0$ and $\tilde{\beta}_{s|t}= \frac{\beta_{s|0}}{\beta_{t|0}}\beta_{t|s}$.
\end{restatable}

Then, we extend Theorem~\ref{thm:opt_diag} to DPMs with continuous timesteps. Similarly to problem~{(\ref{eq:elbo_mean}) \& (\ref{eq:elbo_cov})}, we consider the following two optimizations w.r.t. the mean or the covariance solely:
\begin{align}
\mbox{Given arbitrary covariance }\mSigma_{s|t}, &\min\limits_{\vmu_{s|t}} \ell_{s,t}, \label{eq:kl_mean_c} \\
\mbox{Given arbitrary mean }\vmu_{s|t}, &\min\limits_{\mSigma_{s|t}} \ell_{s,t} \label{eq:kl_cov_c}.
\end{align}
The optimal mean to problem~{(\ref{eq:kl_mean_c})} is also Eq.~{(\ref{eq:opt_mean_c})}, which is irrelevant to $\mSigma_{s|t}$. The optimal covariance to problem~{(\ref{eq:kl_cov_c})} is given in Proposition~\ref{thm:opt_cov_c}.

\begin{restatable}{proposition}{optcovc}
\label{thm:opt_cov_c}
Suppose $\mSigma_{s|t}(\vx_t) = \diag(\vsigma_{s|t}(\vx_t)^2)$. For any mean $\vmu_{s|t}(\vx_t)$ that is parameterized by a noise prediction network $\hat{\vepsilon}_t(\vx_t)$ as in Eq.~{(\ref{eq:mu_param_c})}, the optimal covariance $\tilde{\vsigma}_{s|t}^*(\vx_t)^2$ to problem~{(\ref{eq:kl_cov_c})} is
\begin{align*}
\tilde{\vsigma}_{s|t}^*(\vx_t)^2 \!=\! & \vsigma_{s|t}^*(\vx_t)^2 \!+\! \frac{\beta_{t|s}^2}{\beta_{t|0}\alpha_{t|s}} \underbrace{(\hat{\vepsilon}_t(\vx_t) \!-\! \E_{q(\vx_0|\vx_t)}[\vepsilon_t])^2}_{\text{\normalsize{error}}} \\
= & \tilde{\beta}_{s|t} \vone + \frac{\beta_{t|s}^2}{\beta_{t|0}\alpha_{t|s}} \E_{q(\vx_0|\vx_t)}[(\vepsilon_t - \hat{\vepsilon}_t(\vx_t))^2],
\end{align*}
where $\vepsilon_t = \frac{\vx_t - \sqrt{\alpha_{t|0}} \vx_0}{\sqrt{\beta_{t|0}}}$ is the noise used to generate $\vx_t$ from $\vx_0$ and $\tilde{\beta}_{s|t}= \frac{\beta_{s|0}}{\beta_{t|0}}\beta_{t|s}$.
\end{restatable}

Similarly to Eq.~{(\ref{eq:opt_est}) \& (\ref{eq:err_est})}, to obtain estimates of $\vsigma_{s|t}^*(\vx_t)^2$ and $\tilde{\vsigma}_{s|t}^*(\vx_t)^2$, we can learn conditional expectations appeared in Proposition~{\ref{thm:opt_cov_c_no_err} \& \ref{thm:opt_cov_c}} by minimizing MSE losses. Please see details in Appendix~\ref{sec:detail_ext_c}.

\section{Experiments}
\label{sec:exp}
We evaluate \sndpm{} and \nprdpm{} in DPMs with both discrete and continuous timesteps\footnote{We provide our implementation at \url{https://github.com/baofff/Extended-Analytic-DPM}.}.

As for DPMs with discrete timesteps (see Section~\ref{sec:bg}), we consider the DDPM forward process (corresponding to $\lambda_n^2 = \tilde{\beta}_n$ in Eq.~{(\ref{eq:ddim})}), and the DDIM forward process (corresponding to $\lambda_n^2 = 0$ in Eq.~{(\ref{eq:ddim})}). We explicitly call our method \textit{\nprddpm} and \textit{\nprddim} (or \textit{\snddpm} and \textit{\snddim}) in these two cases respectively. We compare our method with the following baselines: (\romannumeral1) the \emph{original DDPM}~\citep{ho2020denoising}, which uses handcrafted values $\mSigma_n(\vx_n) = \tilde{\beta}_n \mI$ or $\mSigma_n(\vx_n) = \beta_n \mI$, (\romannumeral2) the \emph{original DDIM}~\citep{song2020denoising}, which uses handcrafted values $\mSigma_n(\vx_n) = \lambda_n^2 \mI = \vzero$, and (\romannumeral3) the \emph{Analytic-DPM}~\cite{bao2022analytic}, as introduced in Section~\ref{sec:bg}.

As for DPMs with continuous timesteps (see Section~\ref{sec:ext_cdpm}), we consider the VP SDE~\cite{song2020score} as the forward process.
We compare our method with the following baselines: (\romannumeral1) the \emph{Euler-Maruyama solver}~\citep{song2020score}, which firstly reverses the VP SDE and then discretizes the reverse one using the Euler-Maruyama solver, (\romannumeral2) the \emph{ancestral sampling}~\citep{song2020score}, which designs a Markov chain similar to Eq.~{(\ref{eq:markov})} and samples from the Markov chain, (\romannumeral3) the \emph{probability flow}~\citep{song2020score}, which firstly derives an equivalent ODE to the SDE and then discretizes the ODE, (\romannumeral4) the \emph{Analytic-DPM}~\cite{bao2022analytic}, 
as introduced in Section~\ref{sec:bg}, and (\romannumeral5) the \textit{``Gotta Go Fast'' SDE solver}~\citep{jolicoeur2021gotta}, which uses adaptive step sizes.

Since the performance on a subset of timesteps is crucial, we compare our method and baselines constrained on trajectories $1 \leq \tau_1 < \cdots < \tau_K = N$ with different number of timesteps $K$~\cite{song2020denoising,bao2022analytic} (see Appendix~\ref{sec:detail_fast} for details of how to apply our method to trajectories). Following \citet{bao2022analytic}, we consider two kinds of trajectories. The first one is the even trajectory (ET)~\cite{nichol2021improved}, where the timesteps are evenly spaced. The second one is the optimal trajectory (OT)~\cite{watson2021learning}, where the timesteps are determined by dynamic programming that maximizes ELBO.

We evaluate our method on six pretrained noise prediction networks provided by prior works~\cite{ho2020denoising,song2020score,song2020denoising,nichol2021improved,bao2022analytic}. Three of them are trained on CIFAR10~\cite{krizhevsky2009learning} with the linear schedule (LS) of $\beta_n$~\cite{ho2020denoising}, the cosine schedule (CS) of $\beta_n$~\cite{nichol2021improved} and the VP SDE~\cite{song2020score} respectively. We refer to the first two settings as CIFAR10 (LS) and CIFAR10 (CS) respectively, which have discrete timesteps, and refer to the last setting as CIFAR10 (VP SDE), which has continuous timesteps. The others are trained with discrete timesteps on CelebA 64x64~\citep{liu2015faceattributes}, ImageNet 64x64~\citep{deng2009imagenet} and LSUN Bedroom~\citep{yu15lsun} respectively. We train NPR (or SN) prediction networks for all of them, following the implementation in Section~\ref{sec:imp}. See more experimental details in Appendix~\ref{sec:exp_detail}.

\begin{table*}[t]
\centering
\caption{FID score $\downarrow$. All results are evaluated under the ET. Here A-DPM denotes the baseline Analytic-DPM. Note that the extra time cost of the NPR (or SN) prediction network is negligible on CIFAR10, CelebA 64x64, and at most 4.5\% on ImageNet 64x64 (see Appendix~\ref{sec:mem_time} for details). Thus, we can use the number of timesteps to compare the efficiency of these methods.}
\label{tab:fid}
\vskip 0.15in
\begin{small}
\begin{sc}
\setlength{\tabcolsep}{5.3pt}{
\begin{tabular}{lrrrrrrrrrrrr}
\toprule
& \multicolumn{6}{c}{CIFAR10 (LS)} & \multicolumn{6}{c}{CIFAR10 (CS)} \\
\cmidrule(lr){2-7} \cmidrule(lr){8-13}
\# timesteps $K$ & 
10 & 25 & 50 & 100 & 200 & 1000 &
10 & 25 & 50 & 100 & 200 & 1000 \\
\midrule
DDPM, $\tilde{\beta}_n$ & 
44.45 & 21.83 & 15.21 & 10.94 & 8.23 & 5.11 & 
34.76 & 16.18 & 11.11 & 8.38 & 6.66 & 4.92 \\
DDPM, $\beta_n$ & 
233.41 & 125.05 & 66.28 & 31.36 & 12.96 & \textbf{3.04} &
205.31 & 84.71 & 37.35 & 14.81 & 5.74 & \textbf{3.34} \\
A-DDPM & 
34.26 & 11.60 & 7.25 & 5.40 & 4.01 & 4.03 &
22.94 & 8.50 & 5.50 & 4.45 & 4.04 & 4.31 \\
\nprddpm & 
32.35 & 10.55 & 6.18 & 4.52 & 3.57 & 4.10 &
19.94 & 7.99 & 5.31 & 4.52 & 4.10 & 4.27 \\
\snddpm &
\textbf{24.06} & \textbf{6.91} & \textbf{4.63} & \textbf{3.67} & \textbf{3.31} & 3.65 & 
\textbf{16.33} & \textbf{6.05} & \textbf{4.17} & \textbf{3.83} & \textbf{3.72} & 4.07 \\
\arrayrulecolor{black!30}\midrule
DDIM & 
21.31 & 10.70 & 7.74 & 6.08 & 5.07 & 4.13 &
34.34 & 16.68 & 10.48 & 7.94 & 6.69 & 4.89 \\
A-DDIM & 
14.00 & 5.81 & 4.04 & 3.55 & 3.39 & 3.74 &
26.43 & 9.96 & 6.02 & 4.88 & 4.92 & 4.66 \\
\nprddim & 
13.34 & 5.38 & 3.95 & 3.53 & 3.42 & 3.72 &
22.81 & 9.47 & 6.04 & 5.02 & 5.06 & 4.62 \\
\snddim & 
\textbf{12.19} & \textbf{4.28} & \textbf{3.39} & \textbf{3.23} & \textbf{3.22} & \textbf{3.65} &
\textbf{17.90} & \textbf{7.36} & \textbf{5.16} & \textbf{4.63} & \textbf{4.63} & \textbf{4.51} \\
\arrayrulecolor{black}\midrule
\end{tabular}}
\end{sc}
\end{small}
\vskip 0.05in
\begin{small}
\begin{sc}
\setlength{\tabcolsep}{5.3pt}{
\begin{tabular}{lrrrrrrrrrrrr}
\toprule
& \multicolumn{6}{c}{CelebA 64x64} & \multicolumn{6}{c}{ImageNet 64x64} \\
\cmidrule(lr){2-7} \cmidrule(lr){8-13}
\# timesteps $K$ & 
10 & 25 & 50 & 100 & 200 & 1000 &
25 & 50 & 100 & 200 & 400 & 4000 \\
\midrule
DDPM, $\tilde{\beta}_n$ & 
36.69 & 24.46 & 18.96 & 14.31 & 10.48 & 5.95 & 
29.21 & 21.71 & 19.12 & 17.81 & 17.48 & 16.55 \\
DDPM, $\beta_n$ & 
294.79 & 115.69 & 53.39 & 25.65 & 9.72 & \textbf{3.16} &
170.28 & 83.86 & 45.04 & 28.39 & 21.38 & 16.38 \\
A-DDPM & 
28.99 & 16.01 & 11.23 & 8.08 & 6.51 & 5.21 &
32.56 & 22.45 & 18.80 & 17.16 & 16.40 & 16.34 \\
\nprddpm & 
28.37 & 15.74 & 10.89 & 8.23 & 7.03 & 5.33 &
28.27 & 20.89 & 18.06 & 16.96 & \textbf{16.32} & 16.38 \\
\snddpm & 
\textbf{20.60} & \textbf{12.00} & \textbf{7.88} & \textbf{5.89} & \textbf{5.02} & 4.42 &
\textbf{27.58} & \textbf{20.74} & \textbf{18.04} & \textbf{16.61} & 16.37 & \textbf{16.22}
\\
\arrayrulecolor{black!30}\midrule
DDIM & 
20.54 & 13.45 & 9.33 & 6.60 & 4.96 & 3.40 &
26.06 & 20.10 & 18.09 & 17.84 & 17.74 & 19.00 \\
A-DDIM & 
15.62 & 9.22 & 6.13 & 4.29 & 3.46 & 3.13 &
\textbf{25.98} & \textbf{19.23} & 17.73 & 17.49 & 17.44 & 18.98 \\
\nprddim & 
14.98 & 8.93 & 6.04 & 4.27 & 3.59 & 3.15 &
28.84 & 19.62 & 17.63 & 17.42 & 17.30 & 18.91 \\
\snddim &
\textbf{10.20} & \textbf{5.48} & \textbf{3.83} & \textbf{3.04} & \textbf{2.85} & \textbf{2.90} &
28.07 & 19.38 & \textbf{17.53} & \textbf{17.23} & \textbf{17.23} & \textbf{18.89} \\
\arrayrulecolor{black}\midrule
\end{tabular}}
\end{sc}
\end{small}
\vskip 0.05in
\begin{small}
\begin{sc}
\begin{tabular}{lrrrrrr}
\toprule
& \multicolumn{6}{c}{CIFAR10 (VP SDE)} \\
\cmidrule(lr){2-7} 
\# timesteps $K$ & 
10 & 25 & 50 & 100 & 200 & 1000 \\
\midrule
Euler-Maruyama & 
292.20 & 170.17 & 90.79 & 47.46 & 21.92 & \textbf{2.55} \\
Ancestral sampling & 
235.28 & 129.29 & 68.52 & 31.99 & 12.81 & 2.72 \\
Probability flow &
107.74 & 21.34 & 7.78 & 4.33 & 3.27 & 2.82 \\
A-DPM & 
35.10 & 11.57 & 6.54 & 4.71 & 3.61 & 2.98 \\
\nprdpm & 
33.70 & 10.44 & 5.83 & 3.97 & 3.05 & 3.04 \\
\sndpm &
\textbf{25.30} & \textbf{7.34} & \textbf{4.46} & \textbf{3.27} & \textbf{2.83} & 2.71 \\
\arrayrulecolor{black}\midrule
\end{tabular}
\end{sc}
\end{small}
\vspace{-.1cm}
\end{table*}

\begin{table*}[t]
\centering
\caption{The least number of timesteps $\downarrow$ required to achieve a FID around 6 (with the corresponding FID). The time cost in a single timestep is the same for all baselines. Since our methods have extra prediction networks which require extra time cost, we multiply their results by the ratio of the time cost in a single timestep in them to that in baselines. See experimental details in Appendix~\ref{sec:detail_timesteps}.}
\label{tab:timesteps}
\vskip 0.15in
\begin{small}
\begin{sc}
\begin{tabular}{lrrr}
\toprule
    Method & CIFAR10 & CelebA 64x64 & LSUN Bedroom \\
\midrule
   DDPM~\citep{ho2020denoising} & 90 (6.12) & $>$ 200 & 130 (6.06)\\
   DDIM~\citep{song2020denoising} & 30 (5.85) & $>$ 100 & Best FID $>$ 6 \\
   Improved DDPM~\citep{nichol2021improved} & 45 (5.96) & Missing model & \textbf{90} (6.02) \\
   Analytic-DPM~\citep{bao2022analytic} & 25 (5.81) & 55 (5.98) & 100 (6.05) \\
   \nprdpm{} (ours) & 1.002$\times$23 (5.76) & 1.011$\times$50 (6.04) & 1.017$\times$\textbf{90} (6.01) \\
   \sndpm{} (ours) & 1.005$\times$\textbf{17} (5.81) & 1.012$\times$\textbf{22} (5.96) & 1.100$\times$92 (6.02)\\
\bottomrule
\end{tabular}
\end{sc}
\end{small}
\vspace{-.1cm}
\end{table*}

\subsection{Sample Quality}
\label{sec:exp_sample}

In this part, we compare the sample quality quantitatively, as measured by the commonly used FID score~\citep{heusel2017gans}. We evaluate DPMs with both the discrete and continuous timesteps. As for DPMs with discrete timesteps, we report results under both the DDPM and DDIM forward processes. Since \citet{watson2021learning,bao2022analytic} report that a DPM with the optimal trajectory is likely to have a worse FID score, we only include results under the even trajectory.

As shown in Table~\ref{tab:fid}, both \nprdpm{} and \sndpm{} outperform baselines in most cases. In particular, \sndpm{} improves the sample quality remarkably on CIFAR10 and CelebA 64x64 when the number of timesteps is small. Moreover, our methods perform well in DPMs with both discrete and continuous timesteps.

Following \citet{bao2022analytic}, we also compare the least number of timesteps required to achieve a FID around 6, as shown in Table~\ref{tab:timesteps}. Remarkably, on CelebA 64x64, \sndpm{} requires less than half the number of timesteps compared to baselines. In Appendix~\ref{sec:norm_cmp}, we show FID comparison after normalizing for the extra cost, as well as the comparison with the ``Gotta Go Fast'' SDE solver. In Appendix~\ref{sec:samples}, we show generated samples of our methods.

\begin{table*}[t]
\centering
\caption{Upper bound ($-\lelbo$) on the negative log-likelihood (bits/dim) $\downarrow$. Here A-DPM denotes the baseline Analytic-DPM. Note that the extra time cost of the NPR prediction network is negligible on CIFAR10, CelebA 64x64, and at most 4.5\% on ImageNet 64x64 (see Appendix~\ref{sec:mem_time} for details). Thus, we can use the number of timesteps to compare the efficiency of these methods. To reduce the effect of randomness, we also repeat the experiments under the full timesteps and report the variance in Appendix~\ref{sec:ll_var}.}
\label{tab:bpd}
\vskip 0.15in
\begin{small}
\begin{sc}
\setlength{\tabcolsep}{5.5pt}{
\begin{tabular}{llrrrrrrrrrrrrr}
\toprule
& & \multicolumn{6}{c}{CIFAR10 (LS)} & \multicolumn{6}{c}{CIFAR10 (CS)} \\
\cmidrule(lr){3-8} \cmidrule(lr){9-14}
% \arrayrulecolor{black!30}\midrule
\multicolumn{2}{l}{\# timesteps $K$} & 
10 & 25 & 50 & 100 & 200 & 1000 &
10 & 25 & 50 & 100 & 200 & 1000 \\
\midrule
\multirow{4}{*}{ET} & DDPM, $\tilde{\beta}_n$ & 
74.95 & 24.98 & 12.01 & 7.08 & 5.03 & 3.73 & 
75.96 & 24.94 & 11.96 & 7.04 & 4.95 & 3.60 \\
& DDPM, $\beta_n$ & 
6.99 & 6.11 & 5.44 & 4.86 & 4.39 & 3.75 & 
6.51 & 5.55 & 4.92 & 4.41 & 4.03 & 3.54 \\
& A-DDPM & 
5.47 & 4.79 & 4.38 & 4.07 & 3.84 & 3.59 & 
5.08 & 4.45 & 4.09 & 3.83 & 3.64 & 3.42 \\
& \nprddpm{} & 
\textbf{5.40} & \textbf{4.64} & \textbf{4.25} & \textbf{3.98} & \textbf{3.79} & \textbf{3.57} & 
\textbf{5.03} & \textbf{4.33} & \textbf{3.99} & \textbf{3.76} & \textbf{3.59} & \textbf{3.41} \\
\arrayrulecolor{black!30}\midrule
\multirow{3}{*}{OT} & DDPM, $\beta_n$ & 
5.38 & 4.34 & 3.97 & 3.82 & 3.77 & 3.75 & 
5.51 & 4.30 & 3.86 & 3.65 & 3.57 & 3.54 \\
&  A-DDPM & 
4.11 & 3.68 & 3.61 & 3.59 & 3.59 & 3.59 & 
3.99 & 3.56 & 3.47 & 3.44 & 3.43 & 3.42 \\
& \nprddpm{} & 
\textbf{3.91} & \textbf{3.64} & \textbf{3.59} & \textbf{3.58} & \textbf{3.57} & \textbf{3.57} & 
\textbf{3.88} & \textbf{3.52} & \textbf{3.45} & \textbf{3.42} & \textbf{3.41} & \textbf{3.41} \\
\arrayrulecolor{black}\midrule
\end{tabular}}
\end{sc}
\end{small}
\vskip 0.05in
\begin{small}
\begin{sc}
\setlength{\tabcolsep}{5.5pt}{
\begin{tabular}{llrrrrrrrrrrrrr}
\toprule
& & \multicolumn{6}{c}{CelebA 64x64} & \multicolumn{6}{c}{ImageNet 64x64} \\
\cmidrule(lr){3-8} \cmidrule(lr){9-14}
% \arrayrulecolor{black!30}\midrule
\multicolumn{2}{l}{\# timesteps $K$} & 
10 & 25 & 50 & 100 & 200 & 1000 &
25 & 50 & 100 & 200 & 400 & 4000 \\
\midrule
\multirow{4}{*}{ET} & DDPM, $\tilde{\beta}_n$ &
33.42 & 13.09 & 7.14 & 4.60 & 3.45 & 2.71 & 
105.87 & 46.25 & 22.02 & 12.10 & 7.59 & 3.89 \\
& DDPM, $\beta_n$ & 
6.67 & 5.72 & 4.98 & 4.31 & 3.74 & 2.93 & 
5.81 & 5.20 & 4.70 & 4.31 & 4.04 & 3.65 \\
& A-DDPM & 
4.54 & 3.89 & 3.48 & 3.16 & 2.92 & 2.66 &
4.78 & 4.42 & 4.15 & 3.95 & 3.81 & 3.61 \\
& \nprddpm{} &
\textbf{4.46} & \textbf{3.78} & \textbf{3.40} & \textbf{3.11} & \textbf{2.89} & \textbf{2.65} &
\textbf{4.66} & \textbf{4.22} & \textbf{3.96} & \textbf{3.80} & \textbf{3.71} & \textbf{3.60} \\
\arrayrulecolor{black!30}\midrule
\multirow{3}{*}{OT} & DDPM, $\beta_n$ & 
4.76 & 3.58 & 3.16 & 2.99 & 2.94 & 2.93 &
4.56 & 4.09 & 3.84 & 3.73 & 3.68 & 3.65 \\
& A-DDPM &
2.97 & 2.71 & 2.67 & 2.66 & 2.66 & 2.66 &
3.83 & 3.70 & 3.64 & 3.62 & 3.62 & 3.61 \\
& \nprddpm{} & 
\textbf{2.88} & \textbf{2.69} & \textbf{2.66} & \textbf{2.66} & \textbf{2.65} & \textbf{2.65} & 
\textbf{3.73} & \textbf{3.65} & \textbf{3.62} & \textbf{3.60} & \textbf{3.60} & \textbf{3.60} \\
\arrayrulecolor{black}\midrule
\end{tabular}}
\end{sc}
\end{small}
\vspace{-.2cm}
\end{table*}

\subsection{Likelihood Results}

In this part, we evaluate the upper bound $-\lelbo$ (see Eq.~{(\ref{eq:elbo})}) on the negative log-likelihood of DPMs with discrete timesteps. Since \citet{bao2022analytic} claim that $-\lelbo = \infty$ in the DDIM forward process, we only include results under the DDPM forward process. 

As shown in Table~\ref{tab:bpd}, our \nprdpm{} consistently outperforms all baselines on all datasets, all number of timesteps and both kinds of trajectories.

In initial experiments, we find that \sndpm{} doesn't perform well on the likelihood results, potentially due to not considering the imperfect mean and the amplified error of $\hat{\vepsilon}_n(\vx_n)^2$ as mentioned in Section~\ref{sec:flex_cov}. The inconsistency with its FID results in Table~\ref{tab:fid} roots in the different natures of the two metrics, which has been observed and explored extensively in prior works~\citep{ho2020denoising,nichol2021improved,song2021maximum,vahdat2021score,watson2021learning,kingma2021variational,bao2022analytic}. We also note that while NPR-DPM performs better than SN-DPM w.r.t. likelihood, it does not perform better w.r.t. FID. Please see Appendix~\ref{sec:not_always_better} for a discussion.

We also compare to Improved DDPM~\cite{nichol2021improved} in terms of likelihood estimation. With a small number of timesteps, our method can get a much better likelihood result than \citet{nichol2021improved} on ImageNet 64x64. With the full timesteps, two methods show a similar likelihood performance (the upper bound on NLL of our method is 3.60 and that of Improved DDPM is 3.57), which implies that using an alternative objective such as the MSE loss can achieve a similar likelihood performance to directly optimizing $\lelbo$. See details in Appendix~\ref{sec:ll_iddpm}.

\section{Related Work}

\textbf{DPMs and their variants.} The idea that a data generation process can be constructed by reversing a diffusion process is initially introduced by \citet{sohl2015deep}. Specifically, \citet{sohl2015deep} reverse the diffusion process using a Markov chain with discrete timesteps, which is learned on the ELBO objective. \citet{ho2020denoising} propose to parameterize the mean of the Markov chain by the noise prediction network, which share parameters between different timesteps and is learned on a MSE loss. \citet{song2020score} consider DPMs with infinitesimal timesteps, which can be represented by stochastic differential equations (SDEs). Recently, some variants of DPMs are proposed. \citet{kingma2021variational} additionally learn the forward diffusion process, \citet{vahdat2021score} learn DPMs in latent space and \citet{dockhorn2021score} introduce auxiliary velocity variables to the diffusion process.

Based on the powerful generation performance, DPMs have shown great promise in a wide variety of applications, such as controllable generation~\cite{choi2021ilvr,meng2021sdedit,sinha2021d2c,nichol2021glide}, voice conversion~\cite{popov2021diffusion}, image super-resolution~\cite{saharia2021image,li2021srdiff}, image-to-image translation~\citep{sasaki2021unit}, shape generation~\citep{zhou20213d}, 3d point cloud generation~\citep{luo2021diffusion} and time series forecasting~\citep{rasul2021autoregressive}.

\textbf{Covariance design in DPMs.} As mentioned in Section~\ref{sec:intro}, the covariance design in DPMs is crucial for the performance on a subset of timesteps, and some prior works~\cite{ho2020denoising,song2020denoising,bao2022analytic} use isotropic covariances that only depends
on the timestep without considering the state, which is less flexible compared to ours. Besides, we note that \citet{nichol2021improved} also use state-dependent diagonal covariances. The difference is that \citet{nichol2021improved} directly train a covariance network on $\lelbo$, and our one is learned by optimizing a MSE loss based on the form of the optimal covariance. As a result, our method can get a reasonable likelihood result on a small number of timesteps, which is contrast to \citet{nichol2021improved} (see Appendix~\ref{sec:ll_iddpm}).

\textbf{Faster DPMs.} In addition to exploring covariance design in DPMs, there are also other attempts toward faster DPMs. Several works explore to find short trajectories by applying searching algorithms, e.g., grid search~\cite{chen2020wavegrad}, dynamic programming~\cite{watson2021learning} and differentiable search~\cite{anonymous2022optimizing}. These searching algorithms can be further combined with our method to get a better performance. For example, the combination of the dynamic programming and our method leads to a better likelihood performance, as shown by the OT results in Table~\ref{tab:bpd}. Several works change the model family in the reverse process, replacing the Gaussian model. \citet{luhman2021knowledge,anonymous2022progressive} distill the original reverse process to a model of a single or a few timesteps, and \citet{xiao2021tackling} model the reverse process as a conditional generator and propose to train with an adversarial loss. However, their likelihood estimation is non-trivial due to the change of the model family. Several works design faster solvers~\cite{jolicoeur2021gotta,popov2021diffusion} for DPMs with continuous timesteps. However, these solvers only consider isotropic covariances as in \citet{ho2020denoising,bao2022analytic}, which are less flexible compared to ours.

\vspace{-.1cm}
\section{Conclusion}
\vspace{-.1cm}
We consider diagonal and full covariances to improve the expressive power of DPMs. We derive the optimal mean and covariance in terms of the maximum likelihood, and also correct it given an imperfect mean. Both the optimal and the corrected ones can be decomposed into terms of conditional expectations over functions of noise, which can be estimated by minimizing MSE losses. We consider diagonal covariances in our implementation for computational efficiency. Besides, we adopt a parameter sharing scheme for inference efficiency and a two-stage training process motivated by our theory. Our method is applicable to DPMs with both discrete and continuous timesteps. Empirically, our method consistently outperforms a wide variety of baselines on likelihood estimation, and improves the sample quality especially on a small number of timesteps.

% Authors are kindly asked to make their submissions as accessible as possible for everyone including people with disabilities and sensory or neurological differences.
% Tips of how to achieve this and what to pay attention to will be provided on the conference website \url{http://icml.cc/}.

% \section*{Software and Data}

% If a paper is accepted, we strongly encourage the publication of software and data with the
% camera-ready version of the paper whenever appropriate. This can be
% done by including a URL in the camera-ready copy. However, \textbf{do not}
% include URLs that reveal your institution or identity in your
% submission for review. Instead, provide an anonymous URL or upload
% the material as ``Supplementary Material'' into the CMT reviewing
% system. Note that reviewers are not required to look at this material
% when writing their review.

\section*{Acknowledgements}

This work was supported by NSF of China Projects (Nos. 62061136001, 61620106010, 62076145, U19B2034, U1811461, U19A2081, 6197222); Beijing NSF Project (No. JQ19016); Beijing Outstanding Young Scientist Program NO. BJJWZYJH012019100020098; a grant from Tsinghua Institute for Guo Qiang; the NVIDIA NVAIL Program with GPU/DGX Acceleration; the High Performance Computing Center, Tsinghua University; and Major Innovation \& Planning Interdisciplinary Platform for the ``Double-First Class" Initiative, Renmin University of China.

% \textbf{Do not} include acknowledgements in the initial version of
% the paper submitted for blind review.

% If a paper is accepted, the final camera-ready version can (and
% probably should) include acknowledgements. In this case, please
% place such acknowledgements in an unnumbered section at the
% end of the paper. Typically, this will include thanks to reviewers
% who gave useful comments, to colleagues who contributed to the ideas,
% and to funding agencies and corporate sponsors that provided financial


\bibliography{example_paper}
\bibliographystyle{icml2022}

%%%%%%%%%%%%%%%%%%%%%%%%%%%%%%%%%%%%%%%%%%%%%%%%%%%%%%%%%%%%%%%%%%%%%%%%%%%%%%%
%%%%%%%%%%%%%%%%%%%%%%%%%%%%%%%%%%%%%%%%%%%%%%%%%%%%%%%%%%%%%%%%%%%%%%%%%%%%%%%
% APPENDIX
%%%%%%%%%%%%%%%%%%%%%%%%%%%%%%%%%%%%%%%%%%%%%%%%%%%%%%%%%%%%%%%%%%%%%%%%%%%%%%%
%%%%%%%%%%%%%%%%%%%%%%%%%%%%%%%%%%%%%%%%%%%%%%%%%%%%%%%%%%%%%%%%%%%%%%%%%%%%%%%
\newpage
\appendix
\onecolumn

% You can have as much text here as you want. The main body must be at most $8$ pages long.
% For the final version, one more page can be added.
% If you want, you can use an appendix like this one, even using the one-column format.
%%%%%%%%%%%%%%%%%%%%%%%%%%%%%%%%%%%%%%%%%%%%%%%%%%%%%%%%%%%%%%%%%%%%%%%%%%%%%%%
%%%%%%%%%%%%%%%%%%%%%%%%%%%%%%%%%%%%%%%%%%%%%%%%%%%%%%%%%%%%%%%%%%%%%%%%%%%%%%%
\section{Proof}
\label{sec:proof}

\subsection{Proof of Theorem~{\ref{thm:opt_diag_no_err}, \ref{thm:opt_mean} \& \ref{thm:opt_diag}}}

\begin{lemma}
\label{thm:opt_sigma_general}
(Conditioned moment matching) 
Suppose $q(\vx)$ is a probability density function with mean $\vmu_q$ and covariance matrix $\mSigma_q$ and $p(\vx) = \gN(\vx|\vmu, \diag(\vsigma^2))$ is a Gaussian distribution. Then 
\begin{enumerate}
    \item for any $\vsigma^2$, the optimization problem $\min\limits_{\vmu} \KL(q\|p)$ admits an optimal solution $\vmu^* = \vmu_q$;
    \item for any $\vmu$, the optimization problem $\min\limits_{\vsigma^2} \KL(q\|p)$ admits an optimal solution $\tilde{\vsigma}^{*2} = \diag(\mSigma_q) + (\vmu - \vmu_q)^2$, where $(\cdot)^2$ is the element-wise square;
    \item the optimization problem $\min\limits_{\vmu, \vsigma^2} \KL(q\|p)$ admits an optimal solution $\vmu^* = \vmu_q$ and $\vsigma^{*2} = \diag(\mSigma_q)$.
\end{enumerate}
\end{lemma}

\begin{proof}
\citet{bao2022analytic} show that the KL divergence between a probability density function $q$ and a Gaussian distribution $p$ can be written as $\KL(q\|p) = \KL(\gN(\vx|\vmu_q, \mSigma_q) \| p) + c$, where $c$ is a constant only related to $q$ (see Lemma 2 of \citet{bao2022analytic}). Thus, the optimizations of $\KL(q\|p)$ and $\KL(\gN(\vx|\vmu_q, \mSigma_q) \| p)$ w.r.t. $\vmu$ or $\vsigma^2$ are equivalent.

We expand $\KL(\gN(\vx|\vmu_q, \mSigma_q) \| p)$ according to the definition of KL divergence:
\begin{align}
    & 2 \KL(\gN(\vx|\vmu_q, \mSigma_q) \| p) = \tr(\diag(\vsigma^{-2}) \mSigma_q) - d + \log \frac{|\diag(\vsigma^2)|}{|\mSigma_q|} + (\vmu - \vmu_q)^\top \diag(\vsigma^{-2}) (\vmu - \vmu_q) \label{eq:kl_expand_mu}\\
    = & \sum\limits_{i=1}^d \left\{ \sigma_i^{-2} \left[ (\mSigma_{q})_{i,i} + (\vmu_i - (\vmu_q)_i)^2 \right] + \log \sigma_i^2 \right\} - \log |\mSigma_q| - d. \label{eq:kl_expand_cov}
\end{align}
According to Eq.~{(\ref{eq:kl_expand_mu})}, we know that
\begin{align*}
    \min\limits_{\vmu} \KL(q\|p) \Leftrightarrow \min\limits_{\vmu} (\vmu - \vmu_q)^\top \diag(\vsigma^{-2}) (\vmu - \vmu_q),
\end{align*}
which admits an optimal solution $\vmu^* = \vmu_q$.

According to Eq.~{(\ref{eq:kl_expand_cov})}, we know that
\begin{align*}
    \min\limits_{\vsigma^2} \KL(q\|p) \Leftrightarrow \min\limits_{\vsigma^2} \sum\limits_{i=1}^d \left\{ \sigma_i^{-2} \left[ (\mSigma_{q})_{i,i} + (\vmu_i - (\vmu_q)_i)^2 \right] + \log \sigma_i^2 \right\}
\end{align*}

By taking gradients, we know $\sigma_i^{-2} \left[ (\mSigma_{q})_{i,i} + (\vmu_i - (\vmu_q)_i)^2 \right] + \log \sigma_i^2$ as a function of $\sigma_i^2$ gets its minimum at
\begin{align*}
    \tilde{\sigma}_i^{*2} = (\mSigma_{q})_{i,i} + (\vmu_i - (\vmu_q)_i)^2.
\end{align*}
Thus, $\min\limits_{\vsigma^2} \KL(q\|p)$ admits an optimal solution $\tilde{\vsigma}^{*2} = \diag(\mSigma_{q}) + (\vmu - \vmu_q)^2$.

Combining $\vmu^*$ and $\tilde{\vsigma}^{*2}$, we know that $\min\limits_{\vmu,\vsigma^2} \KL(q\|p)$ admits an optimal solution $\vmu^* = \vmu_q$ and $\vsigma^{*2} = \diag(\mSigma_{q}) + (\vmu^* - \vmu_q)^2 = \diag(\mSigma_{q})$.
\end{proof}

\begin{lemma}
\label{thm:mm}
Suppose $q(\vx_{0:N})$ is defined as Eq.~{(\ref{eq:ddim})}. Then we have
\begin{align*}
    & \E_{q(\vx_{n-1}|\vx_n)} [\vx_{n-1}] = \tilde{\vmu}_n(\vx_n, \frac{1}{\sqrt{\overline{\alpha}_n}} ( \vx_n - \sqrt{\overline{\beta}_n} \E_{q(\vx_0|\vx_n)}[\vepsilon_n] )), \\
    & \Cov_{q(\vx_{n-1}|\vx_n)}[\vx_{n-1}] = \lambda_n^2 \mI + \gamma_n^2 \frac{\overline{\beta}_n}{\overline{\alpha}_n} \left( \E_{q(\vx_0|\vx_n)}[\vepsilon_n \vepsilon_n^\top] - \E_{q(\vx_0|\vx_n)}[\vepsilon_n] \E_{q(\vx_0|\vx_n)}[\vepsilon_n]^\top \right), \\
    & \diag(\Cov_{q(\vx_{n-1}|\vx_n)}[\vx_{n-1}]) = \lambda_n^2 \vone + \gamma_n^2 \frac{\overline{\beta}_n}{\overline{\alpha}_n} \left(\E_{q(\vx_0|\vx_n)}[\vepsilon_n^2] - \E_{q(\vx_0|\vx_n)}[\vepsilon_n]^2\right).
\end{align*}
where $\vepsilon_n = \frac{\vx_n - \sqrt{\overline{\alpha}_n} \vx_0}{\sqrt{\overline{\beta}_n}}$ is the noise used to generate $\vx_n$ from $\vx_0$ and $\gamma_n = \sqrt{\overline{\alpha}_{n-1}} - \sqrt{\overline{\beta}_{n-1} - \lambda_n^2} \sqrt{\frac{\overline{\alpha}_n}{\overline{\beta}_n}}$.
\end{lemma}

\begin{proof}
Since $\tilde{\vmu}_n(\vx_n, \vx_0)$ is linear w.r.t. $\vx_0$, we have
\begin{align*}
    & \E_{q(\vx_{n-1}|\vx_n)} [\vx_{n-1}] = \E_{q(\vx_0|\vx_n)} \E_{q(\vx_{n-1}|\vx_n, \vx_0)}[\vx_{n-1}] \\
    = & \E_{q(\vx_0|\vx_n)} \tilde{\vmu}_n(\vx_n, \vx_0) = \tilde{\vmu}_n(\vx_n, \E_{q(\vx_0|\vx_n)} \vx_0) = \tilde{\vmu}_n(\vx_n, \frac{1}{\sqrt{\overline{\alpha}_n}} ( \vx_n - \sqrt{\overline{\beta}_n} \E_{q(\vx_0|\vx_n)}[\vepsilon_n] )).
\end{align*}
As for $\Cov_{q(\vx_{n-1}|\vx_n)}[\vx_{n-1}]$, \citet{bao2022analytic} show that it can be expanded as $\Cov_{q(\vx_{n-1}|\vx_n)}[\vx_{n-1}] = \lambda_n^2 \mI + \gamma_n^2 \Cov_{q(\vx_0|\vx_n)}[\vx_0]$,
where $\gamma_n = \sqrt{\overline{\alpha}_{n-1}} - \sqrt{\overline{\beta}_{n-1} - \lambda_n^2} \sqrt{\frac{\overline{\alpha}_n}{\overline{\beta}_n}}$ (see Lemma 13 in \citet{bao2022analytic}).

Besides, we have $\Cov_{q(\vx_0|\vx_n)} [\vepsilon_n] = \Cov_{q(\vx_0|\vx_n)} [\frac{\vx_n - \sqrt{\overline{\alpha}_n} \vx_0}{ \sqrt{\overline{\beta}_n}}] = \frac{\overline{\alpha}_n}{\overline{\beta}_n} \Cov_{q(\vx_0|\vx_n)} [\vx_0]$.
Thus,
\begin{align*}
    & \Cov_{q(\vx_{n-1}|\vx_n)}[\vx_{n-1}] = \lambda_n^2 \mI + \gamma_n^2 \frac{\overline{\beta}_n}{\overline{\alpha}_n} \Cov_{q(\vx_0|\vx_n)}[\vepsilon_n] \\
    = & \lambda_n^2 \mI + \gamma_n^2 \frac{\overline{\beta}_n}{\overline{\alpha}_n} \left( \E_{q(\vx_0|\vx_n)}[\vepsilon_n \vepsilon_n^\top] - \E_{q(\vx_0|\vx_n)}[\vepsilon_n] \E_{q(\vx_0|\vx_n)}[\vepsilon_n]^\top \right), \\
    & \diag(\Cov_{q(\vx_{n-1}|\vx_n)}[\vx_{n-1}]) = \lambda_n^2 \vone + \gamma_n^2 \frac{\overline{\beta}_n}{\overline{\alpha}_n} \left(\E_{q(\vx_0|\vx_n)}[\vepsilon_n^2] - \E_{q(\vx_0|\vx_n)}[\vepsilon_n]^2\right).
\end{align*}
\end{proof}

\optdiagnoerr*
\begin{proof}
\citet{bao2022analytic} show that the KL divergence between a joint probability density function $q(\vx_{0:N})$ and a Markov chain $p(\vx_{0:N})$ can be written as
\begin{align*}
    \KL(q(\vx_{0:N}) \| p(\vx_{0:N})) = \sum\limits_{n=1}^N \E_q \KL(q(\vx_{n-1}|\vx_n) \| p(\vx_{n-1}|\vx_n)) + c,
\end{align*}
where $c$ is a constant only related to $q$ (see Lemma 8 in \citet{bao2022analytic}). As a result, the optimization of the KL divergence is decomposed into $n$ independent optimization sub-problems:
\begin{align*}
    \min\limits_{\vmu_n(\cdot), \vsigma_n(\cdot)^2} \E_q \KL(q(\vx_{n-1}|\vx_n) \| p(\vx_{n-1}|\vx_n)), \quad 1 \leq n \leq N.
\end{align*}

Since $q(\vx_{n-1}|\vx_n)$ is a probability density function with mean $\E_{q(\vx_{n-1}|\vx_n)}[\vx_{n-1}]$ and covariance matrix $\Cov_{q(\vx_{n-1}|\vx_n)}[\vx_{n-1}]$ and $p(\vx_{n-1}|\vx_n) = \gN(\vx_{n-1}|\vmu_n(\vx_n), \diag(\vsigma_n(\vx_n)^2))$ is a Gaussian distribution, according to Lemma~\ref{thm:opt_sigma_general} and Lemma~\ref{thm:mm}, we know the optimal solution is
\begin{align*}
    & \vmu_n^*(\vx_n) = \E_{q(\vx_{n-1}|\vx_n)}[\vx_{n-1}] = \tilde{\vmu}_n(\vx_n, \frac{1}{\sqrt{\overline{\alpha}_n}} ( \vx_n - \sqrt{\overline{\beta}_n} \E_{q(\vx_0|\vx_n)}[\vepsilon_n] )), \\
    & \vsigma_n^*(\vx_n)^2 = \diag(\Cov_{q(\vx_{n-1}|\vx_n)}[\vx_{n-1}]) = \lambda_n^2 \vone + \gamma_n^2 \frac{\overline{\beta}_n}{\overline{\alpha}_n} \left(\E_{q(\vx_0|\vx_n)}[\vepsilon_n^2] - \E_{q(\vx_0|\vx_n)}[\vepsilon_n]^2\right).
\end{align*}
\end{proof}

\optmean*

\begin{proof}
Similarly to the proof in Theorem~\ref{thm:opt_diag_no_err}, these optimizations are equivalent
\begin{align*}
\max\limits_{\{\vmu_n\}_{n=1}^N} \lelbo \Leftrightarrow \min\limits_{\{\vmu_n\}_{n=1}^N} \KL(q(\vx_{0:N}) \| p(\vx_{0:N})) \Leftrightarrow \min\limits_{\vmu_n} \E_q \KL(q(\vx_{n-1}|\vx_n) \| p(\vx_{n-1}|\vx_n)), \quad 1 \leq n \leq N.
\end{align*}

According to Lemma~\ref{thm:opt_sigma_general} and Lemma~\ref{thm:mm}, we know the optimal $\vmu_n^*(\vx_n)$ is
\begin{align*}
    \vmu_n^*(\vx_n) = \E_{q(\vx_{n-1}|\vx_n)}[\vx_{n-1}] = \tilde{\vmu}_n \left(\vx_n, \frac{1}{\sqrt{\overline{\alpha}_n}} ( \vx_n - \sqrt{\overline{\beta}_n} \E_{q(\vx_0|\vx_n)} [\vepsilon_n]) \right).
\end{align*}
\end{proof}

\optdiag*

\begin{proof}
Similarly to the proof in Theorem~\ref{thm:opt_diag_no_err}, these optimizations are equivalent
\begin{align*}
\max\limits_{\{\mSigma_n\}_{n=1}^N} \lelbo \Leftrightarrow \min\limits_{\{\mSigma_n\}_{n=1}^N} \KL(q(\vx_{0:N}) \| p(\vx_{0:N})) \Leftrightarrow \min\limits_{\mSigma_n} \E_q \KL(q(\vx_{n-1}|\vx_n) \| p(\vx_{n-1}|\vx_n)), \quad 1 \leq n \leq N.
\end{align*}

According to Lemma~\ref{thm:opt_sigma_general}, we know the optimal $\tilde{\vsigma}_n^*(\vx_n)^2$ is
\begin{align*}
    \tilde{\vsigma}_n^*(\vx_n)^2 = \diag(\Cov_{q(\vx_{n-1}|\vx_n)}[\vx_{n-1}]) + (\vmu_n(\vx_n) - \E_{q(\vx_{n-1}|\vx_n)}[\vx_{n-1}])^2.
\end{align*}

According to Theorem~\ref{thm:opt_diag_no_err}, $\diag(\Cov_{q(\vx_{n-1}|\vx_n)}[\vx_{n-1}]) = \vsigma_n^*(\vx_n)^2$. 

According to Lemma~\ref{thm:mm} and the parameterization of $\vmu_n(\vx_n)$ in Eq.~{(\ref{eq:mu_param})}, we have
\begin{align*}
    \vmu_n(\vx_n) - \E_{q(\vx_{n-1}|\vx_n)} [\vx_{n-1}] = \gamma_n
    \sqrt{\frac{\overline{\beta}_n}{\overline{\alpha}_n}} ( \E_{q(\vx_0|\vx_n)}[\vepsilon_n] -  \hat{\vepsilon}_n(\vx_n)).
\end{align*}

Finally,
\begin{align*}
    & \tilde{\vsigma}_n^*(\vx_n)^2 = \diag(\Cov_{q(\vx_{n-1}|\vx_n)}[\vx_{n-1}]) + (\vmu_n(\vx_n) - \E_{q(\vx_{n-1}|\vx_n)}[\vx_{n-1}])^2 \\
    = & \vsigma_n^*(\vx_n)^2 + \gamma_n^2 \frac{\overline{\beta}_n}{\overline{\alpha}_n} \underbrace{(\hat{\vepsilon}_n(\vx_n) - \E_{q(\vx_0|\vx_n)}[\vepsilon_n])^2}_{\text{\normalsize{error}}} \\
    = & \lambda_n^2 \vone + \gamma_n^2 \frac{\overline{\beta}_n}{\overline{\alpha}_n} \left\{ \E_{q(\vx_0|\vx_n)}[\vepsilon_n^2] - \E_{q(\vx_0|\vx_n)}[\vepsilon_n]^2 + \E_{q(\vx_0|\vx_n)}[\vepsilon_n]^2 - 2 \E_{q(\vx_0|\vx_n)}[\vepsilon_n] \hat{\vepsilon}_n(\vx_n) + \hat{\vepsilon}_n(\vx_n)^2 \right\} \\
    = & \lambda_n^2 \vone + \gamma_n^2 \frac{\overline{\beta}_n}{\overline{\alpha}_n} \left\{ \E_{q(\vx_0|\vx_n)}[\vepsilon_n^2] - 2 \E_{q(\vx_0|\vx_n)}[\vepsilon_n] \hat{\vepsilon}_n(\vx_n) + \hat{\vepsilon}_n(\vx_n)^2 \right\} \\
    = & \lambda_n^2 \vone + \gamma_n^2 \frac{\overline{\beta}_n}{\overline{\alpha}_n} \E_{q(\vx_0|\vx_n)}[(\vepsilon_n - \hat{\vepsilon}_n(\vx_n))^2 ].
\end{align*}
\end{proof}

\subsection{Proof of Proposition~\ref{thm:opt_cov_c_no_err} and Proposition~\ref{thm:opt_cov_c}}
\label{sec:proof_ext_c}

\begin{lemma}
\label{thm:mm_c}
Let $0 \leq s < t \leq T$, and $q(\vx_s|\vx_t)$ be the conditional distribution of $\vx_s$ given $\vx_t$ determined by the SDE $\mathrm{d} \vx = f(t) \vx \mathrm{d}t + g(t) \mathrm{d}\vw$. Then we have
\begin{align*}
    & \E_{q(\vx_s|\vx_t)}[\vx_s] = \frac{1}{\sqrt{\alpha_{t|s}}}(\vx_t - \frac{\beta_{t|s}}{\sqrt{\beta_{t|0}}} \E_{q(\vx_0|\vx_t)}[\vepsilon_t]), \\
    & \Cov_{q(\vx_s|\vx_t)}[\vx_s] = \tilde{\beta}_{s|t}\mI + \frac{\beta_{t|s}^2}{\beta_{t|0}\alpha_{t|s}} \left( \E_{q(\vx_0|\vx_t)} [\vepsilon_t \vepsilon_t^\top] - \E_{q(\vx_0|\vx_t)}[\vepsilon_t] \E_{q(\vx_0|\vx_t)}[\vepsilon_t]^\top \right), \\
    & \diag(\Cov_{q(\vx_s|\vx_t)}[\vx_s]) = \tilde{\beta}_{s|t}\vone + \frac{\beta_{t|s}^2}{\beta_{t|0}\alpha_{t|s}} \left( \E_{q(\vx_0|\vx_t)} [\vepsilon_t^2] - \E_{q(\vx_0|\vx_t)}[\vepsilon_t]^2 \right),
\end{align*}
where $\vepsilon_t = \frac{\vx_t - \sqrt{\alpha_{t|0}} \vx_0}{\sqrt{\beta_{t|0}}}$ is the noise used to generate $\vx_t$ from $\vx_0$ and $\tilde{\beta}_{s|t} = \frac{\beta_{s|0}}{\beta_{t|0}}\beta_{t|s}$.
\end{lemma}

\begin{proof}
The conditional distribution of $\vx_s$ given $\vx_t$ and $\vx_0$ determined by the SDE is
\begin{align*}
    q(\vx_s|\vx_t, \vx_0) = \gN(\vx_s| \frac{\sqrt{\alpha_{t|s}} \beta_{s|0}}{\beta_{t|0}} \vx_t + \frac{\sqrt{\alpha_{s|0}}\beta_{t|s}}{\beta_{t|0}}\vx_0, \tilde{\beta}_{s|t}\mI ).
\end{align*}

As for $\E_{q(\vx_s|\vx_t)}[\vx_s]$, we have
\begin{align*}
    \E_{q(\vx_s|\vx_t)}[\vx_s] = \frac{1}{\sqrt{\alpha_{t|s}}}(\vx_t + \beta_{t|s} \nabla \log q_t(\vx_t)) = \frac{1}{\sqrt{\alpha_{t|s}}}(\vx_t - \frac{\beta_{t|s}}{\sqrt{\beta_{t|0}}} \E_{q(\vx_0|\vx_t)}[\vepsilon_t]).
\end{align*}

As for $\Cov_{q(\vx_s|\vx_t)}[\vx_s]$, we can expand it using the law of total variance
\begin{align*}
    \Cov_{q(\vx_s|\vx_t)}[\vx_s] = & \E_{q(\vx_0|\vx_t)} \Cov_{q(\vx_s|\vx_t, \vx_0)}[\vx_s] + \Cov_{q(\vx_0|\vx_t)} \E_{q(\vx_s|\vx_t, \vx_0)}[\vx_s] \\
    = & \tilde{\beta}_{s|t}\mI + \frac{\alpha_{s|0}\beta_{t|s}^2}{\beta_{t|0}^2} \Cov_{q(\vx_0|\vx_t)} [\vx_0] \\
    = & \tilde{\beta}_{s|t}\mI + \frac{\alpha_{s|0}\beta_{t|s}^2}{\beta_{t|0}^2} \frac{\beta_{t|0}}{\alpha_{t|0}} \Cov_{q(\vx_0|\vx_t)} [\vepsilon_t] \\
    = & \tilde{\beta}_{s|t}\mI + \frac{\beta_{t|s}^2}{\beta_{t|0}\alpha_{t|s}} \left( \E_{q(\vx_0|\vx_t)} [\vepsilon_t \vepsilon_t^\top] - \E_{q(\vx_0|\vx_t)}[\vepsilon_t] \E_{q(\vx_0|\vx_t)}[\vepsilon_t]^\top \right).
\end{align*}
Its diagonal is
\begin{align*}
    \diag(\Cov_{q(\vx_s|\vx_t)}[\vx_s]) = \tilde{\beta}_{s|t}\vone + \frac{\beta_{t|s}^2}{\beta_{t|0}\alpha_{t|s}} \left( \E_{q(\vx_0|\vx_t)} [\vepsilon_t^2] - \E_{q(\vx_0|\vx_t)}[\vepsilon_t]^2 \right).
\end{align*}
\end{proof}

\optcovcnoerr*
\begin{proof}
According to Lemma~\ref{thm:opt_sigma_general} and Lemma~\ref{thm:mm_c}, the optimal solution is
\begin{align*}
    & \vmu_{s|t}^*(\vx_t) = \E_{q(\vx_s|\vx_t)}[\vx_s] = \frac{1}{\sqrt{\alpha_{t|s}}}(\vx_t - \frac{\beta_{t|s}}{\sqrt{\beta_{t|0}}} \E_{q_{0|t}(\vx_0|\vx_t)}[\vepsilon_t]), \\
    & \vsigma_{s|t}^{*}(\vx_t)^2 = \diag(\Cov_{q(\vx_s|\vx_t)}[\vx_s]) = \tilde{\beta}_{s|t}\vone + \frac{\beta_{t|s}^2}{\beta_{t|0}\alpha_{t|s}} \left( \E_{q(\vx_0|\vx_t)} [\vepsilon_t^2] - \E_{q(\vx_0|\vx_t)}[\vepsilon_t]^2 \right).
\end{align*}
\end{proof}

\optcovc*
\begin{proof}
According to Lemma~\ref{thm:opt_sigma_general} and Proposition~\ref{thm:opt_cov_c_no_err}, the optimal covariance $\tilde{\vsigma}_{s|t}^{*}(\vx_t)^2$ is
\begin{align*}
    \tilde{\vsigma}_{s|t}^{*}(\vx_t)^2 = & \diag(\Cov_{q(\vx_s|\vx_t)}[\vx_s]) + (\vmu_{s|t}(\vx_t) - \E_{q(\vx_s|\vx_t)}[\vx_s])^2 \\
    = & \vsigma_{s|t}^*(\vx_t)^2 +  \frac{\beta_{t|s}^2}{\beta_{t|0}\alpha_{t|s}} \underbrace{(\hat{\vepsilon}_t(\vx_t) - \E_{q(\vx_0|\vx_t)}[\vepsilon_t])^2}_{\text{\normalsize{error}}} \\
    = & \tilde{\beta}_{s|t}\vone + \frac{\beta_{t|s}^2}{\beta_{t|0}\alpha_{t|s}} \left\{ \E_{q(\vx_0|\vx_t)} [\vepsilon_t^2] - \E_{q(\vx_0|\vx_t)}[\vepsilon_t]^2 + \E_{q(\vx_0|\vx_t)}[\vepsilon_t]^2 - 2 \E_{q(\vx_0|\vx_t)}[\vepsilon_t] \hat{\vepsilon}_t(\vx_t) + \hat{\vepsilon}_t(\vx_t)^2 \right\} \\
    = & \tilde{\beta}_{s|t}\vone + \frac{\beta_{t|s}^2}{\beta_{t|0}\alpha_{t|s}} \left\{ \E_{q(\vx_0|\vx_t)} [\vepsilon_t^2] - 2 \E_{q(\vx_0|\vx_t)} [\vepsilon_t] \hat{\vepsilon}_t(\vx_t) + \hat{\vepsilon}_t(\vx_t)^2 \right\} \\
    = & \tilde{\beta}_{s|t}\vone + \frac{\beta_{t|s}^2}{\beta_{t|0}\alpha_{t|s}} \E_{q(\vx_0|\vx_t)} [(\vepsilon_t - \hat{\vepsilon}_t(\vx_t) )^2].
\end{align*}
\end{proof}

\subsection{Strictly Improved Expressive Power when the Data Distribution is a Mixture of Gaussian}

\begin{lemma}
\label{thm:express}
Suppose $q(\vx, \vy)$ is a probability density function, s.t., $q(\vx) = \sum\limits_{j=1}^J \gamma_j \gN(\vmu_j, c\mI)$ with $J \geq 2$ and pairwise distinct $\{\vmu_j\}_{j=1}^J$, and $q(\vy|\vx) = \gN(\vy|\sqrt{\alpha} \vx, \beta \mI)$ for some $\alpha, \beta > 0$. Suppose $p(\vx|\vy) = \gN(\vx|\vmu(\vy), \mSigma(\vy))$ is a Gaussian distribution conditioned on $\vy$. Then we have, 
\begin{align*}
    \min\limits_{\vmu, \mSigma: \mSigma(\vy) = \diag(\vsigma(\vy)^2)} \E_{q(\vy)}\KL(q(\vx|\vy)\|p(\vx|\vy)) < \min\limits_{\vmu, \mSigma: \mSigma(\vy) = \sigma^2 \mI} \E_{q(\vy)}\KL(q(\vx|\vy)\|p(\vx|\vy)).
\end{align*}
\end{lemma}

\begin{proof}
As mentioned in the proof of Lemma~\ref{thm:opt_sigma_general}, the KL divergence between a probability density function $q$ and a Gaussian distribution $p$ can be written as $\KL(q\|p) = \KL(\gN(\vx|\vmu_q, \mSigma_q) \| p) + const$, where $\vmu_q$ is the mean of $q$, $\mSigma_q$ is the covariance matrix of $q$ and $const$ is only related to $q$. Thus, we only need to prove
\begin{align*}
    & \min\limits_{\vmu, \mSigma: \mSigma(\vy) = \diag(\vsigma(\vy)^2)} \E_{q(\vy)} \KL(\gN(\vx|\E_{q(\vx|\vy)}[\vx], \Cov_{q(\vx|\vy)}[\vx]) \| p(\vx|\vy)) \\
    < & \min\limits_{\vmu, \mSigma: \mSigma(\vy) = \sigma^2 \mI} \E_{q(\vy)} \KL(\gN(\vx|\E_{q(\vx|\vy)}[\vx], \Cov_{q(\vx|\vy)}[\vx]) \| p(\vx|\vy)).
\end{align*}

We know the optimal mean is $\vmu^*(\vy) = \E_{q(\vx|\vy)}[\vx]$ for both problems. Let $\mM(\vy) = \Cov_{q(\vx|\vy)}[\vx]$. We expand the KL divergence, and substitute $\vmu^*$ into $\vmu$
\begin{align*}
    2 \KL(\gN(\vx|\E_{q(\vx|\vy)}[\vx], \Cov_{q(\vx|\vy)}[\vx]) \| p(\vx|\vy)) = \tr(\mSigma(\vy)^{-1} \mM(\vy) ) - d + \log \frac{|\mSigma(\vy)|}{|\Cov_{q(\vx|\vy)}[\vx]|}.
\end{align*}
Thus, we only need to prove
\begin{align*}
\min\limits_{\mSigma: \mSigma(\vy) = \diag(\vsigma(\vy)^2)} \E_{q(\vy)} \left[ \tr(\mSigma(\vy)^{-1} \mM(\vy) ) + \log |\mSigma(\vy)| \right] < \min\limits_{\mSigma: \mSigma(\vy) = \sigma^2 \mI} \E_{q(\vy)} \left[ \tr(\mSigma(\vy)^{-1} \mM(\vy) ) + \log |\mSigma(\vy)| \right],
\end{align*}
which is equivalent to
\begin{align*}
\min\limits_{\vsigma(\cdot)^2} \E_{q(\vy)} \left[ \sum\limits_{i=1}^d \left(\vsigma(\vy)_i^{-2} \mM(\vy)_{ii}  + \log \vsigma(\vy)_i^2\right) \right] < \min\limits_{\sigma^2} \E_{q(\vy)} \left[ \sum\limits_{i=1}^d \left(\sigma^{-2} \mM(\vy)_{ii} + \log \sigma^2\right) \right].
\end{align*}
The left problem admits the optimal solution $\vsigma^*(\vy)_i^2 = \mM(\vy)_{ii}$, and the right one admits the optimal solution $\sigma^{*2} = \E_{q(\vy)} \frac{\tr(\mM(\vy))}{d}$. Thus, we only need to prove
\begin{align}
\label{eq:sleq}
\E_{q(\vy)} \left[ \sum\limits_{i=1}^d \left(\vsigma^*(\vy)_i^{-2} \mM(\vy)_{ii}  + \log \vsigma^*(\vy)_i^2\right) \right] < \E_{q(\vy)} \left[ \sum\limits_{i=1}^d \left(\sigma^{*-2} \mM(\vy)_{ii} + \log \sigma^{*2}\right) \right].
\end{align}

We prove it by contradiction. Suppose
\begin{align*}
\E_{q(\vy)} \left[ \sum\limits_{i=1}^d \left(\vsigma^*(\vy)_i^{-2} \mM(\vy)_{ii}  + \log \vsigma^*(\vy)_i^2\right) \right] = \E_{q(\vy)} \left[ \sum\limits_{i=1}^d \left(\sigma^{*-2} \mM(\vy)_{ii} + \log \sigma^{*2}\right) \right].
\end{align*}
Since $q(\vy)$ is fully supported, it implies that
\begin{align*}
\forall \vy, \sum\limits_{i=1}^d \left(\vsigma^*(\vy)_i^{-2} \mM(\vy)_{ii}  + \log \vsigma^*(\vy)_i^2\right) = \sum\limits_{i=1}^d \left(\sigma^{*-2} \mM(\vy)_{ii} + \log \sigma^{*2}\right),
\end{align*}
which holds if and only if $\forall \vy, \forall i, \vsigma^*(\vy)_i^2 = \sigma^{*2}$. This means $\forall \vy, \forall i, \mM(\vy)_{ii} = \E_{q(\vy)} \frac{\tr(\mM(\vy))}{d}$, which implies that there exists $m \in \sR$ irrelevant to $\vy$, s.t., $\diag(\Cov_{q(\vx|\vy)}[\vx]) = m \mI$.

Now we derive $q(\vx|\vy)$, 
\begin{align*}
& q(\vx|\vy) \propto q(\vx) q(\vy|\vx) \propto \sum\limits_{j=1}^J \gamma_j \exp(- \frac{\|\vx - \vmu_j\|_2^2}{2 c}) \exp(- \frac{\|\vy - \sqrt{\alpha}\vx \|_2^2}{2 \beta}) \\
\propto & \sum\limits_{j=1}^J \gamma_j \exp(-\frac{\|\vmu_j\|_2^2}{2c}) \exp(-  \frac{(c\alpha+\beta)\|\vx\|_2^2 - 2\left<\vx, \beta\vmu_j+c\sqrt{\alpha}\vy\right>}{2c\beta} ) \\
= & \sum\limits_{j=1}^J \gamma_j \exp(\frac{\|\beta\vmu_j+c\sqrt{\alpha}\vy\|_2^2}{2c\beta(c\alpha+\beta)}-\frac{\|\vmu_j\|_2^2}{2c}) \exp(-\frac{\|\vx-(\beta\vmu_j+c\sqrt{\alpha}\vy)/(c\alpha+\beta)\|_2^2}{2c\beta/(c\alpha+\beta)}) \\
\propto & \sum\limits_{j=1}^J \exp(\log \gamma_j - \frac{\alpha}{c\alpha+\beta}\frac{\|\vmu_j\|_2^2}{2}+\frac{\sqrt{\alpha}}{c\alpha+\beta}\vmu_j^\top \vy) \gN((\beta\vmu_j+c\sqrt{\alpha}\vy)/(c\alpha+\beta), c\beta/(c\alpha+\beta) \mI).
\end{align*}

Let $\xi_j = \log \gamma_j - \frac{\alpha}{c\alpha+\beta}\frac{\|\vmu_j\|_2^2}{2}$. Let $\vphi(\vy) = (\phi_1(\vy), \cdots, \phi_J(\vy))$, where $\phi_j(\vy) = \xi_j+\frac{\sqrt{\alpha}}{c\alpha+\beta} \vmu_j^\top \vy$, and $\veta(\vy) = \softmax(\vphi(\vy))$. 
Let $\vnu_j(\vy) = (\beta\vmu_j+c\sqrt{\alpha}\vy)/(c\alpha+\beta)$, $\overline{\vmu}(\vy) = \sum\limits_{j=1}^J \eta_j(\vy) \vmu_j$ and $\overline{\vnu}(\vy) = \sum\limits_{j=1}^J \eta_j(\vy) \vnu_j(\vy) =  ( \beta\overline{\vmu}(\vy)+c\sqrt{\alpha}\vy)/(c\alpha+\beta)$.
Then
\begin{align*}
    q(\vx|\vy) = \sum\limits_{j=1}^J \eta_j(\vy) \gN(\vnu_j(\vy), c\beta/(c\alpha+\beta)\mI),
\end{align*}
which is also a mixture of Gaussian. According to the property of mixture of Gaussian, the diagonal of its covariance is
\begin{align*}
    \diag(\Cov_{q(\vx|\vy)}[\vx]) = c\beta/(c\alpha+\beta)\vone + \sum\limits_{j=1}^J \eta_j(\vy) \vnu_j(\vy)^2  - \overline{\vnu}(\vy)^2.
\end{align*}

Meanwhile, $\diag(\Cov_{q(\vx|\vy)}[\vx]) = m\mI$. Thus,
\begin{align*}
    c\beta/(c\alpha+\beta)d + \sum\limits_{j=1}^J \eta_j(\vy) \|\vnu_j(\vy)\|_2^2 - \|\overline{\vnu}(\vy)\|_2^2 = d m,
\end{align*}
which is irrelevant to $\vy$. Since
\begin{align*}
    \sum\limits_{j=1}^J \eta_j(\vy) \|\vnu_j(\vy)\|_2^2 - \|\overline{\vnu}(\vy)\|_2^2 = & \beta^2/(c\alpha+\beta)^2 \left[\sum\limits_{j=1}^J \eta_j(\vy) \|\vmu_j\|_2^2 - \|\overline{\vmu}(\vy)\|_2^2\right],
\end{align*}
$\sum\limits_{j=1}^J \eta_j(\vy) \|\vmu_j\|_2^2 - \|\overline{\vmu}(\vy)\|_2^2$ is also irrelevant to $\vy$. Let $\vy = t \vy_0$, $j_0 = \argmax\limits_j \vmu_j^\top \vy_0$. Then $\lim\limits_{t\rightarrow \infty} \eta_{j_0}(t \vy_0) = 1$, and
\begin{align*}
    \lim\limits_{t\rightarrow \infty} \sum\limits_{j=1}^J \eta_j(t \vy_0) \|\vmu_j\|_2^2 - \|\overline{\vmu}(t \vy_0)\|_2^2 = 0.
\end{align*}
Thus, $\sum\limits_{j=1}^J \eta_j(\vy) \|\vmu_j\|_2^2 - \|\overline{\vmu}(\vy)\|_2^2 = 0$ for all $\vy$. This means
\begin{align*}
    \sum\limits_{j=1}^J \eta_j(\vy) \|\vmu_j\|_2^2 - \|\overline{\vmu}(\vy)\|_2^2 = \sum\limits_{j=1}^J \eta_j(\vy) \|\vmu_j-\overline{\vmu}(\vy)\|_2^2 = 0.
\end{align*}
Thus, $\vmu_1=\vmu_2=\cdots=\vmu_J=\overline{\vmu}(\vy)$. This contradicts the assumption that $\{\vmu_j\}_{j=1}^J$ are pairwise distinct. Therefore, Eq.~{(\ref{eq:sleq})} must hold.
\end{proof}

\begin{proposition}
\label{thm:express_example}
If the data distribution $q(\vx_0)=\sum\limits_{j=1}^J \gamma_j \gN(\vmu_j, c\mI)$ is a mixture of $J \geq 2$ Gaussian, and $\{\vmu_j\}_{j=1}^J$ are pairwise distinct, then we have
\begin{align*}
    \max\limits_{\vmu_n, \mSigma_n: \mSigma_n(\vx_n)=\diag(\vsigma_n(\vx_n)^2)} \lelbo > \max\limits_{\vmu_n, \mSigma_n: \mSigma_n(\vx_n) = \sigma_n^2 \mI} \lelbo.
\end{align*}
\end{proposition}

\begin{proof}
We only need to prove
\begin{align*}
    \min\limits_{\vmu_n,\mSigma_n: \mSigma_n(\vx_n)=\diag(\vsigma_n(\vx_n)^2)} \KL (q(\vx_{0:N}) \| p(\vx_{0:N})) < \min\limits_{\vmu_n,\mSigma_n: \mSigma_n(\vx_n) = \sigma_n^2 \mI} \KL (q(\vx_{0:N}) \| p(\vx_{0:N})).
\end{align*}
As mentioned in the proof of Theorem~\ref{thm:opt_diag_no_err}, $\KL (q(\vx_{0:N}) \| p(\vx_{0:N}))$ can be decomposed as 
\begin{align*}
\KL(q(\vx_{0:N}) \| p(\vx_{0:N})) = \sum\limits_{n=1}^N \E_q \KL(q(\vx_{n-1}|\vx_n) \| p(\vx_{n-1}|\vx_n)) + c,
\end{align*}
where $c$ is a constant only related to $q$. Thus, we only need to prove
\begin{align*}
    & \sum\limits_{n=1}^N \min\limits_{\vmu_n,\mSigma_n: \mSigma_n(\vx_n)=\diag(\vsigma_n(\vx_n)^2)} \E_q \KL(q(\vx_{n-1}|\vx_n) \| p(\vx_{n-1}|\vx_n)) \\
    < & \sum\limits_{n=1}^N \min\limits_{\vmu_n,\mSigma_n: \mSigma_n(\vx_n) = \sigma_n^2 \mI} \E_q \KL(q(\vx_{n-1}|\vx_n) \| p(\vx_{n-1}|\vx_n)).
\end{align*}
It is sufficient to prove when $n=1$,
\begin{align*}
    & \min\limits_{\vmu_1,\mSigma_1: \mSigma_1(\vx_1)=\diag(\vsigma_1(\vx_1)^2)} \E_q \KL(q(\vx_0|\vx_1) \| p(\vx_0|\vx_1)) \\
    < & \min\limits_{\vmu_1,\mSigma_1: \mSigma_1(\vx_1) = \sigma_1^2 \mI} \E_q \KL(q(\vx_0|\vx_1) \| p(\vx_0|\vx_1)),
\end{align*}
which holds according to Lemma~\ref{thm:express}.
\end{proof}

\section{Results for Full Covariances}
\label{sec:ext_cov_g}
In this part, we derive results on full covariances. In Proposition~\ref{thm:opt_cov_g_no_err}, we derive the optimal solution to the joint optimization problem~{(\ref{eq:elbo})}. In Proposition~\ref{thm:opt_cov_g}, we derive the optimal covariance to the optimization problem~{(\ref{eq:elbo_cov})} w.r.t. the covariance solely. Their proof requires Lemma~\ref{thm:opt_cov_general}, which is very similar to Lemma~\ref{thm:opt_sigma_general}.

\begin{lemma}
\label{thm:opt_cov_general}
Suppose $q(\vx)$ is a probability density function with mean $\vmu_q$ and covariance matrix $\mSigma_q$ and $p(\vx) = \gN(\vx|\vmu, \mSigma)$ is a Gaussian distribution. Then 
\begin{enumerate}
    \item for any $\mSigma$, the optimization problem $\min\limits_{\vmu} \KL(q\|p)$ admits an optimal solution $\vmu^* = \vmu_q$;
    \item for any $\vmu$, the optimization problem $\min\limits_{\mSigma} \KL(q\|p)$ admits an optimal solution $\tilde{\mSigma}^* = \mSigma_q + (\vmu - \vmu_q) (\vmu - \vmu_q)^\top$
    \item the optimization problem $\min\limits_{\vmu, \mSigma} \KL(q\|p)$ admits an optimal solution $\vmu^* = \vmu_q$ and $\mSigma^* = \mSigma_q$.
\end{enumerate}
\end{lemma}

\begin{proof}
Similarly to the proof of Lemma~\ref{thm:opt_sigma_general}, the optimizations of $\KL(q\|p)$ and $\KL(\gN(\vx|\vmu_q, \mSigma_q) \| p)$ w.r.t. $\vmu$ and $\mSigma$ are equivalent. We expand $\KL(\gN(\vx|\vmu_q, \mSigma_q) \| p)$ according to the definition of KL divergence:
\begin{align*}
    & 2 \KL(\gN(\vx|\vmu_q, \mSigma_q) \| p) = \tr(\mSigma^{-1} \mSigma_q) - d + \log \frac{|\mSigma|}{|\mSigma_q|} + (\vmu - \vmu_q)^\top \mSigma^{-1} (\vmu - \vmu_q) \\
    = & \tr(\mSigma^{-1}(\mSigma_q + (\vmu - \vmu_q) (\vmu - \vmu_q)^\top)) + \log \frac{|\mSigma|}{|\mSigma_q|} - d.
\end{align*}
Thus, $\min\limits_{\vmu} \KL(q\|p) \Leftrightarrow \min\limits_{\vmu} (\vmu - \vmu_q)^\top \mSigma^{-1} (\vmu - \vmu_q)$, which admits an optimal solution $\vmu^* = \vmu_q$, and $\min\limits_{\mSigma} \KL(q\|p) \Leftrightarrow \min\limits_{\mSigma} \tr(\mSigma^{-1}(\mSigma_q + (\vmu - \vmu_q) (\vmu - \vmu_q)^\top)) + \log |\mSigma|$, which admits an optimal solution $\tilde{\mSigma}^* = \mSigma_q + (\vmu - \vmu_q) (\vmu - \vmu_q)^\top$. Combining $\vmu^*$ and $\tilde{\mSigma}^*$, we know that $\min\limits_{\vmu,\mSigma} \KL(q\|p)$ admits an optimal solution $\vmu^* = \vmu_q$ and $\mSigma^* = \mSigma_q$.
\end{proof}

\begin{proposition}
\label{thm:opt_cov_g_no_err}
Suppose $\mSigma_n(\vx_n)$ is a full covariance. Then the optimal mean to problem~{(\ref{eq:elbo})} is $\vmu_n^*(\vx_n)$ as in Eq.~{(\ref{eq:opt_mean})}, and the optimal covariance to problem~{(\ref{eq:elbo})} is
\begin{align*}
\mSigma_n^*(\vx_n) = & \lambda_n^2 \mI + \gamma_n^2 \frac{\overline{\beta}_n}{\overline{\alpha}_n} \left(\E_{q(\vx_0|\vx_n)}[\vepsilon_n \vepsilon_n^\top] - \E_{q(\vx_0|\vx_n)}[\vepsilon_n] \E_{q(\vx_0|\vx_n)}[\vepsilon_n]^\top\right).
\end{align*}
\end{proposition}

\begin{proof}
\citet{bao2022analytic} show that the KL divergence between a joint probability density function $q(\vx_{0:N})$ and a Markov chain $p(\vx_{0:N})$ can be written as
\begin{align*}
    \KL(q(\vx_{0:N}) \| p(\vx_{0:N})) = \sum\limits_{n=1}^N \E_q \KL(q(\vx_{n-1}|\vx_n) \| p(\vx_{n-1}|\vx_n)) + c,
\end{align*}
where $c$ is a constant only related to $q$ (see Lemma 8 in \citet{bao2022analytic}). As a result, the optimization of the KL divergence is decomposed into $n$ independent optimization sub-problems:
\begin{align*}
    \min\limits_{\vmu_n(\cdot), \mSigma_n(\cdot)} \E_q \KL(q(\vx_{n-1}|\vx_n) \| p(\vx_{n-1}|\vx_n)), \quad 1 \leq n \leq N.
\end{align*}

Since $q(\vx_{n-1}|\vx_n)$ is a probability density function with mean $\E_{q(\vx_{n-1}|\vx_n)}[\vx_{n-1}]$ and covariance matrix $\Cov_{q(\vx_{n-1}|\vx_n)}[\vx_{n-1}]$ and $p(\vx_{n-1}|\vx_n) = \gN(\vx_{n-1}|\vmu_n(\vx_n), \mSigma_n(\vx_n)))$ is a Gaussian distribution, according to Lemma~\ref{thm:opt_cov_general} and Lemma~\ref{thm:mm}, we know the optimal solution is
\begin{align*}
    & \vmu_n^*(\vx_n) = \E_{q(\vx_{n-1}|\vx_n)}[\vx_{n-1}] = \tilde{\vmu}_n(\vx_n, \frac{1}{\sqrt{\overline{\alpha}_n}} ( \vx_n - \sqrt{\overline{\beta}_n} \E_{q(\vx_0|\vx_n)}[\vepsilon_n] )), \\
    & \mSigma_n^*(\vx_n) = \Cov_{q(\vx_{n-1}|\vx_n)}[\vx_{n-1}] = \lambda_n^2 \mI + \gamma_n^2 \frac{\overline{\beta}_n}{\overline{\alpha}_n} \left( \E_{q(\vx_0|\vx_n)}[\vepsilon_n \vepsilon_n^\top] - \E_{q(\vx_0|\vx_n)}[\vepsilon_n] \E_{q(\vx_0|\vx_n)}[\vepsilon_n]^\top \right).
\end{align*}
\end{proof}

\begin{proposition}
\label{thm:opt_cov_g}
Suppose $\mSigma_n(\vx_n)$ is a full covariance. For any mean $\vmu_n(\vx_n)$ that is parameterized by a noise prediction network $\hat{\vepsilon}_n(\vx_n)$ as in Eq.~{(\ref{eq:mu_param})}, the optimal covariance $\tilde{\mSigma}_n^*(\vx_n)$ to problem~{(\ref{eq:elbo_cov})} is
\begin{align*}
\tilde{\mSigma}_n^*(\vx_n) = & \mSigma_n^*(\vx_n) + \gamma_n^2 \frac{\overline{\beta}_n}{\overline{\alpha}_n} \underbrace{(\E_{q(\vx_0|\vx_n)}[\vepsilon_n] - \hat{\vepsilon}_n(\vx_n)) (\E_{q(\vx_0|\vx_n)}[\vepsilon_n] - \hat{\vepsilon}_n(\vx_n))^\top}_{\text{\normalsize{error}}} \nonumber \\
= & \lambda_n^2 \mI + \gamma_n^2 \frac{\overline{\beta}_n}{\overline{\alpha}_n} \E_{q(\vx_0|\vx_n)}[(\vepsilon_n - \hat{\vepsilon}_n(\vx_n)) (\vepsilon_n - \hat{\vepsilon}_n(\vx_n))^\top].
\end{align*}
\end{proposition}

\begin{proof}
According to Lemma~\ref{thm:opt_cov_general}, we know the optimal $\tilde{\mSigma}_n^*(\vx_n)$ is
\begin{align*}
    \tilde{\mSigma}_n^*(\vx_n) = \Cov_{q(\vx_{n-1}|\vx_n)}[\vx_{n-1}] + (\vmu_n(\vx_n) - \E_{q(\vx_{n-1}|\vx_n)}[\vx_{n-1}]) (\vmu_n(\vx_n) - \E_{q(\vx_{n-1}|\vx_n)}[\vx_{n-1}])^\top.
\end{align*}

According to Proposition~\ref{thm:opt_cov_g_no_err}, $\Cov_{q(\vx_{n-1}|\vx_n)}[\vx_{n-1}] = \mSigma_n^*(\vx_n)$. 

According to Lemma~\ref{thm:mm} and the parameterization of $\vmu_n(\vx_n)$ in Eq.~{(\ref{eq:mu_param})}, we have
\begin{align}
\label{eq:err}
    \vmu_n(\vx_n) - \E_{q(\vx_{n-1}|\vx_n)} [\vx_{n-1}] = \gamma_n
    \sqrt{\frac{\overline{\beta}_n}{\overline{\alpha}_n}} ( \E_{q(\vx_0|\vx_n)}[\vepsilon_n] -  \hat{\vepsilon}_n(\vx_n)).
\end{align}
% where $\gamma_n = \sqrt{\overline{\alpha}_{n-1}} - \sqrt{\overline{\beta}_{n-1} - \lambda_n^2} \sqrt{\frac{\overline{\alpha}_n}{\overline{\beta}_n}}$ and $\E_{q(\vx_0|\vx_n)}[\vepsilon_n] = \E_{q(\vx_0|\vx_n)}[\vepsilon_n]$.

Finally,
\begin{align*}
    \tilde{\mSigma}_n^*(\vx_n) = & \Cov_{q(\vx_{n-1}|\vx_n)}[\vx_{n-1}] + (\vmu_n(\vx_n) - \E_{q(\vx_{n-1}|\vx_n)}[\vx_{n-1}]) (\vmu_n(\vx_n) - \E_{q(\vx_{n-1}|\vx_n)}[\vx_{n-1}])^\top \\
    = & \mSigma_n^*(\vx_n) + \gamma_n^2 \frac{\overline{\beta}_n}{\overline{\alpha}_n} \underbrace{(\E_{q(\vx_0|\vx_n)}[\vepsilon_n] - \hat{\vepsilon}_n(\vx_n)) (\E_{q(\vx_0|\vx_n)}[\vepsilon_n] - \hat{\vepsilon}_n(\vx_n))^\top}_{\text{\normalsize{error}}} \\
    = & \lambda_n^2 \mI + \gamma_n^2 \frac{\overline{\beta}_n}{\overline{\alpha}_n} \left\{ \E_{q(\vx_0|\vx_n)}[\vepsilon_n \vepsilon_n^\top] - \E_{q(\vx_0|\vx_n)}[\vepsilon_n] \hat{\vepsilon}_n(\vx_n)^\top - \hat{\vepsilon}_n(\vx_n) \E_{q(\vx_0|\vx_n)}[\vepsilon_n]^\top + \hat{\vepsilon}_n(\vx_n) \hat{\vepsilon}_n(\vx_n))^\top \right\} \\
    = & \lambda_n^2 \mI + \gamma_n^2 \frac{\overline{\beta}_n}{\overline{\alpha}_n} \E_{q(\vx_0|\vx_n)}[\vepsilon_n \vepsilon_n^\top - \vepsilon_n \hat{\vepsilon}_n(\vx_n)^\top - \hat{\vepsilon}_n(\vx_n) \vepsilon_n^\top + \hat{\vepsilon}_n(\vx_n) \hat{\vepsilon}_n(\vx_n)^\top ] \\
    = & \lambda_n^2 \mI + \gamma_n^2 \frac{\overline{\beta}_n}{\overline{\alpha}_n} \E_{q(\vx_0|\vx_n)}[ (\vepsilon_n - \hat{\vepsilon}_n(\vx_n)) (\vepsilon_n - \hat{\vepsilon}_n(\vx_n))^\top ].
\end{align*}
\end{proof}

Note that $\E_{q(\vx_0|\vx_n)}[\vepsilon_n \vepsilon_n^\top]$ in Proposition~\ref{thm:opt_cov_g_no_err} and $\E_{q(\vx_0|\vx_n)}[ (\vepsilon_n - \hat{\vepsilon}_n(\vx_n)) (\vepsilon_n - \hat{\vepsilon}_n(\vx_n))^\top ]$ in Proposition~\ref{thm:opt_cov_g} are matrices of shape $d\times d$. We can also learn these conditional expectations by training neural networks that output matrices of shape $d\times d$ on MSE losses, and then estimate $\mSigma_n^*(\vx_n)$ and $\tilde{\mSigma}_n^*(\vx_n)$ using these neural networks. % This is more difficult than learning their diagonals, i.e., $\E_{q(\vx_0|\vx_n)}[\vepsilon_n^2]$ in Theorem~\ref{thm:opt_diag_no_err} and $\E_{q(\vx_0|\vx_n)}[ (\vepsilon_n - \hat{\vepsilon}_n(\vx_n))^2 ]$ in Theorem~\ref{thm:opt_diag}.

However, obtaining samples from Gaussian transitions with full covariances requires a decomposition of the covariance matrix, e.g., Cholesky decomposition, which has a $\gO(d^3)$ time complexity. In practice, $d$ can be large (e.g., high resolution images), and such a decomposition can be time consuming. Therefore, in practical settings, it is more suitable to consider diagonal covariances as in the main text, which do not involve a matrix decomposition.

\section{Details on Extension to DPMs with Continuous Timesteps}
\label{sec:detail_ext_c}

To obtain the estimate of $\vsigma_{s|t}^*(\vx_t)^2$ appeared in Proposition~\ref{thm:opt_cov_c_no_err}, we use a network $\vh_t(\vx_t) \in \R^d$ to learn $\E_{q(\vx_0|\vx_t)}[\vepsilon_t^2]$ by minimizing the following MSE loss
\begin{align*}
    \min\limits_{\vh_t} \E_t \E_{q(\vx_0, \vx_t)} \|\vepsilon_t^2 - \vh_t(\vx_t)\|_2^2,
\end{align*}
where $t$ is uniformly sampled from $[0, T]$. Then we obtain the estimate of $\vsigma_{s|t}^*(\vx_t)^2$:
\begin{align*}
\hat{\vsigma}_{s|t}(\vx_t)^2 = \tilde{\beta}_{s|t}\vone + \frac{\beta_{t|s}^2}{\beta_{t|0}\alpha_{t|s}} \left( \vh_t(\vx_t) - \hat{\vepsilon}_t(\vx_t)^2 \right).
\end{align*}

To obtain the estimate of $\tilde{\vsigma}_{s|t}^*(\vx_t)^2$ appeared in Proposition~\ref{thm:opt_cov_c}, we use a network $\vg_t(\vx_t) \in \R^d$ to learn $\E_{q(\vx_0|\vx_t)}[(\vepsilon_t - \hat{\vepsilon}_t(\vx_t))^2]$ by minimizing the following MSE loss
\begin{align*}
    \min\limits_{\vg_t} \E_t \E_{q(\vx_0, \vx_t)} \|(\vepsilon_t - \hat{\vepsilon}_t(\vx_t))^2 - \vg_t(\vx_t)\|_2^2,
\end{align*}
where $t$ is uniformly sampled from $[0, T]$. Then we obtain the estimate of $\tilde{\vsigma}_{s|t}^*(\vx_t)^2$:
\begin{align*}
\hat{\tilde{\vsigma}}_{s|t}(\vx_t)^2 = \tilde{\beta}_{s|t} \vone + \frac{\beta_{t|s}^2}{\beta_{t|0}\alpha_{t|s}} \vg_t(\vx_t).
\end{align*}

\section{Inference on Trajectories}
\label{sec:detail_fast}
To speed up the inference, we can reverse a shorter forward process $q(\vx_{\tau_1}, \cdots, \vx_{\tau_k}|\vx_0)$ constrained on a trajectory $1 \leq \tau_1 < \cdots < \tau_K = N$ of $K$ timesteps~\cite{song2020denoising,bao2022analytic}. Following the notation of \citet{bao2022analytic}, the shorter forward process is defined as
\begin{align*}
    & q(\vx_{\tau_1}, \cdots, \vx_{\tau_K} | \vx_0)=q(\vx_{\tau_K}|\vx_0) \prod_{k=2}^K q(\vx_{\tau_{k-1}}|\vx_{\tau_k}, \vx_0), \\
    & q(\vx_{\tau_{k-1}}|\vx_{\tau_k}, \vx_0) = \gN(\vx_{\tau_{k-1}}| \tilde{\vmu}_{\tau_{k-1}|\tau_k} (\vx_{\tau_k}, \vx_0), \lambda_{\tau_{k-1}|\tau_k}^2 \mI), \\
    & \tilde{\vmu}_{\tau_{k-1}|\tau_k} (\vx_{\tau_k}, \vx_0) = \sqrt{\overline{\alpha}_{\tau_{k-1}}} \vx_0 + \sqrt{\overline{\beta}_{\tau_{k-1}} - \lambda_{\tau_{k-1}|\tau_k}^2} \cdot \frac{\vx_{\tau_k} - \sqrt{\overline{\alpha}_{\tau_k}}\vx_0}{\sqrt{\overline{\beta}_{\tau_k}}}.
\end{align*}

Similarly to Eq.~{(\ref{eq:markov})}, the shorter Markov model to reverse the forward process is defined as
\begin{align*}
    & p(\vx_0, \vx_{\tau_1}, \cdots, \vx_{\tau_K}) = p(\vx_{\tau_K}) \prod_{k=1}^K p(\vx_{\tau_{k-1}}|\vx_{\tau_k}), \\
    & p(\vx_{\tau_{k-1}}|\vx_{\tau_k}) = \gN(\vx_{\tau_{k-1}} | \vmu_{\tau_{k-1} | \tau_k} (\vx_{\tau_k}), \mSigma_{\tau_{k-1} | \tau_k} (\vx_{\tau_k}) \mI).
\end{align*}

According to Theorem~\ref{thm:opt_diag_no_err}, when $\mSigma_{\tau_{k-1} | \tau_k} (\vx_{\tau_k}) = \diag(\vsigma_{\tau_{k-1} | \tau_k} (\vx_{\tau_k})^2)$, the optimal covariance $\vsigma_{\tau_{k-1} | \tau_k}^* (\vx_{\tau_k})^2$ is
\begin{align*}
    \vsigma^*_{\tau_{k-1} | \tau_k} (\vx_{\tau_k})^2 = \lambda_{\tau_{k-1} | \tau_k}^2 \vone + \gamma_{\tau_{k-1} | \tau_k}^2 \frac{\overline{\beta}_{\tau_k}}{\overline{\alpha}_{\tau_k}} \left(\E_{q(\vx_0|\vx_{\tau_k})}[\vepsilon_{\tau_k}^2] - \E_{q(\vx_0|\vx_{\tau_k})}[\vepsilon_{\tau_k}]^2\right),
\end{align*}
where $\vepsilon_{\tau_k} = \frac{\vx_{\tau_k} - \sqrt{\overline{\alpha}_{\tau_k}} \vx_0}{\sqrt{\overline{\beta}_{\tau_k}}}$ is the noise used to generate $\vx_{\tau_k}$ from $\vx_0$ and $\gamma_{\tau_{k-1} | \tau_k} = \sqrt{\overline{\alpha}_{\tau_{k-1}}} - \sqrt{\overline{\beta}_{\tau_{k-1}} - \lambda_{\tau_{k-1}|\tau_k}^2} \sqrt{\frac{\overline{\alpha}_{\tau_k}}{\overline{\beta}_{\tau_k}}}$. Thus, we can reuse the SN prediction network $\vh_n(\vx_n)$ trained on the loss in Eq.~{(\ref{eq:mse_eps2})}, and estimate $\vsigma_{\tau_{k-1} | \tau_k}^* (\vx_{\tau_k})^{2}$ as follows:
\begin{align*}
    \hat{\vsigma}_{\tau_{k-1}|\tau_k} (\vx_{\tau_k})^2 = \lambda_{\tau_{k-1}|\tau_k}^2 \vone + \gamma_{\tau_{k-1} | \tau_k}^2 \frac{\overline{\beta}_{\tau_k}}{\overline{\alpha}_{\tau_k}} \left(\vh_{\tau_k}(\vx_{\tau_k}) - \hat{\vepsilon}_{\tau_k}(\vx_{\tau_k})^2 \right).
\end{align*}

When considering an imperfect mean (i.e., considering the error of $\hat{\vepsilon}_n(\vx_n)$), according to Theorem~\ref{thm:opt_diag}, the corrected optimal covariance $\tilde{\vsigma}^*_{\tau_{k-1} | \tau_k} (\vx_{\tau_k})^2$ is
\begin{align*}
    \tilde{\vsigma}^*_{\tau_{k-1} | \tau_k} (\vx_{\tau_k})^2 = \lambda_{\tau_{k-1} | \tau_k}^2 \vone + \gamma_{\tau_{k-1} | \tau_k}^2 \frac{\overline{\beta}_{\tau_k}}{\overline{\alpha}_{\tau_k}} \E_{q(\vx_0|\vx_{\tau_k})}[(\vepsilon_{\tau_k} - \hat{\vepsilon}_{\tau_k}(\vx_{\tau_k}))^2].
\end{align*}
Thus, we can reuse the NPR prediction network $\vg_n(\vx_n)$ trained on the loss in Eq.~{(\ref{eq:mse_npr})}, and estimate $\tilde{\vsigma}_{\tau_{k-1} | \tau_k}^* (\vx_{\tau_k})^{2}$ as follows:
\begin{align*}
    \hat{\tilde{\vsigma}}_{\tau_{k-1}|\tau_k} (\vx_{\tau_k})^2 = \lambda_{\tau_{k-1}|\tau_k}^2 \vone + \gamma_{\tau_{k-1} | \tau_k}^2 \frac{\overline{\beta}_{\tau_k}}{\overline{\alpha}_{\tau_k}} \vg_{\tau_k}(\vx_{\tau_k}).
\end{align*}

We emphasize that the inference on a shorter trajectory doesn't need to train a new NPR (or SN) prediction network on the corresponding shorter process, since the shorter process and the original one share the same marginal distribution. We reuse the one trained on the full timesteps.

\section{Additional Discussion}
\subsection{Error Amplification}
\label{sec:err_amp}

Note that in SN-DPM, the error of $\hat{\vepsilon}_n(\vx_n)^2$ is likely to be amplified compared to that of $\hat{\vepsilon}_n(\vx_n)$.
SN-DPM estimates $\E_{q(\vx_0|\vx_n)}[\vepsilon_n]^2$ using $\hat{\vepsilon}_n(\vx_n)^2$, which is the square of the model. However, the training objective for $\hat{\vepsilon}_n(\vx_n)$ in Eq.~{(\ref{eq:mse_eps})} is the MSE loss between $\hat{\vepsilon}_n(\vx_n)$ and $\E_{q(\vx_0|\vx_n)}[\vepsilon_n]$ instead of their squares. Therefore, the error between $\hat{\vepsilon}_n(\vx_n)^2$ and $\E_{q(\vx_0|\vx_n)}[\vepsilon_n]^2$ is likely to be amplified.
Specifically, the error of $\hat{\vepsilon}_n(\vx_n)^2$ at $i$-th component is $|\hat{\vepsilon}_n(\vx_n)_i^2 - \E_{q(\vx_0|\vx_n)}[\vepsilon_n]_i^2| = A_i | \hat{\vepsilon}_n(\vx_n)_i - \E_{q(\vx_0|\vx_n)}[\vepsilon_n]_i |$, where $A_i = |\hat{\vepsilon}_n(\vx_n)_i + \E_{q(\vx_0|\vx_n)}[\vepsilon_n]_i|$. Thus, the error of $\hat{\vepsilon}_n(\vx_n)^2$ at $i$-th component is amplified by a factor of $A_i$. In practice, the value of $A_i$ is not guaranteed to be bounded, and the error is likely to be amplified by a large factor. % Note that $A_i \leq \|\hat{\vepsilon}_n(\vx_n)\|_\infty + \|\E_{q(\vx_0|\vx_n)}[\vepsilon_n]\|_\infty =: A$. By summing $\gE_i^2$, we can bound the $\ell_2$ error of $\hat{\vepsilon}_n(\vx_n)^2$ as $\|\hat{\vepsilon}_n(\vx_n)^2 - \E_{q(\vx_0|\vx_n)}[\vepsilon_n]^2\|_2 = \sqrt{\sum_i \gE_i^2} \leq A \| \hat{\vepsilon}_n(\vx_n) - \E_{q(\vx_0|\vx_n)}[\vepsilon_n] \|_2$.

In contrast, NPR-DPM doesn't have the error amplification problem, since it uses the model $\vg_n(\vx_n)$ itself to estimate the target $\E_{q(\vx_0|\vx_n)}[(\vepsilon_n - \hat{\vepsilon}_n(\vx_n))^2]$, and directly minimizes the MSE loss between them.

\subsection{NPR-DPM is Not Always Better Than SN-DPM}
\label{sec:not_always_better}
We note that SN-DPM is better w.r.t. FID and NPR-DPM is better w.r.t. likelihood.
The performance of NPR-DPM on likelihood is guaranteed, since it aims to optimize the log-likelihood with the actual imperfect mean. Note the sample quality and likelihood are not necessarily consistent (see Section 3.2 in \citet{theis2015note}), and NPR-DPM does not directly optimize the sample quality. Thus, the performance of NPR-DPM on FID is not guaranteed.

In practice, the NPR-DPM is useful in applications such as lossless compression (see Section 7.3 in \citet{kingma2021variational}), where the likelihood performance and the computation efficiency play a central role.

\section{Experimental Details}
\label{sec:exp_detail}

\subsection{Details of Pretrained Noise Prediction Networks}
\label{sec:pretrained_model}

In Table~\ref{tab:detail_model}, we list details of pretrained noise prediction networks used in our experiments. The ImageNet 64x64 pretrained model includes a noise prediction network and additionally a covariance network, and we only use the former one.

\begin{table}[H]
\centering
\caption{Details of noise prediction networks used in our experiments. LS means the linear schedule of $\beta_n$~\cite{ho2020denoising} in the forward process of discrete timesteps (see Eq.~{(\ref{eq:ddim})}). CS means the cosine schedule of $\beta_n$~\cite{nichol2021improved} in the forward process of discrete timesteps (see Eq.~{(\ref{eq:ddim})}). The VP SDE is defined as $\mathrm{d}\vx = -\frac{1}{2}\beta(t)\vx \mathrm{d} t + \sqrt{\beta(t)} \mathrm{d}\vw$, where $\beta(t)$ is a linear function of $t$ with $\beta(0)=0.1$ and $\beta(1)=20$~\cite{song2020score}.}
\label{tab:detail_model}
\vskip 0.15in
\begin{small}
\begin{sc}
\begin{tabular}{lcccc}
\toprule
    & Timesteps Type & \# timesteps $N$ & forward process & provided by \\
\midrule
    CIFAR10 (LS) & Discrete & 1000 & LS & \citet{bao2022analytic} \\
    CIFAR10 (CS) & Discrete & 1000 & CS & \citet{bao2022analytic} \\
    CIFAR10 (VP SDE) & Continuous & -- & VP SDE & \citet{song2020score} \\
    CelebA 64x64 & Discrete & 1000 & LS & \citet{song2020denoising} \\
    ImageNet 64x64 & Discrete & 4000 & CS & \citet{nichol2021improved} \\
    LSUN Bedroom & Discrete & 1000 & LS & \citet{ho2020denoising} \\
\bottomrule
\end{tabular}
\end{sc}
\end{small}
\end{table}

\subsection{Structure Details of Prediction Networks}

In Table~\ref{tab:nn2}, we list structure details of $\mathrm{NN}_1$ and $\mathrm{NN}_2$ of prediction networks used in our experiments.

\begin{table}[H]
\centering
\caption{$\mathrm{NN}_1$ of pretrained noise prediction networks and $\mathrm{NN}_2$ of SN \& NPR prediction networks used in our experiments. Conv denotes the convolution layer. Res denotes the residual block.}
\label{tab:nn2}
\vskip 0.15in
\begin{small}
\begin{sc}
\begin{tabular}{lccc}
\toprule
    & $\mathrm{NN}_1$ & $\mathrm{NN}_2$ (SN) & $\mathrm{NN}_2$ (NPR) \\
\midrule
    CIFAR10 (LS) & Conv & Conv & Conv \\
    CIFAR10 (CS) & Conv & Conv & Conv  \\
    CIFAR10 (VP SDE) & Conv & Conv & Conv  \\
    CelebA 64x64 & Conv & Conv & Conv  \\
    ImageNet 64x64 & Conv & Res+Conv & Conv  \\
    LSUN Bedroom & Conv & Res+Conv & Conv  \\
\bottomrule
\end{tabular}
\end{sc}
\end{small}
\end{table}

\subsection{Details of Memory and Time Cost}
\label{sec:mem_time}

In Table~\ref{tab:mem_time}, we list the memory and time cost of models (with the corresponding methods) used in our experiments. The extra memory cost of the NPR (or SN) prediction network is negligible. The extra time cost of the NPR (or SN) prediction network is negligible on CIFAR10, CelebA 64x64, at most 4.5\% on ImageNet 64x64, and at most 10\% on LSUN Bedroom.

\begin{table}[H]
\centering
\caption{Model size (MB) and the averaged time (ms) to run a model function evaluation (i.e., the time to compute noise prediction network for baselines, and the time to compute noise \& NPR (or SN) prediction networks that share parameters for our method). We also show the extra time cost of the NPR (or SN) prediction network relative to that of the noise prediction network in parentheses. All are evaluated with a batch size of 10 on one GeForce RTX 2080 Ti.}
\label{tab:mem_time}
\vskip 0.15in
\begin{small}
\begin{sc}
\begin{tabular}{lccc}
\toprule
    & \makecell{Noise prediction \\ network \\ (all baselines)} & \makecell{Noise \& SN prediction \\ networks \\ (\sndpm)} & \makecell{Noise \& NPR prediction \\ networks \\ (\nprdpm)} \\
\midrule
    CIFAR10 (LS) & 200.44 MB / 25.65 ms & 200.45 MB / 25.77 ms (+0.5\%) & 200.45 MB / 25.71 ms (+0.2\%) \\
    CIFAR10 (CS) & 200.44 MB / 25.26 ms & 200.45 MB / 25.62 ms (+1.4\%) & 200.45 MB / 25.64 ms (+1.5\%) \\
    CIFAR10 (VP SDE) & 235.77 MB / 26.98 ms & 235.78 MB / 27.17 ms (+0.7\%) & 235.78 MB / 27.17 ms (+0.7\%) \\
    CelebA 64x64 & 300.22 MB / 48.64 ms & 300.23 MB / 49.20 ms (+1.2\%) & 300.23 MB / 49.18 ms (+1.1\%) \\
    ImageNet 64x64 & 461.82 MB / 72.32 ms & 463.47 MB / 75.56 ms (+4.5\%) & 461.84 MB / 72.80 ms (+0.7\%) \\
    LSUN Bedroom & 433.63 MB / 441.41 ms & 435.02 MB / 485.54 ms (+10.0\%) & 433.64 MB / 448.74 ms (+1.7\%) \\
\bottomrule
\end{tabular}
\end{sc}
\end{small}
\end{table}

\subsection{Training Details}

We use a similar training setting to that of the noise prediction network in \citet{bao2022analytic}. On all datasets, we use the AdamW optimizer~\cite{loshchilov2017decoupled} with a learning rate of 0.0001; we train 500K iterations; we use an exponential moving average (EMA) with a rate of 0.9999. We use a batch size of 64 on LSUN Bedroom, and 128 on other datasets. We save a checkpoint every 10K iterations and select the one with the best FID on 1000 generated samples. By default, these samples are generated with full timesteps, except for LSUN Bedroom. On LSUN Bedroom, these samples are generated with 100 timesteps for an acceptable time cost (the time cost on LSUN Bedroom is much larger than other datasets as shown in Table~\ref{tab:mem_time}).

Training a NPR (or SN) network on CIFAR10 takes about 32 hours on one GeForce RTX 2080 Ti. Training on CelebA 64x64 takes about 72 hours on two GeForce RTX 2080 Ti. Training on ImageNet 64x64 takes about 83 hours on two GeForce RTX 2080 Ti. Training on LSUN Bedroom takes about 171 hours on four GeForce RTX 2080 Ti.

\subsection{Log-Likelihood and Sampling}

The image data is a vector of integers in $\{0, 1, \cdots, 255\}$ and we linearly scale it to $[-1, 1]$ following \citet{ho2020denoising,bao2022analytic}.

\textbf{Log-likelihood.} Following \citet{ho2020denoising,bao2022analytic}, we discretize the last Markov transition $p(\vx_0|\vx_1)$ to get discrete NLL and its upper bound. The likelihood results are evaluated on the whole test dataset by default.

\textbf{Sampling.} Following \citet{ho2020denoising,bao2022analytic}, we only display the mean of $p(\vx_0|\vx_1)$ noiselessly at the end of sampling. Following \citet{bao2022analytic}, we clip the covariance $\vsigma_2(\vx_2)^2$ of $p(\vx_1|\vx_2)$ by infinity norm, such that $\|\vsigma_2(\vx_2)\|_\infty \E|\epsilon| \leq \frac{2}{255} y$, where $\|\vsigma_2(\vx_2)\|_\infty$ denotes the infinity norm, $\epsilon$ is the standard Gaussian noise and $y$ is the maximum tolerated perturbation of a channel. Following \citet{bao2022analytic}, we use $y=2$ on CIFAR10 (LS) and CelebA 64x64 under the DDPM forward process, and use $y=1$ for other cases. Following \citet{bao2022analytic}, we calculate the FID score on 50K generated samples, using the official implementation of FID to pytorch (\url{https://github.com/mseitzer/pytorch-fid}). Following \citet{nichol2021improved,bao2022analytic}, the reference distribution statistics of FID are computed on the full training
set on CIFAR10 and ImageNet 64x64, and 50K training samples on CelebA 64x64 and LSUN Bedroom.

\subsection{Number of Monte Carlo Samples to Calculate the Optimal Trajectory}

Calculating the optimal trajectory needs to estimate every decomposed term appeared in $\lelbo$, which involves a Monte Carlo estimate calculated on a set of training samples~\cite{watson2021learning,bao2022analytic}. Following \citet{bao2022analytic}, we use 50K samples on CIFAR10, 10K samples on CelebA 64x64 and ImageNet 64x64.

\subsection{Experimental Details of Table~{\ref{tab:fid}, \ref{tab:timesteps} \& \ref{tab:bpd}}}
\label{sec:detail_timesteps}

In Table~{\ref{tab:fid} \& \ref{tab:bpd}}, the results of baselines and our methods are based on the same noise prediction networks (i.e., those listed in Table~\ref{tab:detail_model}). The baseline results in DPMs with discrete timesteps are provides by \citet{bao2022analytic}. As for DPMs with continuous timesteps, we get results of the Euler-Maruyama solver and the probability flow by running the pytorch official code (\url{https://github.com/yang-song/score_sde_pytorch}) of \citet{song2020score}, and get results of the ancestral sampling and Analytic-DPM using our implementation.

In Table~\ref{tab:timesteps}, the results of DDPM, DDIM, Analytic-DPM and our methods are based on the same noise prediction networks (i.e., those listed in Table~\ref{tab:detail_model}). The time cost in a single timestep is the same for all baselines (DDPM, DDIM, Improved DDPM and Analytic-DPM), based on the statistics reported in Table 4 in \citet{bao2022analytic}. The ratio of the time cost in a single timestep in \nprdpm{} and \sndpm{} to that in baselines is based on Table~\ref{tab:mem_time}. The result of baselines are provided by \citet{bao2022analytic}. The results of \nprdpm{} and \sndpm{} are based on the DDPM forward process on LSUN Bedroom, and based on the DDIM forward process on other
datasets. These choices require smaller number of timesteps than their alternatives.

\section{Additional Experiments}

\subsection{Additional Likelihood Comparison}
\label{sec:ll_iddpm}

We compare our \nprddpm{} and Improved DDPM~\cite{nichol2021improved} based on the ImageNet 64x64 model mentioned in Appendix~\ref{sec:pretrained_model}. Improved DDPM parameterizes the diagonal covariance using an interpolation between $\beta_n$ and $\tilde{\beta}_n$ and learn the covariance by directly optimizing $\lelbo$. As shown in Table~\ref{tab:improved_ddpm}, our \nprddpm{} can get a reasonable likelihood result with a small number of timesteps, which is contrast to Improved DDPM. With the full timesteps, two methods show a similar likelihood performance. This implies that using an alternative objective such as the MSE loss can achieve a similar likelihood performance to directly optimizing $\lelbo$.

\begin{table}[H]
\centering
\caption{Upper bound ($-\lelbo$) on the negative log-likelihood (bits/dim) $\downarrow$, under the setting of the DDPM forward process on ImageNet 64x64 and the even trajectory.
}
\label{tab:improved_ddpm}
\vskip 0.15in
\begin{small}
\begin{sc}
\begin{tabular}{lrrrrrrrr}
\toprule
Model $\backslash$ \# timesteps $K$ & 25 & 50 & 100 & 200 & 400 & 1000 & 4000 \\
\midrule
Improved DDPM & 18.91 & 8.46 & 5.27 & 4.24 & 3.86 & 3.68 & \textbf{3.57} \\
\nprddpm & \textbf{4.66} & \textbf{4.22} & \textbf{3.96} & \textbf{3.80} & \textbf{3.71} & \textbf{3.64} & 3.60 \\
\arrayrulecolor{black}\bottomrule
\end{tabular}
\end{sc}
\end{small}
\end{table}

\subsection{Variance of Likelihood Results}
\label{sec:ll_var}

We note that under the full timesteps, the likelihood results of Analytic-DDPM and \nprddpm{} in Table~\ref{tab:bpd} are close. To reduce the effect of randomness, we report the standard deviation of the likelihood results under the full timesteps in Table~\ref{tab:ll_var}.

\begin{table}[H]
\caption{Upper bound ($-\lelbo$) on the negative log-likelihood (bits/dim) $\downarrow$. We repeat 5 times with different seeds and report the mean and the standard deviation. Each time we randomly select 10K test samples.}
\label{tab:ll_var}
\vskip 0.15in
\centering
\begin{small}
\begin{sc}
\begin{tabular}{lcccc}
\toprule
     & CIFAR10 (LS) & CIFAR10 (CS) & CelebA 64X64 & ImageNet 64X64 \\
\midrule
    Analytic-DDPM & 3.58593±0.00028 & 3.42229±0.00019 & 2.65657±0.00126 & 3.61749±0.00745 \\
    \nprddpm & 3.57204±0.00039 & 3.41013±0.00031 & 2.65257±0.00113 & 3.59660±0.00793 \\
\bottomrule
\end{tabular}
\end{sc}
\end{small}
\end{table}

\subsection{Comparison After Normalizing for the Extra Cost}
\label{sec:norm_cmp}

We present results after normalizing for the extra inference cost. As shown in Table~\ref{tab:fid_cmp_norm}, our method still outperforms the strong baseline Analytic-DPM. We also compare with the ``Gotta Go Fast'' SDE solver~\cite{jolicoeur2021gotta}, and our method performs better when the NFE is small (e.g., $\leq 49$), and thereby is more efficient. When the NFE is large (e.g., $\geq 147$), different methods have similar performance.

We also present results after normalizing for the extra training cost. On CIFAR10 (LS), we train a noise prediction network with 300K iterations and train a NPR (or SN) prediction network with 200K iterations. This ensures that the total training cost is strictly less than that of the original CIFAR10 (LS) model provided by \citet{bao2022analytic}, which is trained with 500K iterations. With 24 and 49 timesteps respectively, NPR-DDPM gets a FID of 10.56 and 6.24; SN-DDPM gets a FID of 7.51 and 4.64. The results are not affected much by the normalization for the extra training cost, and our method still outperforms the strong baseline Analytic-DPM.

\begin{table}[H]
\centering
\caption{FID comparison after normalizing for the extra inference cost. The normalized number of timesteps and model function evaluations (NFE)  are displayed in parentheses. The normalization ensures the time cost of our method is \emph{strictly less} than that of the baselines.}
\vskip 0.05in
\label{tab:fid_cmp_norm}
\begin{small}
\begin{sc}
\setlength{\tabcolsep}{5.1pt}{
\begin{tabular}{lrrrrrrrrrr}
\toprule
 & \multicolumn{2}{c}{CIFAR10 (LS)} & \multicolumn{2}{c}{CIFAR10 (CS)} & \multicolumn{2}{c}{CelebA 64x64} & \multicolumn{2}{c}{ImageNet 64x64} & \multicolumn{2}{c}{CIFAR10 (VP SDE)} \\
\cmidrule(lr){2-3} \cmidrule(lr){4-5} \cmidrule(lr){6-7} \cmidrule(lr){8-9} \cmidrule(lr){10-11}
\# timesteps & 
25 (24) & 50 (49) &
25 (24) & 50 (49) &
25 (24) & 50 (49) &
25 (23) & 50 (47) &
25 (24) & 50 (49) \\
\midrule
Analytic-DDPM & 
11.60 & 7.25 & 
8.50 & 5.50 &
16.01 & 11.23 &
32.56 & 22.45 &
11.57 & 6.54 \\
\nprddpm & 
10.82 & 6.28 & 
7.98 & 5.33 & 
16.00 & 11.06 &
29.72 & 21.52  &
10.88 & 5.80 \\
\snddpm & 
\textbf{7.42} & \textbf{4.60} &
\textbf{6.15} & \textbf{4.20} &
\textbf{12.11} & \textbf{8.08} & 
\textbf{29.05} & \textbf{20.80} & 
\textbf{7.94} & \textbf{4.47} \\
\bottomrule
\end{tabular}}
\end{sc}
\end{small}
\vskip 0.05in
\begin{small}
\begin{sc}
\setlength{\tabcolsep}{5.9pt}{
\begin{tabular}{lrrrrrrr}
\toprule
\makecell{NFE on CIFAR10 (VP SDE)} & 29 (28) & 
39 (38) & 49 (48) & 147 (145) &
179 (177) & 274 (272) &
329 (326) \\
\midrule
Gotta Go Fast (official code) & 247.79 & 116.91 & 72.29 & \textbf{2.95} & \textbf{2.59} & 2.74 & 2.70 \\
\nprddpm & 9.22 & 6.97 & 5.88 & 3.44 & 3.19 & 2.91 & 2.94 \\
\snddpm & \textbf{6.62} & \textbf{5.21} & \textbf{4.55} & 2.96 & 2.85 & \textbf{2.67} & \textbf{2.64} \\
\bottomrule
\end{tabular}}
\end{sc}
\end{small}
\end{table}

\subsection{Samples}
\label{sec:samples}
Recall that Table~\ref{tab:timesteps} reports the least number of timesteps required to achieve a FID around 6 on CIFAR10, CelebA 64x64 and LSUN Bedroom. In Figure~{\ref{fig:cifar10}-\ref{fig:lsun}}, we show generated samples of \nprdpm{} and \sndpm{} under these number of timesteps on these datasets. Here we use $K$ to denote the number of timesteps. In Figure~\ref{fig:imagenet}, we also show generated samples on ImageNet 64x64 under $K=25$ timesteps and the DDPM forward process.

In Figure~{\ref{fig:K_cifar10_ls}-\ref{fig:K_imagenet64}}, we show more generated samples from both \nprdpm{} and \sndpm{} under both the DDPM and DDIM forward processes, as well as the VP SDE.

\begin{figure}[H]
\begin{center}
\subfloat[Samples from dataset]{\includegraphics[width=0.31\linewidth]{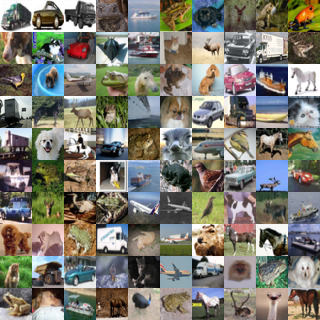}}\ 
\subfloat[\nprdpm{} ($K=23$)]{\includegraphics[width=0.31\linewidth]{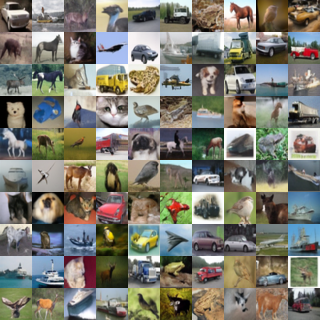}}\ 
\subfloat[\sndpm{} ($K=17$)]{\includegraphics[width=0.31\linewidth]{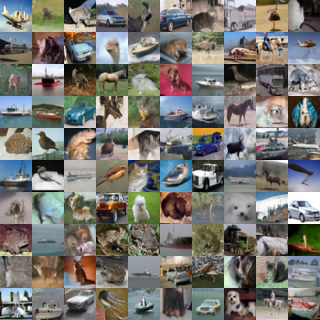}}
\vspace{-.1cm}
\caption{Generated samples on CIFAR10.}
\label{fig:cifar10}
\end{center}
\end{figure}

\newpage
\begin{figure}[H]
\begin{center}
\subfloat[Samples from dataset]{\includegraphics[width=0.31\linewidth]{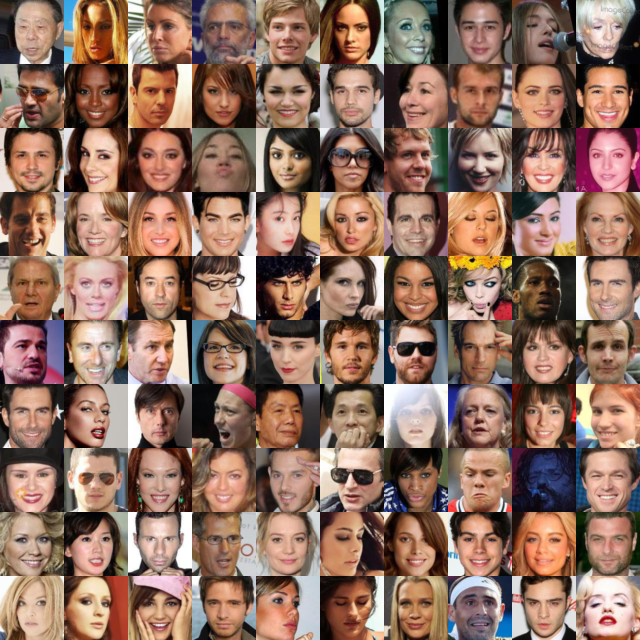}}\ 
\subfloat[\nprdpm{} ($K=50$)]{\includegraphics[width=0.31\linewidth]{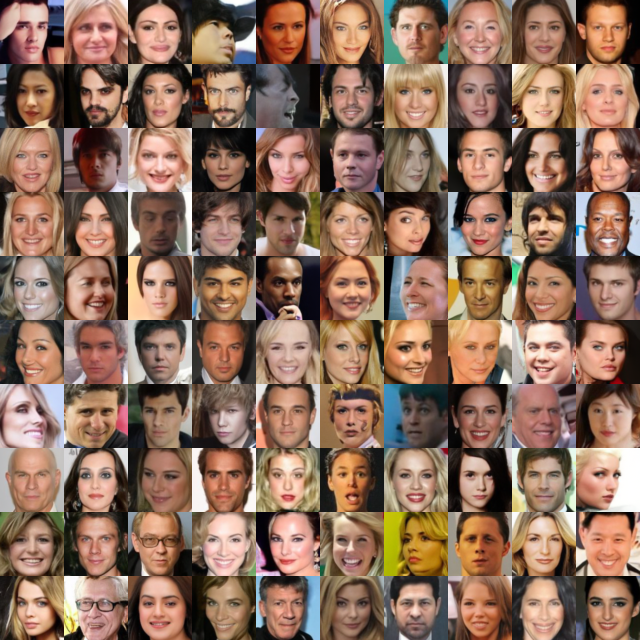}}\ 
\subfloat[\sndpm{} ($K=22$)]{\includegraphics[width=0.31\linewidth]{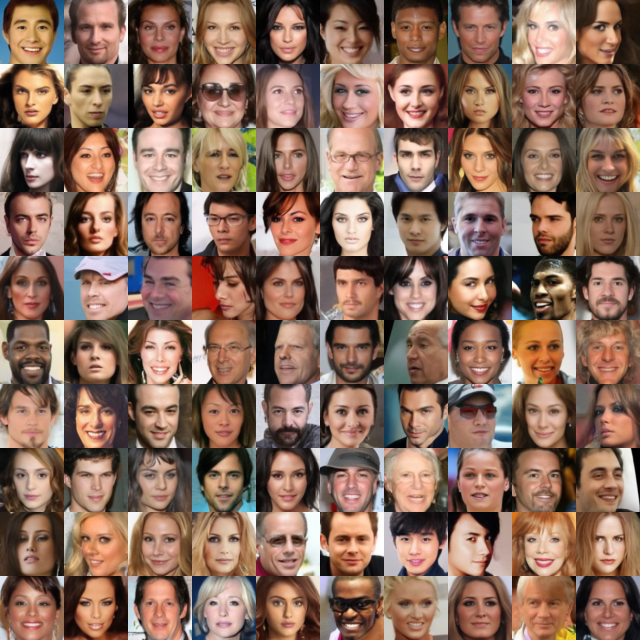}}
\vspace{-.1cm}
\caption{Generated samples on CelebA 64x64.}
\label{fig:celeba}
\end{center}
\end{figure}

\begin{figure}[H]
\begin{center}
\subfloat[Samples from dataset]{\includegraphics[width=0.31\linewidth]{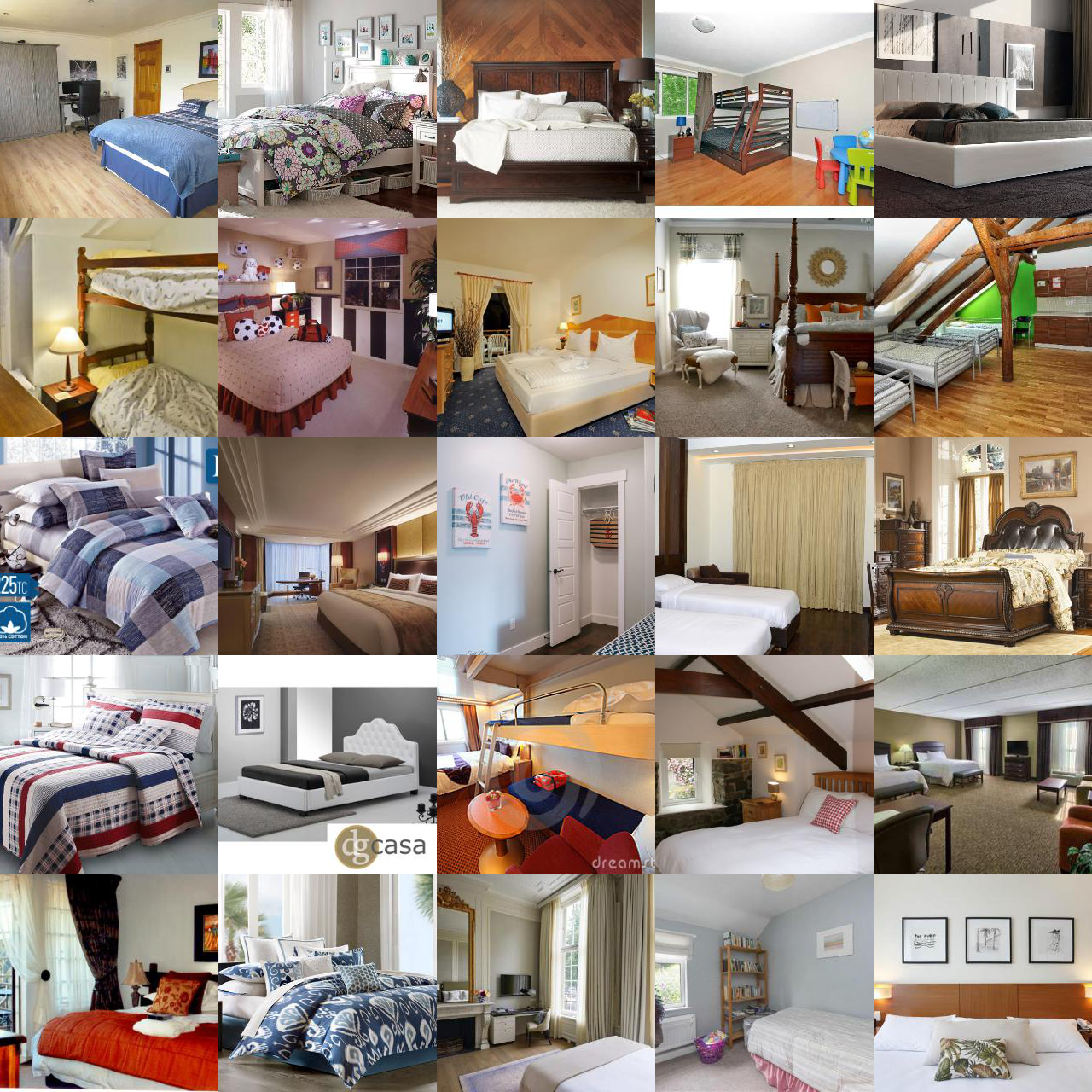}}\ 
\subfloat[\nprdpm{} ($K=90$)]{\includegraphics[width=0.31\linewidth]{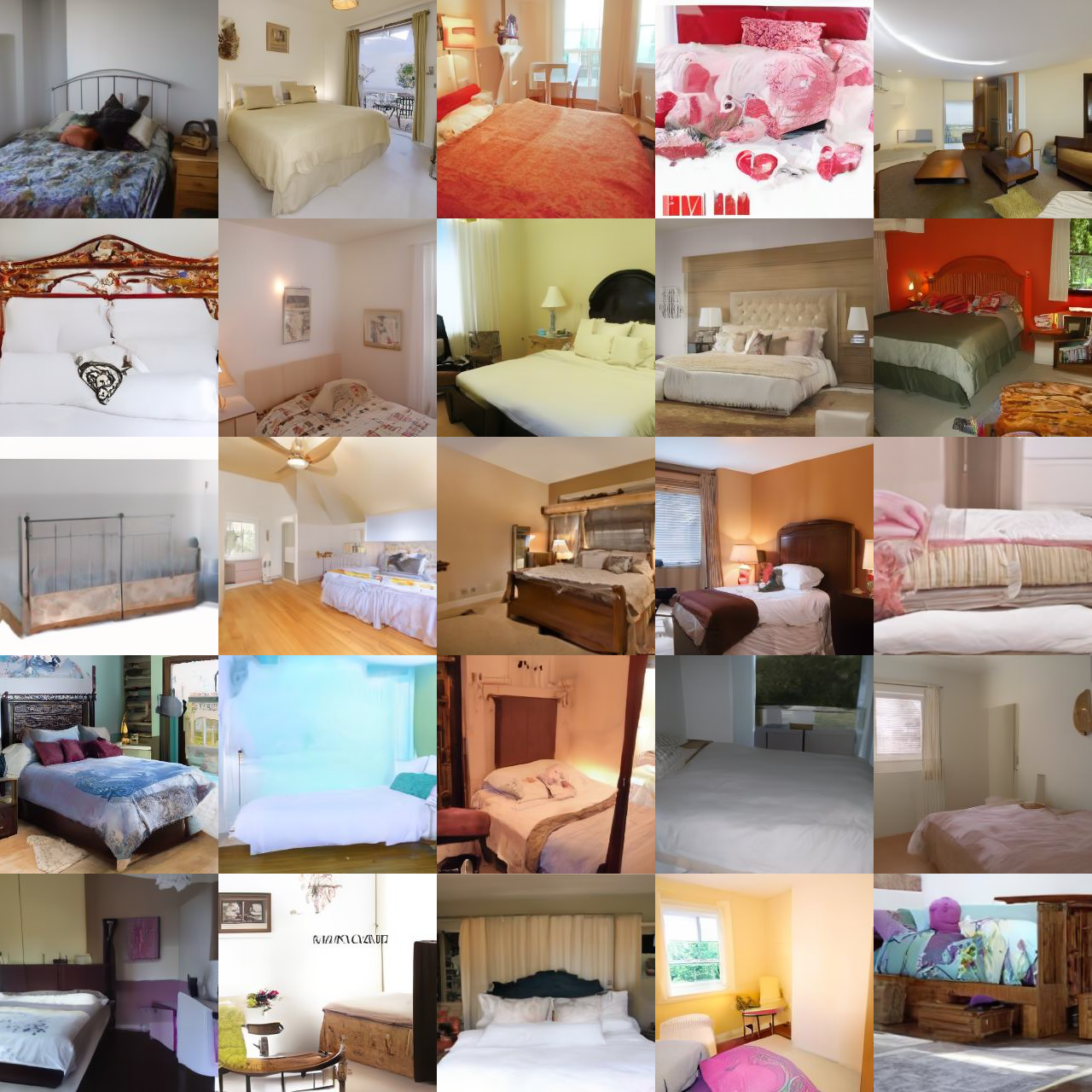}}\ 
\subfloat[\sndpm{} ($K=92$)]{\includegraphics[width=0.31\linewidth]{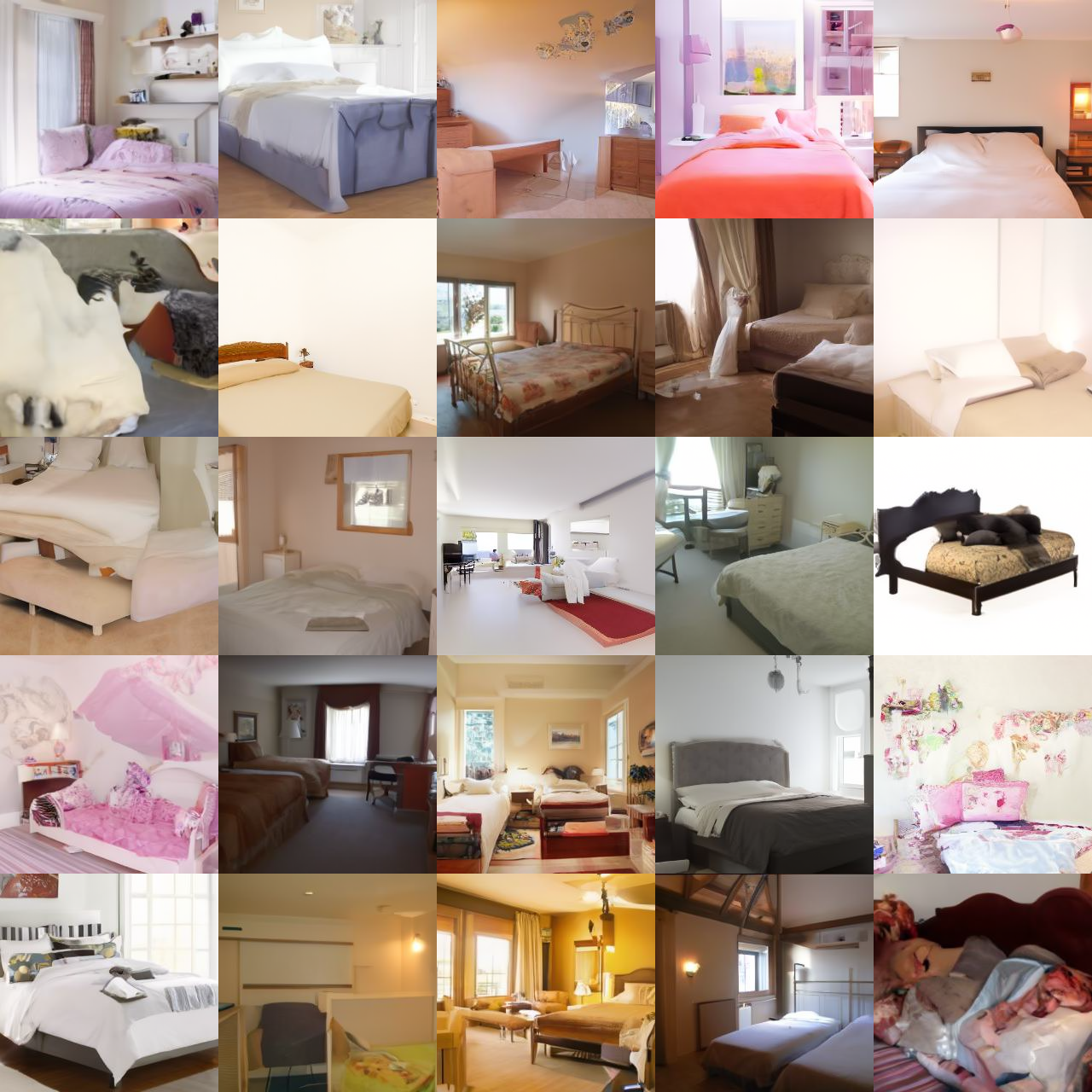}}
\vspace{-.1cm}
\caption{Generated samples on LSUN Bedroom.}
\label{fig:lsun}
\end{center}
\end{figure}

\begin{figure}[H]
\begin{center}
\subfloat[Samples from dataset]{\includegraphics[width=0.31\linewidth]{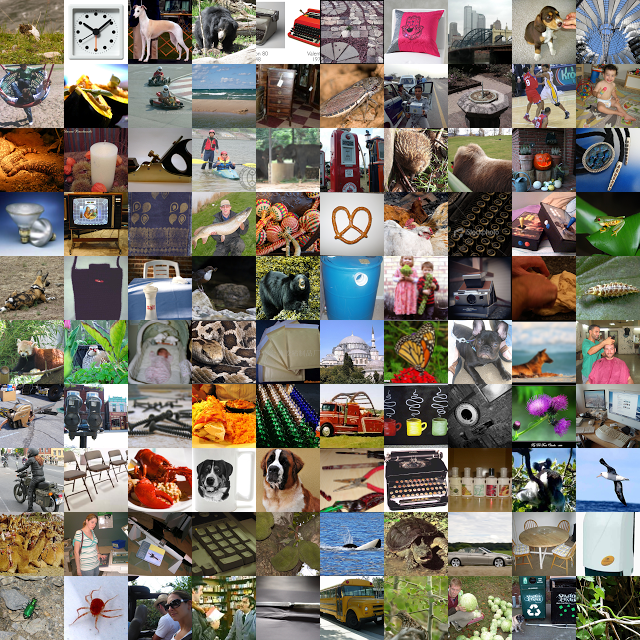}}\ 
\subfloat[\nprdpm{} ($K=25$)]{\includegraphics[width=0.31\linewidth]{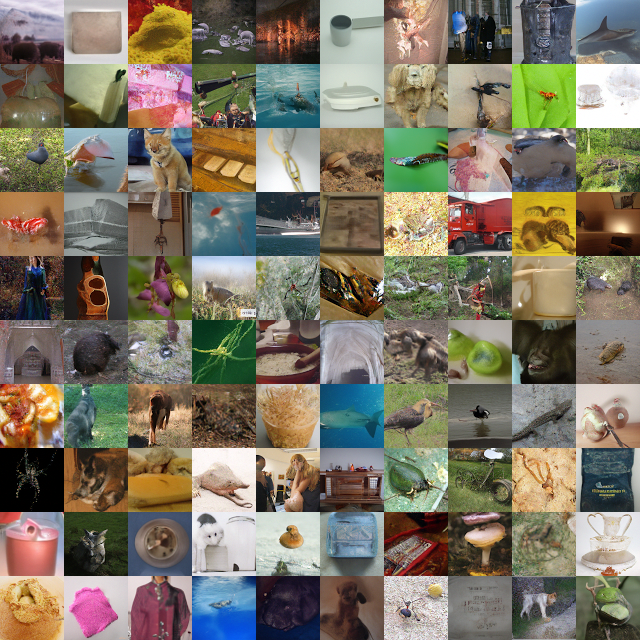}}\ 
\subfloat[\sndpm{} ($K=25$)]{\includegraphics[width=0.31\linewidth]{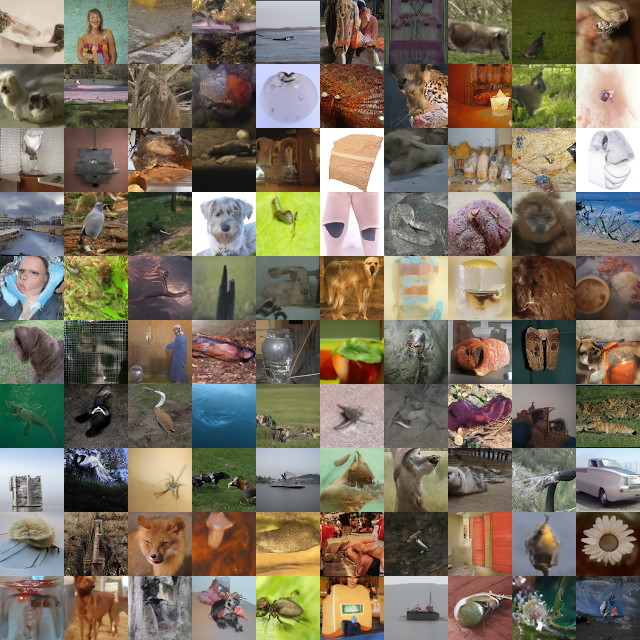}}
\vspace{-.1cm}
\caption{Generated samples on ImageNet 64x64.}
\label{fig:imagenet}
\end{center}
\end{figure}

\newpage
\begin{figure}[H]
\centering
\begin{minipage}{\textwidth}
\begin{center}
\subfloat[\nprddpm{} ($K=10$)]{\includegraphics[width=0.24\linewidth]{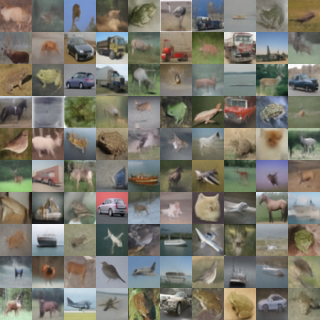}}\ 
\subfloat[\nprddpm{} ($K=100$)]{\includegraphics[width=0.24\linewidth]{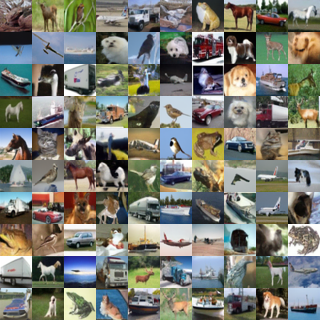}}\ 
\subfloat[\snddpm{} ($K=10$)]{\includegraphics[width=0.24\linewidth]{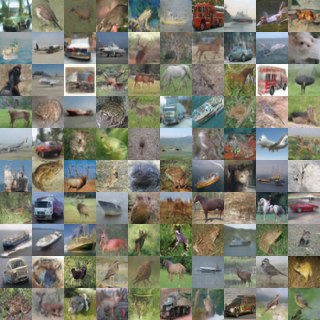}}\ 
\subfloat[\snddpm{} ($K=100$)]{\includegraphics[width=0.24\linewidth]{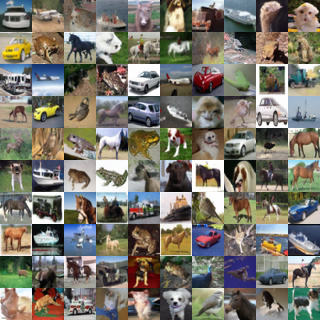}} \\
\vspace{-.2cm}
\subfloat[\nprddim{} ($K=10$)]{\includegraphics[width=0.24\linewidth]{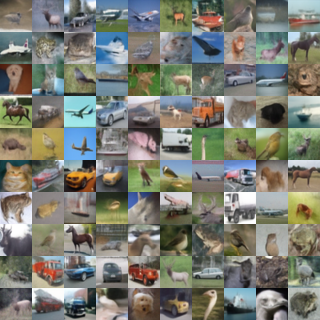}}\ 
\subfloat[\nprddim{} ($K=100$)]{\includegraphics[width=0.24\linewidth]{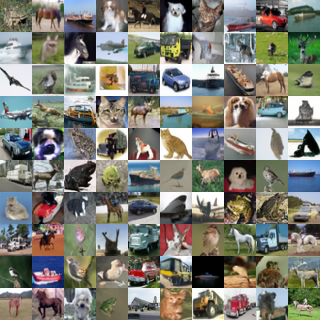}}\ 
\subfloat[\snddim{} ($K=10$)]{\includegraphics[width=0.24\linewidth]{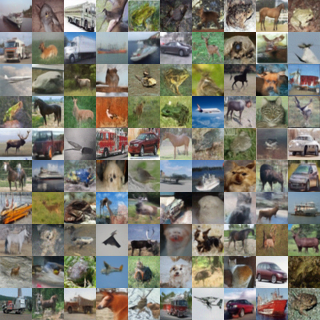}}\ 
\subfloat[\snddim{} ($K=100$)]{\includegraphics[width=0.24\linewidth]{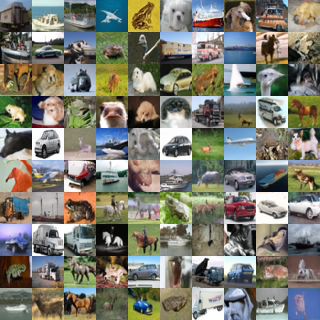}}
\vspace{-.1cm}
\caption{Generated samples on CIFAR10 (LS).}
\label{fig:K_cifar10_ls}
\end{center}
\end{minipage}\\
\begin{minipage}{\textwidth}
\begin{center}
\subfloat[\nprddpm{} ($K=10$)]{\includegraphics[width=0.24\linewidth]{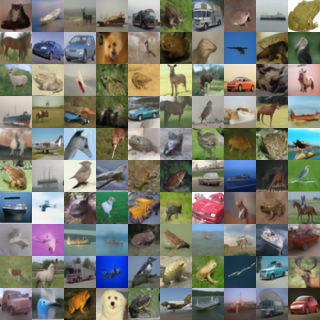}}\ 
\subfloat[\nprddpm{} ($K=100$)]{\includegraphics[width=0.24\linewidth]{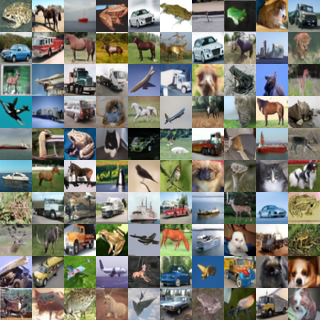}}\ 
\subfloat[\snddpm{} ($K=10$)]{\includegraphics[width=0.24\linewidth]{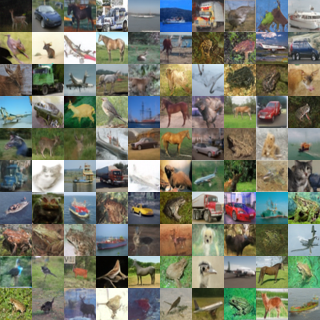}}\ 
\subfloat[\snddpm{} ($K=100$)]{\includegraphics[width=0.24\linewidth]{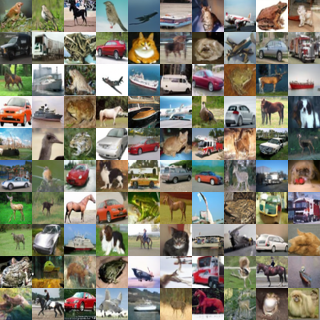}} \\
\vspace{-.2cm}
\subfloat[\nprddim{} ($K=10$)]{\includegraphics[width=0.24\linewidth]{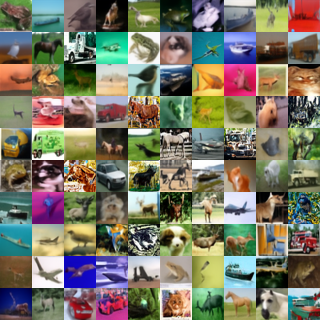}}\ 
\subfloat[\nprddim{} ($K=100$)]{\includegraphics[width=0.24\linewidth]{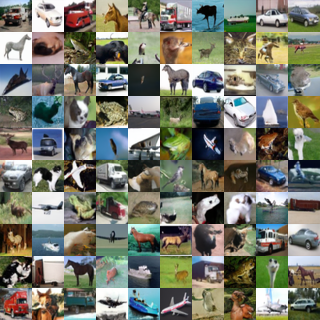}}\ 
\subfloat[\snddim{} ($K=10$)]{\includegraphics[width=0.24\linewidth]{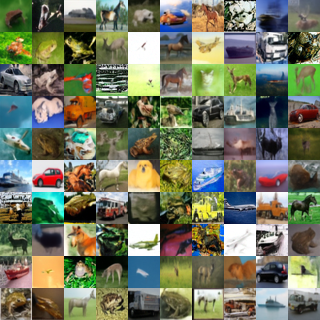}}\ 
\subfloat[\snddim{} ($K=100$)]{\includegraphics[width=0.24\linewidth]{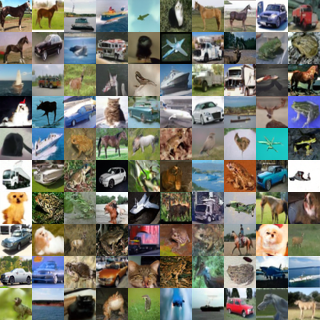}}
\vspace{-.1cm}
\caption{Generated samples on CIFAR10 (CS).}
\label{fig:K_cifar10_cs}
\end{center}
\end{minipage}
\end{figure}

\newpage
\begin{figure}[H]
\centering
\begin{minipage}{\textwidth}
\begin{center}
\subfloat[\nprdpm{} ($K=10$)]{\includegraphics[width=0.24\linewidth]{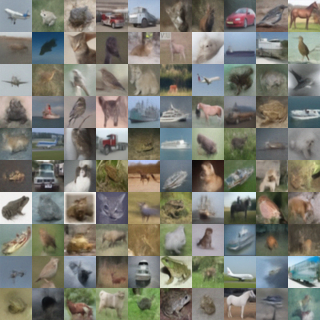}}\ 
\subfloat[\nprdpm{} ($K=100$)]{\includegraphics[width=0.24\linewidth]{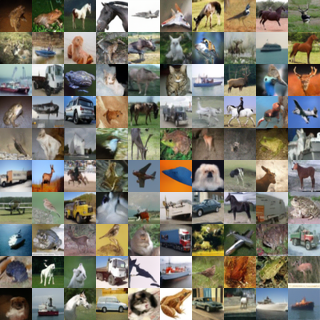}}\ 
\subfloat[\sndpm{} ($K=10$)]{\includegraphics[width=0.24\linewidth]{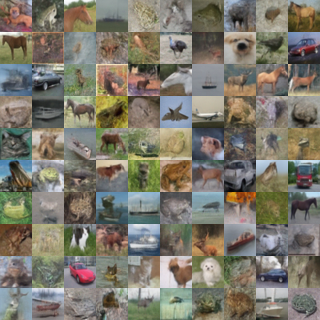}}\ 
\subfloat[\sndpm{} ($K=100$)]{\includegraphics[width=0.24\linewidth]{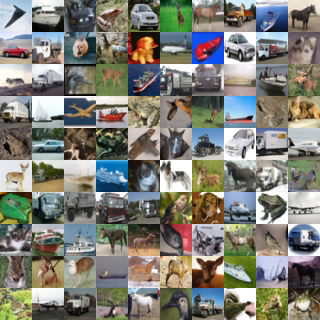}}
\vspace{-.1cm}
\caption{Generated samples on CIFAR10 (VP SDE).}
\label{fig:K_cifar10_vpsde}
\end{center}
\end{minipage}\\
\begin{minipage}{\textwidth}
\begin{center}
\subfloat[\nprddpm{} ($K=10$)]{\includegraphics[width=0.24\linewidth]{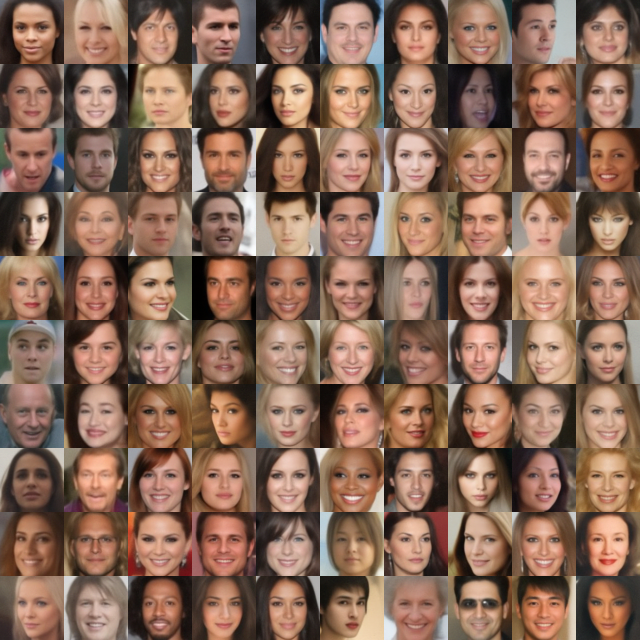}}\ 
\subfloat[\nprddpm{} ($K=100$)]{\includegraphics[width=0.24\linewidth]{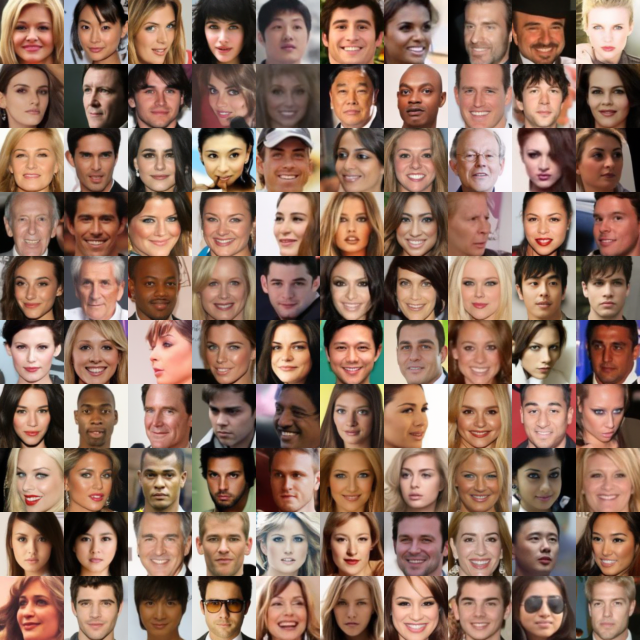}}\ 
\subfloat[\snddpm{} ($K=10$)]{\includegraphics[width=0.24\linewidth]{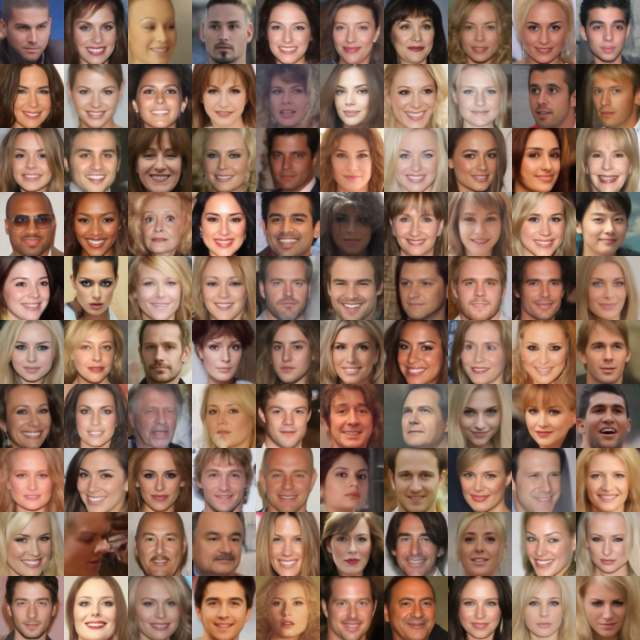}}\ 
\subfloat[\snddpm{} ($K=100$)]{\includegraphics[width=0.24\linewidth]{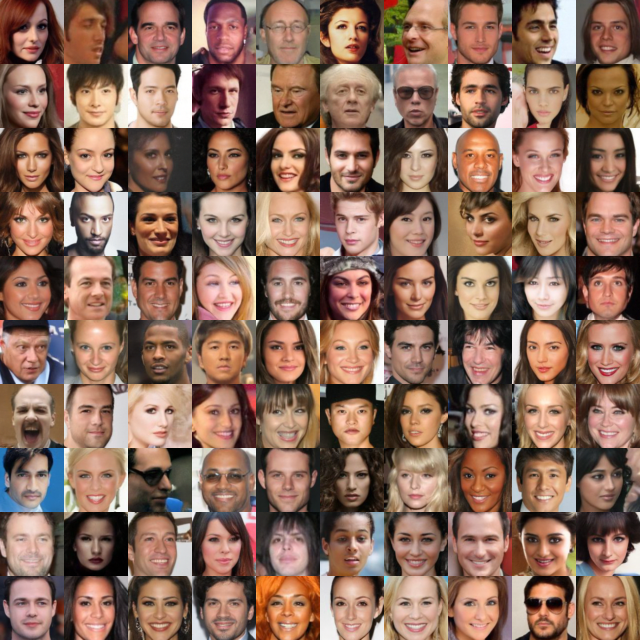}} \\
\vspace{-.2cm}
\subfloat[\nprddim{} ($K=10$)]{\includegraphics[width=0.24\linewidth]{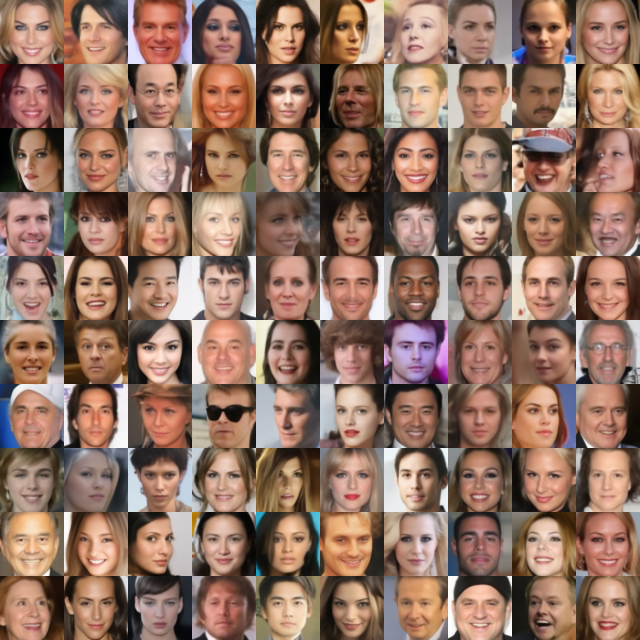}}\ 
\subfloat[\nprddim{} ($K=100$)]{\includegraphics[width=0.24\linewidth]{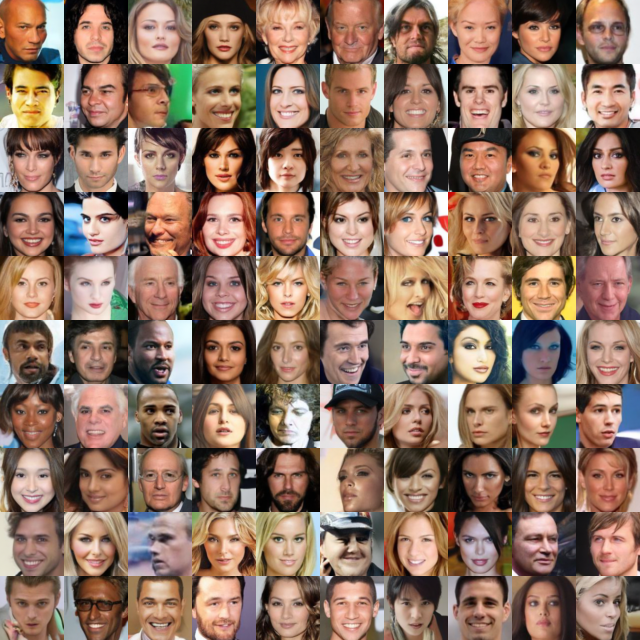}}\ 
\subfloat[\snddim{} ($K=10$)]{\includegraphics[width=0.24\linewidth]{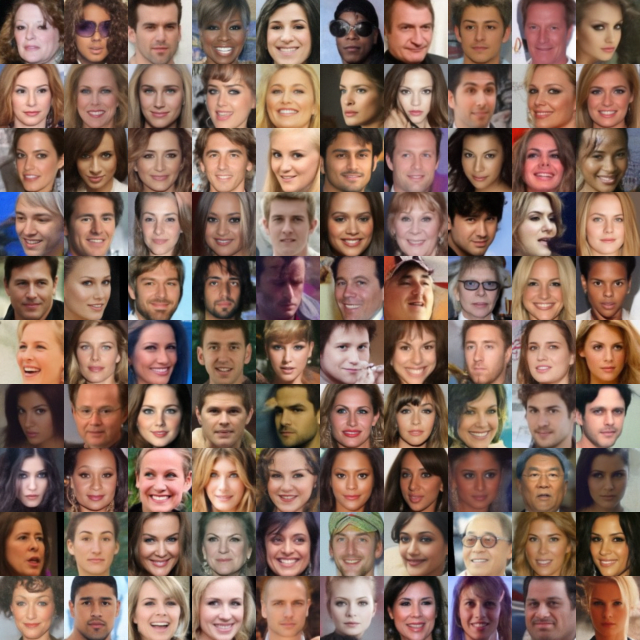}}\ 
\subfloat[\snddim{} ($K=100$)]{\includegraphics[width=0.24\linewidth]{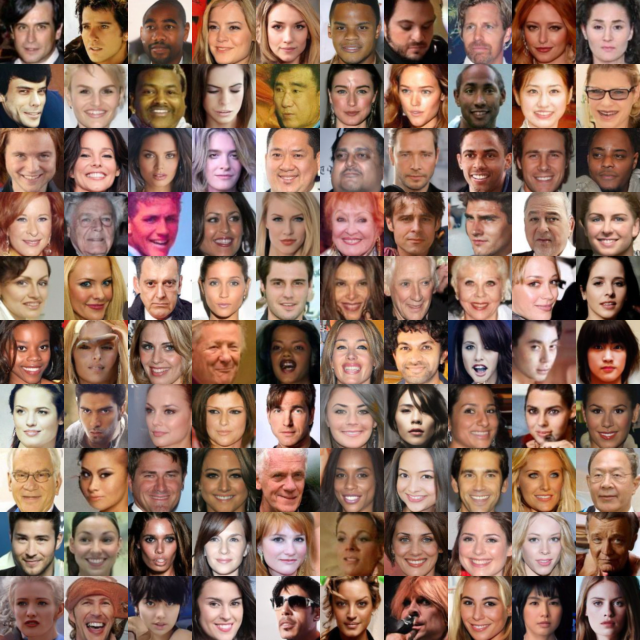}}
\vspace{-.1cm}
\caption{Generated samples on CelebA 64x64.}
\label{fig:K_celeba64}
\end{center}
\end{minipage}
\end{figure}

\newpage
\begin{figure}[H]
\centering
\begin{minipage}{\textwidth}
\begin{center}
\subfloat[\nprddpm{} ($K=25$)]{\includegraphics[width=0.24\linewidth]{imgs/samples/imagenet64_ddpm_npr_25.png}}\ 
\subfloat[\nprddpm{} ($K=200$)]{\includegraphics[width=0.24\linewidth]{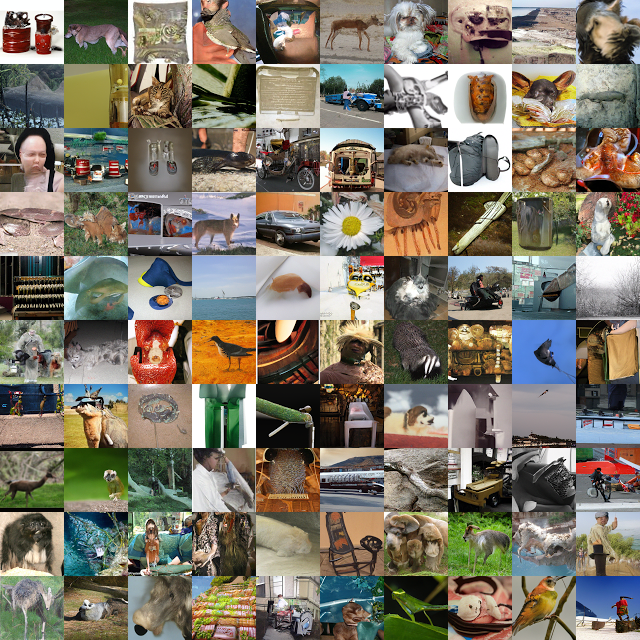}}\ 
\subfloat[\snddpm{} ($K=25$)]{\includegraphics[width=0.24\linewidth]{imgs/samples/imagenet64_ddpm_sn_25.png}}\ 
\subfloat[\snddpm{} ($K=200$)]{\includegraphics[width=0.24\linewidth]{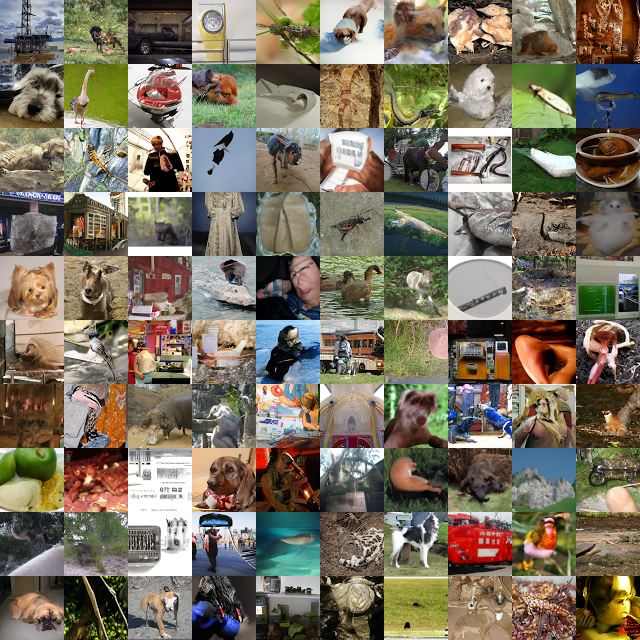}} \\
\vspace{-.2cm}
\subfloat[\nprddim{} ($K=25$)]{\includegraphics[width=0.24\linewidth]{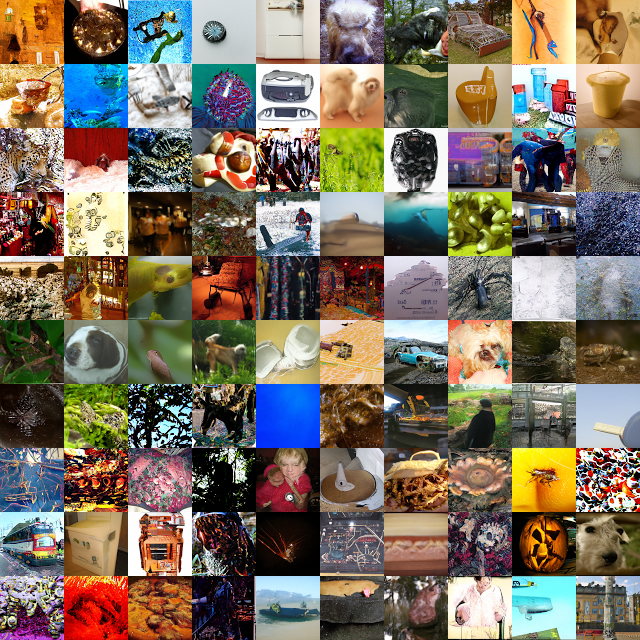}}\ 
\subfloat[\nprddim{} ($K=200$)]{\includegraphics[width=0.24\linewidth]{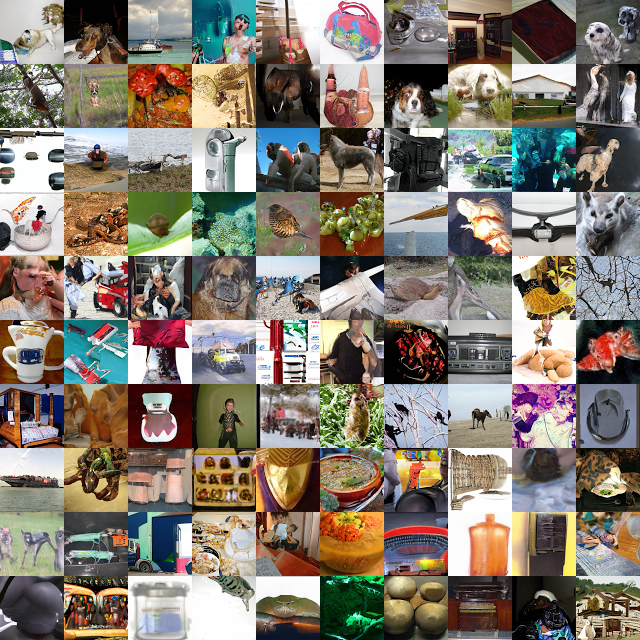}}\ 
\subfloat[\snddim{} ($K=25$)]{\includegraphics[width=0.24\linewidth]{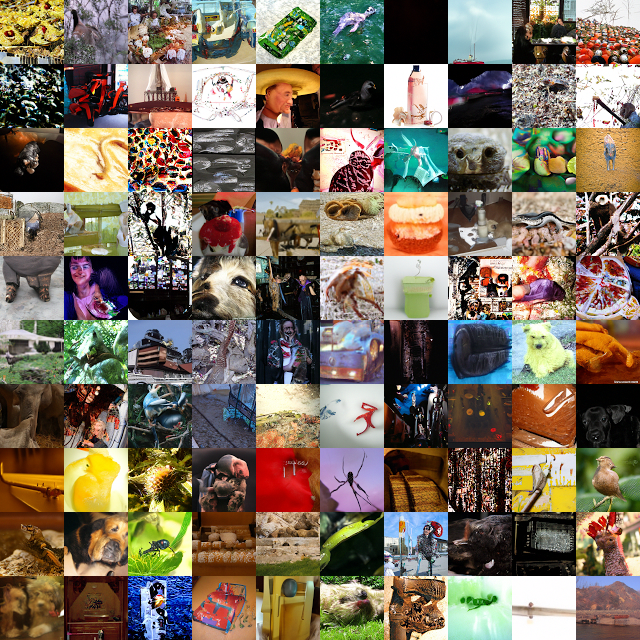}}\ 
\subfloat[\snddim{} ($K=200$)]{\includegraphics[width=0.24\linewidth]{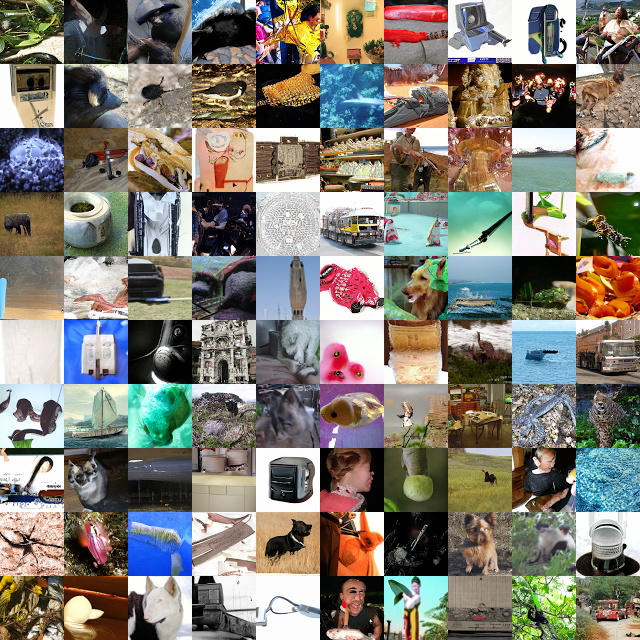}}
\vspace{-.1cm}
\caption{Generated samples on ImageNet 64x64.}
\label{fig:K_imagenet64}
\end{center}
\end{minipage}
\end{figure}

\end{document}